\documentclass[10pt]{article} 
\usepackage[preprint]{tmlr}
\usepackage{microtype}
\usepackage{hyperref}
\usepackage{url}
\usepackage{booktabs}
\usepackage{amsthm}

\newtheorem{theorem}{Theorem}
\newtheorem{proposition}[theorem]{Proposition}
\newtheorem{definition}[theorem]{Definition}

\newtheorem{corollary}{Corollary}
\usepackage{microtype}
\usepackage{booktabs} 
\newcommand{\ie}[0]{\emph{i.e., }}

\newcommand{\eg}[0]{\emph{e.g., }}

\newcommand{\RN}[1]{%
\textup{\lowercase\expandafter{\it \romannumeral#1}}%
}
\usepackage{algorithm} 
\usepackage{algpseudocode}
\usepackage{amsmath}
\usepackage{amssymb}
\usepackage{mathtools}
\usepackage{amsthm}
\usepackage{caption}
\usepackage{color}
\usepackage{colortbl}
\usepackage[breakable]{tcolorbox}
\tcbset{breakable}
\usepackage{arydshln}
\usepackage{array}
\usepackage{adjustbox}
\usepackage{algpseudocode}
\usepackage{enumitem}
\usepackage{multirow} 
\usepackage{makecell}
\usepackage{subcaption}
\usepackage{graphicx}
\usepackage{tabularx}
\usepackage{tablefootnote}
\usepackage{threeparttable}
\usepackage{wrapfig}
\usepackage{xspace}
\usepackage[capitalize,noabbrev]{cleveref}
\newcommand{\qwentwoseven}{Qwen-2.5-7B-Instruct}
\newcommand{\retroagent}{\textsc{RetroAgent}}
\definecolor{darkblue}{rgb}{0, 0, 0.5}
\hypersetup{colorlinks=true, citecolor=darkblue, linkcolor=darkblue, urlcolor=darkblue}
\definecolor{islamicgreen}{rgb}{0.0, 0.56, 0.0}
\definecolor{steelblue}{rgb}{0.27, 0.51, 0.71}
\definecolor{yellow(ncs)}{rgb}{1.0, 0.83, 0.0}

\title{\retroagent{}: From Solving to Evolving via Retrospective Dual Intrinsic Feedback}

\author{
      \name Xiaoying Zhang$^{\ast,\dagger}$ \email zhangxycuhk@gmail.com \\
      \addr Shanghai AI Lab
      \AND
      \name Zichen Liu \\
      \addr National University of Singapore
      \AND
      \name Yipeng Zhang \\
      \addr Independent Researcher
      \AND
      \name Xia Hu \\
      \addr Shanghai AI Lab
      \AND
      \name Wenqi Shao$^{\dagger}$ \\
      \addr Shanghai AI Lab \\
      \AND
      \centerline{\href{https://github.com/zhangxy-2019/RetroAgent}{\textcolor{magenta}{https://github.com/zhangxy-2019/RetroAgent}}}
}

\footnotetext[1]{Xiaoying is project lead, Xiaoying, Zichen, and Yipeng are core contributors.}
\footnotetext[2]{Corresponding authors.}
\def\month{MM}  
\def\year{YYYY} 
\def\openreview{\url{https://openreview.net/forum?id=XXXX}} 

\begin{document}

\maketitle

\begin{abstract}

Standard reinforcement learning (RL) for large language model (LLM) agents primarily optimizes extrinsic task rewards, often favoring isolated task completion over continual adaptation. This paradigm can cause premature convergence to suboptimal policies and leaves useful experience only implicitly encoded in model parameters, limiting its retrieval and reuse for future decisions. We introduce \retroagent{}, an online RL framework that trains agents to master interactive environments \textbf{\textit{not merely by solving tasks, but by evolving across episodes}}. Inspired by human retrospective self-improvement, \retroagent{} augments extrinsic rewards with hindsight-generated dual intrinsic feedback: \((\RN{1})\)~\emph{Intrinsic Numerical Feedback}, which rewards beneficial exploration by measuring incremental subtask progress relative to prior attempts; and \((\RN{2})\)~\emph{Intrinsic Language Feedback}, which distills successes and failures into reusable textual lessons for explicit experience reuse. To leverage these lessons effectively, we propose \emph{Similarity \& Utility-Aware Upper Confidence Bound} (SimUtil-UCB), a retrieval strategy that balances semantic relevance, historical utility, and exploration. Across four challenging agentic benchmarks, \retroagent{} achieves new state-of-the-art performance, outperforming GRPO by \(+18.3\%\) on ALFWorld, \(+15.4\%\) on WebShop, \(+27.1\%\) on Sokoban, and \(+8.9\%\) on MineSweeper, while demonstrating strong test-time adaptation and out-of-distribution generalization.

\end{abstract}


\begin{figure}[th]
    \centering
    \includegraphics[width=\linewidth]{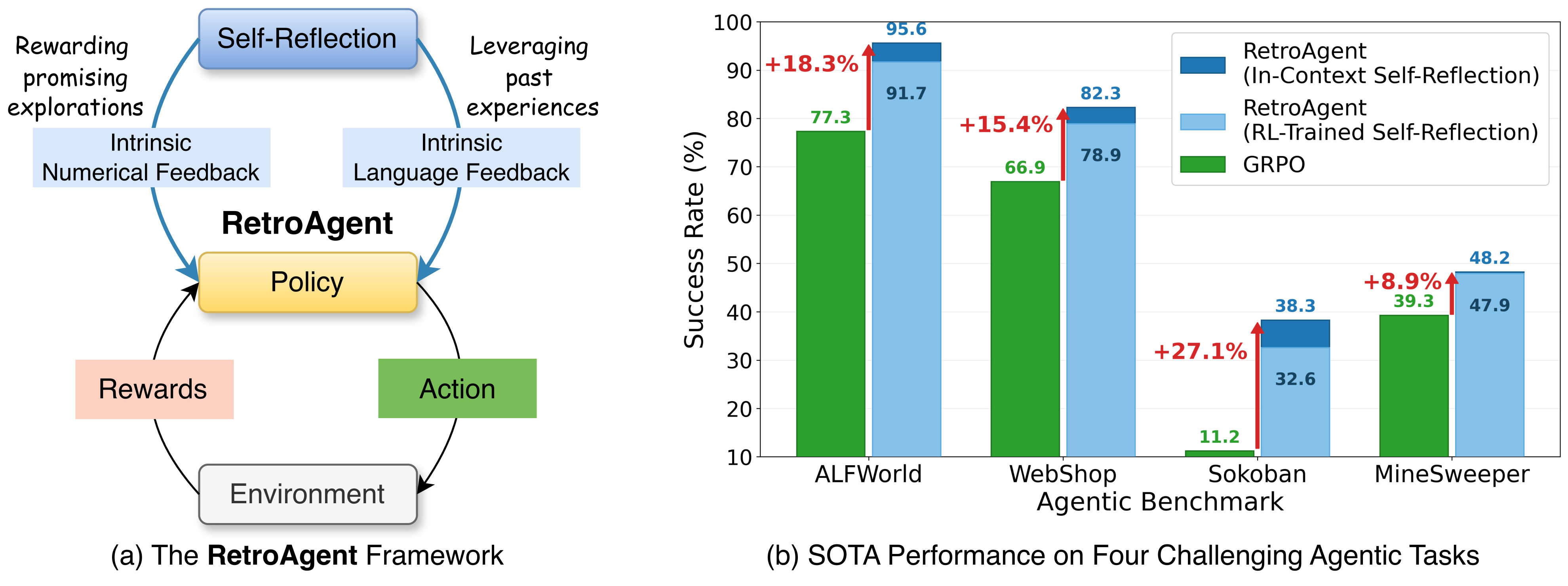}
    \caption{(a)~Overview of \retroagent{}. After each episode, the
agent reflects on its trajectory to generate dual intrinsic
feedback, promoting exploration and facilitating experiential
learning. (b)~Initialized from \qwentwoseven{}, \retroagent{}
substantially outperforms the GRPO baseline and achieves new SOTA across four agentic benchmarks.}
    \label{fig:retroagent_introduction}
\end{figure}

\addtocontents{toc}{\protect\setcounter{tocdepth}{-1}}

\section{Introduction}
\label{sec:intro}
Reinforcement learning (RL)~\citep{sutton1998reinforcement} has become a central paradigm for training large language model (LLM) agents to act in complex interactive environments through direct experience~\citep{DBLP:conf/nips/Ouyang0JAWMZASR22,zhang-etal-2022-toward-self,comanici2025gemini,singh2025openai,liu2025gem}. However, most existing RL frameworks primarily optimize extrinsic task-success rewards, favoring \emph{one-off task solving} over \emph{continuous adaptation} across interactions~\citep{abel2023definition,weng2026learning_beyond_gradients, rafiee2026toward}. For example, in embodied-agent tasks, training often terminates once an agent finds a successful action sequence. Such success-driven optimization has two limitations. First, agents may over-exploit early successes and prematurely converge to suboptimal policies instead of exploring alternatives~\citep{kirk2024understanding}. Second, useful experience is usually stored only implicitly in model parameters, making prior interactions difficult to retrieve, inspect, and reuse for later policy improvement~\citep{lin1992self,graves2014neural}. This can slow learning and weaken generalization, especially when relevant experience is sparse or context-dependent~\citep{goyal2022retrieval}.

Existing work addresses these limitations in two largely separate ways. Exploration-oriented methods broaden search through, for example, meta-RL with cross-episode training~\citep{jiang2025metarlinducesexplorationlanguage,beck2025tutorial} or uncertainty-modulated rewards under sparse feedback~\citep{wang2025harnessinguncertaintyentropymodulatedpolicy}. Yet they do not explicitly preserve reusable experience, limiting the role of past interactions in sustained policy evolution. Conversely, memory-augmented methods store raw trajectories~\citep{goyal2022retrieval,wu2025evolverselfevolvingllmagents,liu2026exploratory} or distilled skills, rules, and lessons~\citep{anthrophic2025skills,wang2025cogitoergoludoagent,liu2026exploratory,xia2026skillrlevolvingagentsrecursive}. However, memory is typically used as \emph{passive context augmentation}: retrieval is governed by fixed similarity metrics, regardless of downstream utility, and remains decoupled from policy optimization. As a result, agents may become brittle when memory is noisy or unavailable, while over-reliance on retrieved experience can suppress further exploration.

These observations motivate a central question: \emph{Can an agent couple exploration with explicit experience reuse so that experience guides policy evolution and is gradually internalized, rather than being persistently supplied by an external memory buffer?} More broadly, can agents exploit intrinsic signals from their own trajectories to evolve~\footnote{In this paper, we use \emph{evolution} to denote the joint process of task solving and continuous adaptation~\citep{hinton1987learning, lehman2023evolution}.} under sparse external supervision?

To answer this question, we introduce \retroagent{} (Figure~\ref{fig:retroagent_introduction}), an online RL framework inspired by human retrospective reflection~\citep{lyons2011monitoring,liu2025position}. \retroagent{} trains agents to master complex interactive environments \textbf{\textit{not merely by solving tasks, but by evolving across episodes}} under both extrinsic task-success rewards and retrospective dual intrinsic feedback. After each episode, the agent analyzes its trajectory to diagnose successes and failures, identify useful progress, and generate two complementary feedback signals. \textbf{Intrinsic Numerical Feedback} rewards exploratory behavior that reflects capability improvement. Since real-time capability evolution is hard to measure directly, we approximate it by incremental subtask progress relative to prior attempts, such as locating a target item even when the final purchase fails. The resulting scalar reward reinforces beneficial exploration and mitigates premature convergence. \textbf{Intrinsic Language Feedback} distills actionable lessons from successful and failed trajectories into an explicit memory buffer for future decision-making. To retrieve useful lessons, we propose \emph{Similarity \& Utility-Aware Upper Confidence Bound} (SimUtil-UCB), which combines semantic relevance with historical utility and uses Upper Confidence Bound (UCB)~\citep{auer2002finite} to balance exploiting high-utility lessons with exploring under-used ones.

We study two variants of \retroagent{}: \((\RN{1})\) an \emph{in-context} self-reflection mechanism, and \((\RN{2})\) an \emph{RL-trained} self-reflection mechanism whose reflective capability is jointly optimized with the decision policy. \retroagent{} is compatible with various RL algorithms; in our implementation, we optimize the decision policy with GRPO~\citep{shao2024deepseekmath} and the self-reflection policy with REINFORCE~\citep{williams1992simple}. We evaluate \retroagent{} using \qwentwoseven{}~\citep{qwen2025qwen25technicalreport} and Llama-3.1-8B-Instruct~\citep{grattafiori2024llama3herdmodels} on four agentic benchmarks: ALFWorld~\citep{DBLP:conf/iclr/ShridharYCBTH21}, WebShop~\citep{DBLP:conf/nips/Yao0YN22}, Sokoban~\citep{racaniere2017imagination}, and MineSweeper~\citep{li2024assessing}. Across all environments, \retroagent{} consistently outperforms prior RL fine-tuning, memory-augmented RL, exploration-guided RL, and meta-RL baselines, improving SOTA success rates by approximately \(+10\%\) on WebShop and \(+16\%\) on Sokoban, while exhibiting strong test-time adaptation and out-of-distribution generalization.

Our contributions are four-fold:
\((\RN{1})\)~We introduce \retroagent{}, an online RL framework with hindsight self-reflection for continuous policy evolution in interactive environments.
\((\RN{2})\)~We design dual intrinsic feedback to promote beneficial exploration and efficient experiential learning.
\((\RN{3})\)~We propose SimUtil-UCB, a retrieval strategy that balances semantic similarity, historical utility, and exploration when leveraging accumulated lessons.
\((\RN{4})\)~We demonstrate across four agentic benchmarks that \retroagent{} substantially outperforms strong baselines and achieves new SOTA results in both in-distribution and out-of-distribution settings.

\section{Related Work}

\paragraph{LLMs as Decision-Making Agents.}
The reasoning capabilities of LLMs have driven their deployment as
autonomous decision-making agents. An initial line of research prompts
frozen LLMs: ReAct~\citep{yao2022react},
Reflexion~\citep{shinn2023reflexion}, and related
methods~\citep{park2023generative, wang2024voyager} leverage in-context
examples, structured prompts, memory
retrieval~\citep{wang2024mobileagentv}, and external
tools~\citep{schick2023toolformer, xie2024osworld, zhang2025ufo} to
tackle complex tasks. However, these approaches are inherently bounded
by the capabilities of the underlying foundation model. This ceiling has
motivated a second line of work that trains LLM agents
directly---through supervised fine-tuning~\citep{tajwar2025training,
xi-etal-2025-agentgym} or RL~\citep{song-etal-2024-trial,
zhang2025agentrlscalingagenticreinforcement, feng2025groupingroup,
jiang2025metarlinducesexplorationlanguage}---enabling them to improve
from environmental interactions rather than relying on static prompts or
handcrafted workflows.

\paragraph{Reinforcement Learning for LLM Agents.}
RL has become a central paradigm for training agents in multi-turn,
dynamic environments~\citep{wang2025ragen, putta2025agent, liu2025understanding, liu2025gem}.
ArCHer~\citep{zhou2024archer} employs hierarchical value functions for
WebShop~\citep{yao2022webshop}, while LOOP~\citep{chen2025reinforcementlearninglonghorizoninteractive}
integrates PPO~\citep{schulman2017proximal} with Leave-One-Out advantage
estimation for long-horizon tasks in
AppWorld~\citep{trivedi2024appworld}. Group-based RL methods have further
refined credit assignment: building on
GRPO~\citep{shao2024deepseekmathpushinglimitsmathematical},
GiGPO~\citep{feng2025groupingroup} introduces two-level advantage
estimation, while other works investigate turn-level reward
shaping~\citep{wei2025reinforcingmultiturnreasoningllm} and stepwise
progress attribution~\citep{wang2025sparlreinforcingllmagents}.
Meta-RL~\citep{beck2025tutorial} offers a complementary perspective;
notably, LAMER~\citep{jiang2025metarlinducesexplorationlanguage} uses
cross-episode training to enable active test-time exploration. However,
these methods optimize primarily against extrinsic environmental
feedback, and recent analyses argue that genuine self-improvement
requires intrinsic signals beyond sparse task
rewards~\citep{liu2025position}. Although prior works have explored
intrinsic motivation~\citep{gao2025navigateunknownenhancingllm} or
entropy-modulated
policies~\citep{wang2025harnessinguncertaintyentropymodulatedpolicy},
\retroagent{} takes a fundamentally different path: a hindsight
self-reflection mechanism produces dual intrinsic feedback, shifting the
objective from isolated problem-solving toward continuous adaptation.

\paragraph{Learning from Experience through Retrospection.}
A growing body of work moves beyond scalar rewards by leveraging verbal
feedback and retrospective memory for agent self-improvement. Early
approaches~\citep{shinn2023reflexion, madaan2023self, yao2024retroformer}
generate natural-language critiques or lessons from interactions,
iteratively refining same-task performance via in-context learning.
Subsequent work internalizes such feedback into model parameters:
\citet{jiang2025metarlinducesexplorationlanguage} use reflections to
guide cross-episode adaptation within a meta-RL framework, while
\citet{zhang2025critique, hübotter2026reinforcementlearningselfdistillation}
refine failed trajectories into high-quality data for policy optimization
through RL or distillation. A complementary direction adopts
memory-based architectures~\citep{goyal2022retrieval,
wu2025evolverselfevolvingllmagents,wang2025cogitoergoludoagent,
zhang2026memrlselfevolvingagentsruntime,
zhou2025mementofinetuningllmagents,
fang2026mempexploringagentprocedural, liu2026exploratory} that store trajectories, lessons,
or skills~\citep{xia2026skillrlevolvingagentsrecursive} in a retrieval
buffer to assist similar future tasks in context. \retroagent{} advances
this paradigm along a new axis: the agent reflects on its trajectories
to produce both intrinsic numerical rewards that guide exploration and
intrinsic language feedback that facilitates exploiting past experiences,
with these dual signals jointly driving policy optimization.
\section{\retroagent{}}
We introduce \retroagent{} (Figure~\ref{fig:retroagent_framework}), an online reinforcement learning framework that employs hindsight self-reflection to improve exploration and experiential learning under sparse extrinsic feedback. Section~\ref{sec:method_overview} formalizes the problem and overviews the self-reflection mechanism. Next, we detail our strategies for utilizing intrinsic feedback: Section~\ref{sec:intrinsic_numerical} introduces numerical feedback to promote exploration, while Section~\ref{sec:intrinsic_language} explains how language feedback enables the exploitation of past experience. Finally, Section~\ref{sec:policy_optimization_dual_feedback} establishes the policy optimization objectives for both \retroagent{} variants.

\begin{figure}[t]
    \centering
    \includegraphics[width=\linewidth]{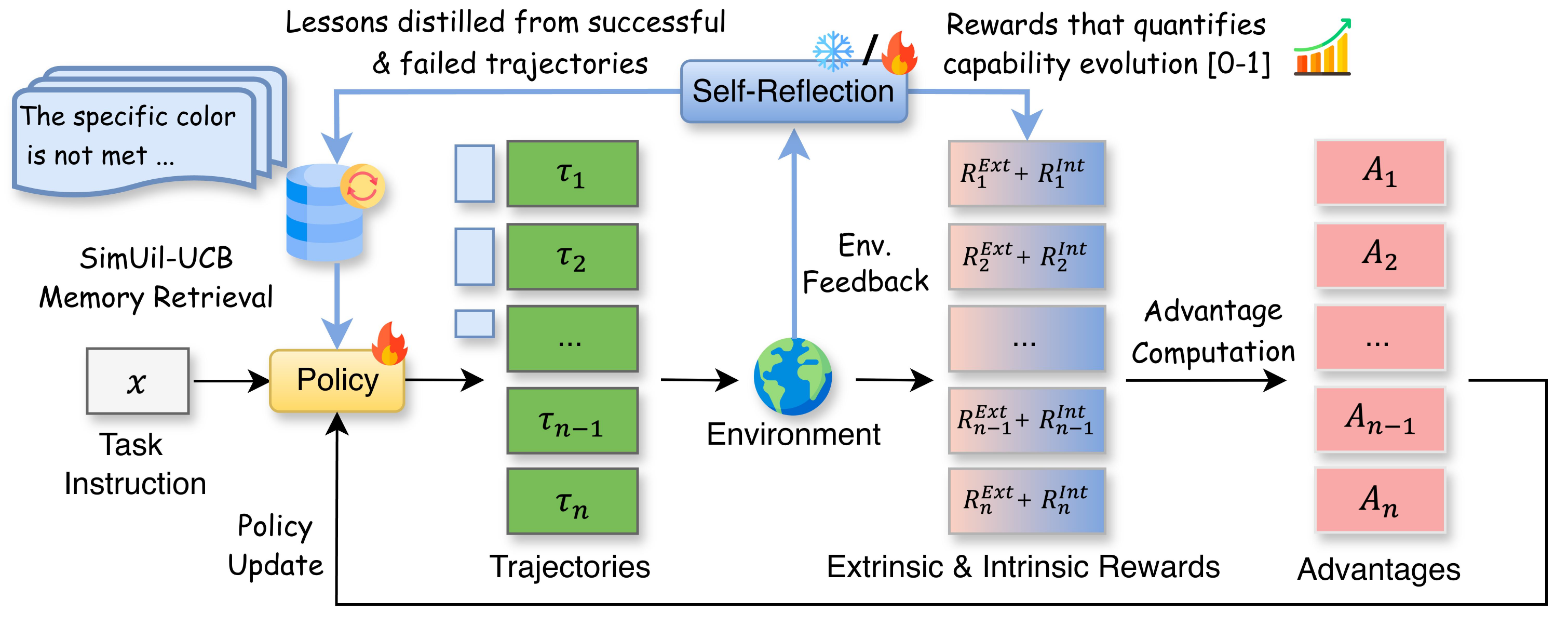}
    \caption{Overview of the \textsc{RetroAgent} framework. After each
    episode, a self-reflection mechanism analyzes the trajectory to
    produce two forms of intrinsic feedback:
    $(\RN{1})$~\emph{Intrinsic Numerical Feedback}, which quantifies
    incremental subtask completion relative to prior attempts,
    rewarding promising exploratory behaviors that may not yet yield task success; and $(\RN{2})$~\emph{Intrinsic Language Feedback}, which
    distills actionable lessons from past successes and failures into a
    memory buffer, retrieved via the proposed SimUtil-UCB strategy to
    effectively leverage accumulated experiences on similar tasks.}
    \label{fig:retroagent_framework}
\end{figure}

\subsection{General Overview}
\label{sec:method_overview}

\paragraph{Problem Formulation.}
We model the LLM agent's multi-turn interaction with its environment as a Markov Decision Process (MDP)~\citep{sutton1998reinforcement}, defined by $\mathcal{M} = (\mathcal{S}, \mathcal{A}, P, R, \gamma)$, where $\mathcal{S}$ is the state space, $\mathcal{A}$ the action space, $P(s_{t+1} \mid s_t, a_t)$ the environment's transition dynamics, $R(s_t, a_t)$ the reward function, and $\gamma \in [0,1]$ the discount factor. At each step $t = 0, \dots, T{-}1$, the agent observes state $s_t \in \mathcal{S}$ and samples action $a_t \in \mathcal{A}$ from its policy $\pi_{\theta}(\cdot \mid s_t)$. In the LLM agent setting, the state is the concatenation of all preceding observations and actions: $s_t = (o_0, a_0, \dots, a_{t-1}, o_t)$. Executing $a_t$ yields reward $r_{t+1} = R(s_t, a_t)$ and successor state $s_{t+1} \sim P(\cdot \mid s_t, a_t)$, producing a trajectory $\tau = (s_0, a_0, r_1, \dots, s_{T-1}, a_{T-1}, r_T)$. With purely extrinsic rewards $r_{t+1} = r_{t+1}^{\text{ext}}$, the standard objective is to maximize the expected discounted return: 
\begin{equation}
    \mathcal{J}_{\text{Standard}}(\theta)
    = \mathbb{E}_{\tau \sim \pi_\theta(\cdot \mid x) \times P}
      \!\left[\, G_0 \,\right]
    = \mathbb{E}_{\tau \sim \pi_\theta(\cdot \mid x) \times P}
      \!\left[\, \sum_{t=0}^{T-1} \gamma^{t}\, r_{t+1}^{\text{ext}} \right],
\end{equation}
\noindent where $x = o_0$ is the task instruction drawn from the training set~$\mathcal{D}$, and $\tau \sim \pi_\theta(\cdot \mid x) \times P$ denotes that trajectories are generated jointly by the policy and the environment dynamics. In practice, extrinsic rewards are sparse: a non-zero terminal reward $R^{\text{ext}}$ is provided only when the episode ends, either upon successful task completion or upon exceeding the allowed number of steps. To simplify credit assignment, we redistribute this terminal reward uniformly across all steps, setting $r_{t+1}^{\text{ext}} = R^{\text{ext}}$ for every~$t$.

\retroagent{} augments this objective with intrinsic feedback from a hindsight self-reflection mechanism. An intrinsic reward $R^{\text{int}}$ (Section~\ref{sec:intrinsic_numerical}) is likewise assigned uniformly to every step, yielding the composite objective:
\begin{equation}
    \mathcal{J}_{\text{RetroAgent}}(\theta)
    = \mathbb{E}_{\tau \sim \Pi_\theta(\cdot \mid x,\, \mathcal{M}) \times P}
      \left[\, \sum_{t=0}^{T-1} \gamma^{t}
      \left(R^{\text{ext}} + R^{\text{int}} \right) \right],
    \label{eq:objective_1}
\end{equation}
\noindent where \(\Pi_\theta(\cdot \mid x,\, \mathcal{M})\) denotes a mixture distribution over trajectories induced by two policies: the base policy \(\pi_\theta(\cdot \mid x)\) and a memory-augmented policy \(\pi_\theta\!\bigl(\cdot \mid f_{\text{memory}}(x,\, \mathcal{M})\bigr)\). Here, \(f_{\text{memory}}(x, \mathcal{M})\) represents the proposed SimUtil-UCB retrieval strategy (Section~\ref{sec:intrinsic_language}), which selects a memory instance that is both relevant and useful from the memory buffer \(\mathcal{M}\) (which grows over time) to augment the task instruction~\(x\).

\paragraph{Hindsight Self-Reflection.}
At its core, \retroagent{} incorporates a hindsight self-reflection mechanism for efficient experiential learning. At the conclusion of each episode, the agent evaluates its trajectory via a reflection function $z = f_{\text{reflect}}(\tau)$, leveraging in-context learning~\citep{DBLP:journals/tmlr/WeiTBRZBYBZMCHVLDF22}.\footnote{For notational simplicity, we reuse $\tau$ to denote the agent--environment interaction history, consisting of interleaved observations and actions.} This function produces a reflection tuple $z = (\phi_{(x,\tau)}, c, l)$ comprising three components: $(\RN{1})$~a scalar \textit{potential score} $\phi_{(x,\tau)} \in [0,1]$ estimating the subtask completion rate, from which the intrinsic numerical reward $R^{\text{int}}$ is derived (Section~\ref{sec:intrinsic_numerical}); $(\RN{2})$~a binary \textit{success prediction} $c \in \{\text{success}, \text{failure}\}$; and $(\RN{3})$~a natural-language \textit{retrospective lesson} $l$ distilled from the trajectory. The lesson \(l\), encapsulated within a memory entry \(m\), is stored in a memory buffer \(\mathcal{M}\) and subsequently retrieved via \(f_{\text{memory}}(x,\, \mathcal{M})\) to provide intrinsic language feedback (Section~\ref{sec:intrinsic_language}).

The central challenge of this mechanism lies in eliciting high-quality intrinsic feedback. To this end, we propose two variants: an \textbf{\emph{in-context}} variant and an \textbf{\emph{RL-trained}} variant.

\noindent \textbf{In-Context Variant.}
We employ \textit{pairwise induction} by augmenting the reflection function with two additional inputs: $(\text{i})$~a binary outcome indicator $I^{\text{ext}} \in \{\text{success}, \text{failure}\}$, and $(\text{ii})$~a contrastive reference trajectory $\tau_{\text{ref}}$ collected from an earlier training step whose outcome differs from that of the current episode. Contrasting successful and failed trajectories enables the model to more precisely isolate behavioral strengths and deficiencies, yielding higher-quality potential scores and lessons~\citep{lee2023rlaif}. The resulting reflection function takes the form $z = f_{\text{reflect}}(\tau_{\text{ref}},\, I^{\text{ext}},\, \tau)$.

\noindent \textbf{RL-Trained Variant.}
In this variant, the agent is jointly optimized so that its self-reflection capability co-evolves with its decision-making policy. We introduce a reflection reward $R^{\text{reflect}}$ that quantifies the accuracy of the agent's self-assessment:
\begin{equation}\label{eq:reflect_reward}
  R^{\text{reflect}} := R^{\text{ext},(i)} \cdot \mathbf{1}\{c=I^{\mathrm{ext}}\},
\end{equation}
where \(\mathbf{1}\{\cdot\}\) is the indicator function and $c$ is the success prediction produced by the reflection. Scaling by $R^{\text{ext}}$ aligns the magnitude of the reflection reward with that of the extrinsic signal.\footnote{Alternative reward-scaling strategies are possible but are left for future work.} Let $\varphi_\theta$ denote the reflection policy, which generates the reflection tuple $z = (\phi_{(x,\tau)},\, c,\, l)$ conditioned on the trajectory~$\tau$. The composite training objective generalizes Equation~\ref{eq:objective_1} by incorporating a self-reflection term:
\begin{equation}
    \label{eq:composite_objective}
    \mathcal{J}_{\text{RetroAgent}}(\theta)
    = \underbrace{
        \mathbb{E}_{\tau \sim \Pi_\theta(\cdot \mid x,\, \mathcal{M}) \times P}
        \left[\,\sum_{t=0}^{T-1} \gamma^{t}
        \left( R^{\text{ext}} + R^{\text{int}} \right)
        \right]
      }_{\text{Decision-Making}}
    \;+\;
      \underbrace{\lambda_{\text{reflect}} \cdot
        \mathbb{E}_{z \sim \varphi_\theta(\cdot \mid \tau)}
        \left[\, R^{\text{reflect}} \right]
      }_{\text{Self-Reflection}},
\end{equation}
where $\lambda_{\text{reflect}} \geq 0$ is a coefficient controlling the relative weight of the self-reflection objective; Equation~\ref{eq:objective_1} is recovered when $\lambda_{\text{reflect}} = 0$. Prompt templates for both variants are provided in Appendix~\ref{app:task_prompts}, and optimization details are discussed in Section~\ref{sec:policy_optimization_dual_feedback}.

\subsection{Intrinsic Numerical Feedback for Directed Exploration}
\label{sec:intrinsic_numerical}

Sparse terminal rewards do not distinguish uninformative failures from
trajectories that make meaningful partial progress. This limitation is
especially severe for LLM agents, whose action spaces are large and whose
useful intermediate reasoning or tool-use behaviors may not immediately produce
task success. To provide a directed exploration signal, \retroagent{} converts
the reflected potential score \(\phi_{(x,\tau)}\) into an intrinsic
\emph{capability-evolution} reward. This reward assigns intrinsic credit only
to trajectories whose reflected progress exceeds the best group-level
performance previously observed for the same task.

For each task \(x\), we maintain a historical baseline \(\Phi_x^{(k)}\), where
\(\Phi_x^{(0)}=0\). The baseline records the maximum group-mean extrinsic
success rate observed before iteration \(k\). After collecting \(N\) rollouts
for task \(x\) at iteration \(k\), we compute
$\bar{I}^{\mathrm{ext}}_k(x)
    =
    \frac{1}{N}
    \sum_{j=1}^{N} I_{k}^{\mathrm{ext},(j)},$
where \(I_{k}^{\mathrm{ext},(j)}\in\{0,1\}\) denotes the extrinsic success
indicator of the \(j\)-th rollout. The baseline is then updated by
\begin{equation}
    \Phi_x^{(k+1)}
    =
    \max\left\{
        \Phi_x^{(k)},
        \bar{I}^{\mathrm{ext}}_k(x)
    \right\}.
    \label{eq:baseline_update}
\end{equation}
Given a trajectory \(\tau\) with reflected potential score
\(\phi_{(x,\tau),k}\), its intrinsic reward is defined as the rectified gain
over the current baseline:
\begin{equation}
    R_k^{\mathrm{int}}(\tau)
    =
    \left[
        \phi_{(x,\tau),k}
        -
        \Phi_x^{(k)}
    \right]_+,
    \qquad
    [a]_+ \coloneqq \max\{a,0\}.
    \label{eq:intrinsic_reward}
\end{equation}
The trajectory-level reward used for policy optimization is
\[
    R_k^{\mathrm{aug}}(\tau)
    =
    R^{\mathrm{ext}}(\tau)
    +
    R_k^{\mathrm{int}}(\tau).
\]
When the iteration is clear from context, we omit the superscript \(k\).

This reward has two useful properties. First, it is
\emph{progress-sensitive}: even failed trajectories may receive positive
intrinsic reward if their reflected potential exceeds the historical baseline.
Second, it is \emph{self-normalizing}: since \(\Phi_x^{(k)}\) is monotonically
non-decreasing, intrinsic credit is assigned only to trajectories that improve
upon the agent's current capability threshold. Thus, the reward encourages
capability improvement rather than undirected novelty or repeated partial
behaviors.

\paragraph{Theoretical Justification.}
We next formalize why the capability-evolution reward promotes directed
exploration. For a task \(x\) and policy class
\(\Pi\), define the extrinsic objective of policy \(\pi\in\Pi\) as $J_{\mathrm{ext}}(\pi;x)
    =
    \mathbb{E}_{\tau\sim\pi(\cdot\mid x)\times P}
    \left[
        R^{\mathrm{ext}}(\tau)
    \right],$
and define the augmented objective at iteration \(k\) as $J_{\mathrm{aug}}(\pi;x)
    =
    \mathbb{E}_{\tau\sim\pi(\cdot\mid x)\times P}
    \left[
        R^{\mathrm{ext}}(\tau)
        +
        R_k^{\mathrm{int}}(\tau)
    \right].$

\begin{definition}[Valuable exploratory trajectory]
\label{def:valuable_exploration}
At iteration \(k\), a trajectory \(\tau\) is \(\delta\)-valuable for task \(x\)
if its reflected potential exceeds the historical baseline by at least
\(\delta>0\): $\phi_{(x,\tau),k}
    \ge
    \Phi_x^{(k)}+\delta.$
We denote the set of all such trajectories by
\[
    \mathcal{V}_{k,\delta}(x)
    =
    \left\{
        \tau:
        \phi_{(x,\tau),k}
        \ge
        \Phi_x^{(k)}+\delta
    \right\}.
\]
\end{definition}

\begin{proposition}[Intrinsic reward favors valuable exploration]
\label{prop:intrinsic_policy_search}
Let $p_{k,\delta}(\pi;x)
    =
    \Pr_{\tau\sim\pi(\cdot\mid x)\times P}
    \left[
        \tau\in\mathcal{V}_{k,\delta}(x)
    \right].$
Then, for any policy \(\pi\in\Pi\),
\begin{equation}
    J_{\mathrm{aug}}(\pi;x)
    \ge
    J_{\mathrm{ext}}(\pi;x)
    +
    \delta\,p_{k,\delta}(\pi;x).
    \label{eq:valuable_lower_bound}
\end{equation}
Consequently, among policies with equal extrinsic value, the augmented
objective favors policies that sample \(\delta\)-valuable trajectories with
higher probability.
\end{proposition}

\begin{proof}
If \(\tau\in\mathcal{V}_{k,\delta}(x)\), then $R_k^{\mathrm{int}}(\tau)
    =
    \left[
        \phi_{(x,\tau),k}
        -
        \Phi_x^{(k)}
    \right]_+
    \ge
    \delta.$
For all other trajectories, \(R_k^{\mathrm{int}}(\tau)\ge 0\). Therefore, $\mathbb{E}_{\tau\sim\pi(\cdot\mid x)\times P}
    \left[
        R_k^{\mathrm{int}}(\tau)
    \right]
    \ge
    \delta\,
    \Pr_{\tau\sim\pi(\cdot\mid x)\times P}
    \left[
        \tau\in\mathcal{V}_{k,\delta}(x)
    \right].$
Adding \(J_{\mathrm{ext}}(\pi;x)\) to both sides proves the claim.
\end{proof}

Proposition~\ref{prop:intrinsic_policy_search} shows that intrinsic reward
breaks the degeneracy among failed trajectories in sparse-reward settings:
failures with stronger reflected progress obtain larger augmented returns and
therefore provide a directed learning signal.

\begin{proposition}[Monotone exploration pressure]
\label{prop:monotone_exploration_pressure}
For every task \(x\), the baseline sequence
\(\{\Phi_x^{(k)}\}_{k\ge 0}\) is monotonically non-decreasing. Moreover,
\[
    R_k^{\mathrm{int}}(\tau)>0
    \quad\Longleftrightarrow\quad
    \phi_{(x,\tau),k}>\Phi_x^{(k)}.
\]
Thus, as the agent's capability improves, intrinsic reward is assigned only to
trajectories that exceed the updated threshold.
\end{proposition}

\begin{proof}
By the update rule in Equation~\eqref{eq:baseline_update}, $\Phi_x^{(k+1)}
    =
    \max\left\{
        \Phi_x^{(k)},
        \bar{I}^{\mathrm{ext}}_k(x)
    \right\}
    \ge
    \Phi_x^{(k)}.$
Hence, \(\{\Phi_x^{(k)}\}_{k\ge 0}\) is monotonically non-decreasing. The
equivalence follows immediately from
$R_k^{\mathrm{int}}(\tau)
    =
    \left[
        \phi_{(x,\tau),k}
        -
        \Phi_x^{(k)}
    \right]_+.$
\end{proof}

\subsection{Intrinsic Language Feedback via Reflection Memory}
\label{sec:intrinsic_language}

The numerical exploration reward indicates \emph{which} trajectories are promising, but it does not explain \emph{how} the agent should improve.
To provide reusable guidance, \retroagent{} maintains a retrieval-augmented reflection memory.
Each memory entry is \(m_i=(x_i,l_i,\tau_i,u_i,n_i,d_i)\), where \(x_i\) is the task instruction, \(l_i\) is the retrospective lesson, \(\tau_i\) is the source trajectory, \(u_i\in[0,1]\) is the empirical utility of the lesson, \(n_i\in\mathbb{N}\) is its retrieval count, and \(d_i\in\{\mathrm{success},\mathrm{failure}\}\) is the originating extrinsic outcome.

For efficient semantic retrieval, each task instruction is embedded using a frozen sentence encoder \(\mathcal{E}\). We embed each memory instruction as $\mathbf{v}_i = \mathcal{E}(x_i).$

\paragraph{SimUtil-UCB retrieval.}
Given a current task \(x\), SimUtil-UCB scores each memory entry using semantic relevance, empirical utility, and an exploration bonus.
Semantic relevance is measured by cosine similarity,
\begin{equation}
    s_{\mathrm{rel}}(x,x_i)
    =
    \frac{
        \mathcal{E}(x)^\top \mathbf{v}_i
    }{
        \|\mathcal{E}(x)\|_2\|\mathbf{v}_i\|_2
    },
    \label{eq:sim}
\end{equation}
and entries with \(s_{\mathrm{rel}}(x,x_i)<0.4\) are discarded.
Each utility score is initialized to \(u_i=0.5\).
Whenever \(m_i\) is retrieved and the resulting episode obtains task success score \(\hat{u}_t\in[0,1]\), the utility is updated by \(u_i\leftarrow(1-\beta_{\mathrm{util}})u_i+\beta_{\mathrm{util}}\hat{u}_t\), where \(\beta_{\mathrm{util}}\in(0,1)\).

To avoid repeatedly exploiting only a small subset of memories, SimUtil-UCB adds a UCB exploration bonus.
Let \(N_{\mathcal{M}}=\sum_j n_j\) be the total number of memory retrievals.
The utility-UCB value is
\begin{equation}
    u_{\mathrm{UCB}}^{(i)}
    =
    u_i
    +
    \kappa
    \sqrt{
        \frac{\ln N_{\mathcal{M}}}{\max(n_i,1)}
    },
    \label{eq:ucb}
\end{equation}
where \(\kappa>0\) controls memory exploration and is set to \(1.0\) in our experiments.
The final retrieval score is
\begin{equation}
    S(m_i\mid x,\mathcal{M})
    =
    \alpha s_{\mathrm{rel}}(x,x_i)
    +
    (1-\alpha)u_{\mathrm{UCB}}^{(i)},
    \label{eq:retrieval_score}
\end{equation}
where \(\alpha\in[0,1]\) controls the relevance--utility trade-off.
Let \(\mathcal{K}\) be the top-\(k\) entries under \(S\).
The selected lessons are concatenated as \(l_{\mathrm{retrieved}}=\bigoplus_{i\in\mathcal{K}}l_i\), yielding the memory-augmented input \(f_{\mathrm{memory}}(x,\mathcal{M})=x\oplus l_{\mathrm{retrieved}}\).
After retrieval, each selected entry is marked as accessed by setting \(n_i\leftarrow n_i+1\) for all \(i\in\mathcal{K}\).

\paragraph{Theoretical Justification.}
We formalize reflection memory as intrinsic language feedback using the Learning from Language Feedback (LLF) framework~\citep{xu2025provablylearninglanguagefeedback}. Scalar rewards provide outcome-level supervision, whereas retrieved lessons impose trajectory-level constraints that can distinguish task-solving hypotheses more efficiently than rewards alone.

Let \(\Pi\) denote the policy class and \(\mathcal{H}\) a class of text-expressible hypotheses over task-solving strategies. Each hypothesis \(\eta \in \mathcal{H}\) induces a value function: $r_\eta : \Pi \to [0,1].$
For a task \(x\), let \(\mathcal{M}(x)\) be the distribution over retrieved lessons. A lesson \(l \sim \mathcal{M}(x)\) induces a verifier loss:
\[
    \ell(\pi, l, \eta) \in [0,1],
\]
where smaller values indicate that hypothesis \(\eta\) better explains the behavior of policy \(\pi\) under lesson \(l\). The memory-induced loss and its corresponding excess loss are defined as:
\[
    L_{\mathcal{M}(x)}(\pi, \eta) = \mathbb{E}_{l \sim \mathcal{M}(x)} \bigl[ \ell(\pi, l, \eta) \bigr],
\]
\[
    \Delta_{\mathcal{M}(x)}(\pi, \eta) = L_{\mathcal{M}(x)}(\pi, \eta) - \inf_{\eta' \in \mathcal{H}} L_{\mathcal{M}(x)}(\pi, \eta').
\]

\begin{definition}[Memory-informative feedback]
\label{def:memory_informative}
For a task \(x\), the retrieved lesson distribution \(\mathcal{M}(x)\) is \((C, \epsilon)\)-memory-informative if there exist constants \(C > 0\) and \(\epsilon \ge 0\) such that, for all \(\eta, \eta' \in \mathcal{H}\) and \(\pi \in \Pi\):
\[
    \bigl| r_\eta(\pi) - r_{\eta'}(\pi) \bigr|^2 \le C \Delta_{\mathcal{M}(x)}(\pi, \eta') + \epsilon.
\]
\end{definition}

Definition~\ref{def:memory_informative} establishes that if two hypotheses yield different policy values, the retrieved lessons induce a detectable verifier-loss gap, bounded by an approximation error \(\epsilon\). Thus, reflection memory is effective when reward-relevant disagreements are captured by language-feedback consistency.

\begin{definition}[Memory-conditioned transfer eluder dimension]
\label{def:memory_te}
Given a memory distribution \(\mathcal{M}\), verifier loss \(\ell\), and precision \(\epsilon > 0\), the memory-conditioned transfer eluder dimension: $\dim_{\mathrm{TE}} \left( \mathcal{H}, \ell, \epsilon \mid \mathcal{M} \right)$
is the maximum length \(d\) of a sequence \(\pi_1, \ldots, \pi_d \in \Pi\) such that, for some \(\epsilon' \ge \epsilon\), each \(\pi_j\) is \(\epsilon'\)-transfer independent of \(\{\pi_1, \ldots, \pi_{j-1}\}\) under verifier losses induced by lessons retrieved from \(\mathcal{M}\).
\end{definition}

This dimension quantifies the exploration burden—specifically, how many policies must be evaluated before memory-induced language feedback resolves relevant uncertainty.

\begin{theorem}[Memory-guided complexity reduction]
\label{thm:memory_reduces_te}
Suppose \(\mathcal{M}(x)\) is \((C, \epsilon)\)-memory-informative, and define the induced reward class \(\mathcal{R}_{\mathcal{H}} = \{r_\eta : \eta \in \mathcal{H}\}\). Following the transfer-eluder analysis for reward-informative language feedback~\citep{xu2025provablylearninglanguagefeedback}, we have: 
\[
    \dim_{\mathrm{TE}} \left( \mathcal{H}, C\ell, \epsilon \mid \mathcal{M} \right) \le \dim_{\mathrm{E}} \left( \mathcal{R}_{\mathcal{H}}, \epsilon \right),
\]
up to the approximation error in Definition~\ref{def:memory_informative}. This inequality can be strict when retrieved lessons identify intermediate subgoals, failure modes, or corrective actions unobservable from scalar rewards alone.
\end{theorem}

\begin{proof}[Proof sketch]
By Definition~\ref{def:memory_informative}, any pair of hypotheses that differ in induced value on a policy also exhibit a corresponding gap in memory-induced verifier loss, scaled by \(C\) and up to error \(\epsilon\). Thus, memory feedback can distinguish all reward-relevant hypothesis disagreements that scalar rewards can. Consequently, any sequence remaining transfer-independent under the scaled verifier \(C\ell\) cannot exceed the length of the longest sequence remaining independent under the induced reward class \(\mathcal{R}_{\mathcal{H}}\). Strict improvement occurs because a single explanatory lesson can eliminate multiple hypotheses that sparse scalar rewards cannot distinguish.
\end{proof}

Theorem~\ref{thm:memory_reduces_te} formalizes reflection memory's role in reducing exploration complexity: retrieved lessons do not merely provide context; they actively constrain the set of plausible task-solving hypotheses.

\begin{proposition}[SimUtil-UCB retrieval behavior]
\label{prop:simutil_ucb_theory}
Assume \(u_i\) is an unbiased or asymptotically consistent estimate of memory \(m_i\)'s expected downstream utility on semantically related tasks, and that \(s_{\mathrm{rel}}(x, x_i)\) is bounded. SimUtil-UCB exhibits the following properties: \((\RN{1})\)~Among memories with comparable semantic relevance, entries with higher empirical utility are preferred as \(n_i\) grows. \((\RN{2})\)~Among memories with comparable empirical utility, entries with higher semantic relevance are preferred. \((\RN{3})\)~Memories with lower retrieval counts receive larger exploration bonuses.
\end{proposition}

\begin{proof}
The retrieval score is defined as:
\[
    S(m_i \mid x, \mathcal{M}) = \alpha s_{\mathrm{rel}}(x, x_i) + (1 - \alpha) \left( u_i + \kappa \sqrt{\frac{\ln N_{\mathcal{M}}}{\max(n_i, 1)}} \right),
\]
where \(N_{\mathcal{M}} = \sum_j n_j\). For a fixed semantic relevance, increasing \(u_i\) increases \(S\), favoring higher-utility memories. For a fixed utility and retrieval count, increasing \(s_{\mathrm{rel}}(x, x_i)\) increases \(S\), favoring more relevant memories. Finally, the exploration term decreases monotonically with \(n_i\), ensuring under-tested memories receive a larger bonus. As \(n_i \to \infty\), this bonus vanishes, and retrieval is governed strictly by semantic relevance and empirical utility, adhering to the optimism-under-uncertainty principle~\citep{auer2002finite}.
\end{proof}

\subsection{Policy Optimization with Dual Intrinsic Feedback}
\label{sec:policy_optimization_dual_feedback}

\retroagent{} is compatible with a broad class of RL algorithms. In this work, we instantiate it with GRPO~\citep{shao2024deepseekmath}, adapted to incorporate dual intrinsic feedback into multi-turn trajectory optimization. We describe the trajectory generation procedure, the decision-making objective, and the optional self-reflection objective in turn.

\paragraph{Trajectory Generation with Memory Augmentation.}
For each task instruction~\(x\) from~\(\mathcal{D}\), we generate~\(N\) trajectories under~\(\Pi_{\theta_{\text{old}}}(\cdot \mid x, \, \mathcal{M}) \times P\) (Equation~\ref{eq:objective_1}). The first~\(N/2\) are sampled from the base policy,~\(\tau^{(i)} \sim \pi_{\theta_{\text{old}}}(\cdot \mid x) \times P\), and the remaining~\(N/2\) from the memory-augmented policy,~\(\tau^{(i)} \sim \pi_{\theta_{\text{old}}}\!\bigl(\cdot \mid f_{\text{memory}}(x, \mathcal{M})\bigr) \times P\). Each trajectory~\(\tau^{(i)} = (s_0^{(i)}, a_0^{(i)}, \dots, s_{T_i-1}^{(i)}, a_{T_i-1}^{(i)})\) is a state--action sequence of length~\(T_i\). This partition lets the agent leverage past experience via memory retrieval while retaining the capacity for independent exploration, facilitating continuous policy adaptation.

\begin{algorithm}[t]
\caption{\retroagent{} Training Framework}
\label{alg:retroagent}
\begin{algorithmic}[1]
\Require Training dataset \(\mathcal{D}\), Sentence encoder \(\mathcal{E}\), hyperparameters \(\alpha, \beta_{\text{util}}, \kappa, \lambda_{\text{reflect}}\).
\State Initialize memory buffer \(\mathcal{M} \leftarrow \emptyset\), policy parameters \(\theta\), historical baselines \(\Phi_x \leftarrow 0\) for all \(x \in \mathcal{D}\).
\For{each training iteration}
    \State Sample task instruction \(x \sim \mathcal{D}\).
    
    \State \textbf{\% 1. Memory Retrieval (SimUtil-UCB)}
    \If{\(\mathcal{M} \neq \emptyset\)}
        \State Compute semantic relevance \(s_{\mathrm{rel}}\) via cosine similarity using \(\mathcal{E}(x)\).
        \State Retrieve top-\(k\) entries \(\mathcal{K}\) maximizing \(S(m_i \mid x, \mathcal{M}) = \alpha s_{\mathrm{rel}} + (1 - \alpha)u_{\mathrm{util-UCB}}^{(i)}\).
        \State Form augmented input \(f_{\text{memory}}(x, \mathcal{M}) = x \oplus l_{\text{retrieved}}\).
        \State Increment access counts: \(n_i \leftarrow n_i + 1\) for all \(i \in \mathcal{K}\).
    \EndIf

    \State \textbf{\% 2. Trajectory Generation}
    \State Generate \(N/2\) trajectories \(\tau\) via base policy \(\pi_{\theta_{\text{old}}}(\cdot \mid x)\).
    \State Generate \(N/2\) trajectories \(\tau\) via memory-augmented policy \(\pi_{\theta_{\text{old}}}\!\bigl(\cdot \mid f_{\text{memory}}(x, \mathcal{M})\bigr)\).

    \State \textbf{\% 3. Self-Reflection \& Intrinsic Feedback}
    \For{each trajectory \(\tau^{(i)}\) in the \(N\) rollouts}
        \State Observe extrinsic reward \(R^{\text{ext},(i)}\) and outcome \(I^{\text{ext},(i)}\).
        \State Generate reflection tuple \(z^{(i)} = (\phi_{(x,\tau)}^{(i)}, c^{(i)}, l^{(i)})\) via reflection function \(f_{\text{reflect}}\) or policy \(\varphi_\theta\).
        \State Compute capability-evolution intrinsic reward: \(R^{\text{int},(i)} \leftarrow \max\!\bigl(0,\; \phi_{(x,\tau)}^{(i)} - \Phi_x\bigr)\).
        \State Store new memory entry \(m^{(i)}\) into buffer \(\mathcal{M}\).
    \EndFor
    
    \State Update task baseline: \(\Phi_x \leftarrow \max\!\bigl(\Phi_x,\; \frac{1}{N}\sum_{j=1}^{N} I^{\text{ext},(j)}\bigr)\).
    \State Update utilities \(u_i\) for retrieved entries \(i \in \mathcal{K}\) via EMA: \(u_i \leftarrow (1-\beta_{\text{util}})u_i + \beta_{\text{util}}\hat{u}_t\).

    \State \textbf{\% 4. Policy Optimization (Dual Feedback)}
    \State Compute advantages \(\hat{A}^{(i)}\) using composite returns \(G^{(i)} = \sum \gamma^{t} \bigl(R^{\text{ext},(i)} + R^{\text{int},(i)}\bigr)\).
    \State Update decision-making policy \(\theta\) by maximizing \(\mathcal{J}_{\text{Decision-Making}}(\theta)\) via GRPO.
    \If{using RL-Trained Reflection Variant}
        \State Compute reflection reward: \(R^{\text{reflect},(i)} \leftarrow R^{\text{ext},(i)} \cdot \mathbf{1}\{c=I^{\mathrm{ext}}\}\).
        \State Update reflection policy \(\varphi_\theta\) by maximizing \(\mathcal{J}_{\text{Self-Reflection}}(\theta)\) via REINFORCE.
    \EndIf
\EndFor
\end{algorithmic}
\end{algorithm}

\paragraph{Decision-Making Objective.}
Since both~$R^{\text{ext}}$ and~$R^{\text{int}}$ are uniform across time steps (Section~\ref{sec:method_overview}), the discounted return reduces to a trajectory-level scalar~$G^{(i)} = \sum_{t=0}^{T_i-1} \gamma^{t}\bigl(R^{\text{ext},(i)} + R^{\text{int},(i)}\bigr)$, and every step within a trajectory shares the same group-relative advantage: $    \hat{A}^{(i)} = \frac{G^{(i)} - \operatorname{mean}\!\bigl(\{G^{(1)}, \dots, G^{(N)}\}\bigr)}{\operatorname{std}\!\bigl(\{G^{(1)}, \dots, G^{(N)}\}\bigr)}.$

Defining the per-token importance ratio as $\rho_{t,j}^{(i)}(\theta) = \frac{\pi_\theta(a_{t,j}^{(i)} \mid s_t^{(i)},\, a_{t, <j}^{(i)})}{\pi_{\theta_{\text{old}}}(a_{t,j}^{(i)} \mid s_t^{(i)},\, a_{t,<j}^{(i)})}$, the decision-making objective is formulated as:
\begin{equation}\label{eq:grpo}
\begin{aligned}
    \mathcal{J}_{\text{Decision-Making}}(\theta)
    &= \mathbb{E}_{x \sim \mathcal{D},\, \{\tau^{(i)}\} \sim \Pi_{\theta_{\text{old}}}(\cdot \mid x,\, \mathcal{M}) \times P} \Bigg[ \frac{1}{N}
    \sum_{i=1}^{N} \frac{1}{T_i} \sum_{t=0}^{T_i-1} \sum_{j=1}^{|a_t^{(i)}|} \bigg( \mathcal{L}_{t,j}^{\text{clip}}\!\bigl(\theta,\, \hat{A}^{(i)}\bigr) \\
    &\qquad - \beta \, D_{\mathrm{KL}}\!\bigl[\pi_\theta(\cdot \mid s_t^{(i)}) \,\|\, \pi_{\text{ref}}(\cdot \mid s_t^{(i)})\bigr] \bigg) \Bigg],
\end{aligned}
\end{equation}
where $|a_t^{(i)}|$ denotes the number of tokens in action $a_t^{(i)}$. Following~\citet{liu2025understanding}, we remove the length normalization term \(\frac{1}{|a_t^{(i)}|}\). The clipped surrogate function is defined as $\mathcal{L}^{\text{clip}}_{t,j}\!\left(\theta,\, \hat{A}^{(i)}\right) = \min \left( \rho_{t,j}^{(i)}(\theta)\, \hat{A}^{(i)},\; \operatorname{clip}\!\left(\rho_{t,j}^{(i)}(\theta),\, 1{-}\epsilon_{\text{clip}},\, 1{+}\epsilon_{\text{clip}}\right) \hat{A}^{(i)} \right)$, where $\epsilon_{\text{clip}}$ bounds the policy update and $\beta$ controls the KL divergence regularization toward the reference policy $\pi_{\text{ref}}$. For the \emph{in-context} self-reflection variant, the total objective is simply $\mathcal{J}_{\text{RetroAgent}}(\theta) = \mathcal{J}_{\text{Decision-Making}}(\theta)$.

\paragraph{Self-Reflection Objective (for RL-Trained Variant).}
The \emph{RL-trained} variant additionally optimizes the reflection policy~$\varphi_\theta$. For each trajectory~$\tau^{(i)}$, $\varphi_\theta$ generates a reflection sequence $z^{(i)} = (\phi^{(i)}_{(x,\tau)},\, c^{(i)},\, l^{(i)})$. The success prediction component $c^{(i)}$ is scored by $R^{\text{reflect},(i)}$ (Equation~\eqref{eq:reflect_reward}). We optimize $\varphi_\theta$ using REINFORCE~\citep{williams1992simple}:
\begin{equation}\label{eq:reinforce}
    \mathcal{J}_{\text{Self-Reflection}}(\theta) =
    \mathbb{E}_{\{z^{(i)} \sim \varphi_{\theta_{\text{old}}}(\cdot \mid \tau^{(i)}) \}}
    \!\biggl[\frac{1}{N}\sum_{i=1}^{N}   \sum_{j=1}^{|z^{(i)}|} 
    \log \varphi_\theta\!\bigl(z_j^{(i)} \mid \tau^{(i)},\,z_{<j}^{(i)} \bigr) \cdot R^{\text{reflect},(i)}\biggr],
\end{equation}
where $|z^{(i)}|$ is the token length of the reflection sequence. Gradients from this objective do not propagate through trajectory generation.
The full RL-trained objective is
\begin{equation}
    \mathcal{J}_{\text{RetroAgent}}(\theta)
    =
    \mathcal{J}_{\text{Decision-Making}}(\theta)
    +
    \lambda_{\text{reflect}} \cdot
    \mathcal{J}_{\text{Self-Reflection}}(\theta).
    \label{eq:full_objective}
\end{equation}

We summarize the \retroagent{} framework in Algorithm~\ref{alg:retroagent} and provide theoretical justifications in Appendix~\ref{app:theoretical_analysis}.

\section{Experiments}

\subsection{Experimental Setup}
\label{sec:setup}
\paragraph{Environments.}
We evaluate \textsc{RetroAgent} across four distinct agentic tasks:
$(\RN{1})$~ALFWorld~\citep{DBLP:conf/iclr/ShridharYCBTH21}, a text-based embodied environment where agents complete household tasks through navigation and object interaction. We assess both in-distribution (seen rooms) and out-of-distribution (unseen rooms) generalization.
$(\RN{2})$~Webshop~\citep{DBLP:conf/nips/Yao0YN22}, a simulated e-commerce environment requiring agents to navigate a web interface to purchase products matching user specifications.
$(\RN{3})$~Sokoban~\citep{racaniere2017imagination}, a planning-heavy puzzle task where agents must push boxes to target locations. Due to the irreversible nature of pushing actions, errors often render puzzles unsolvable. Complexity is governed by board size and box count; we train on $6 \times 6$ boards with 2 boxes, following~\citet{jiang2025metarlinducesexplorationlanguage}.
$(\RN{4})$~MineSweeper~\citep{li2024assessing}, a logic-based puzzle requiring agents to identify mine locations using numerical clues. Difficulty is controlled by board size and mine density; we train on $6 \times 6$ boards with 3 mines. We report Success Rate across all tasks, supplemented by Task Score for WebShop.

\paragraph{Compared Methods.}
We evaluate \retroagent{}gent{} against four categories of competitive baselines, reporting results averaged over three independent runs: $(\RN{1})$~\textit{Prompting-based methods}: We compare against ReAct~\citep{yao2022react} and Reflexion~\citep{shinn2023reflexion}, the latter of which incorporates an in-context self-reflection mechanism for iterative refinement. $(\RN{2})$~\textit{RL algorithms}: We include REINFORCE Leave-One-Out (RLOO)~\citep{kool2019buy, ahmadian2024back}, GRPO~\citep{shao2024deepseekmath}, and Group-in-Group Policy Optimization (GiGPO)~\citep{feng2025groupingroup}. GiGPO represents the current state-of-the-art by utilizing anchor-state grouping for fine-grained credit assignment.
$(\RN{3})$~\textit{RL-based frameworks}: This category includes memory-augmented methods such as MemRL~\citep{zhang2026memrlselfevolvingagentsruntime} (which updates a memory bank while keeping the policy frozen), EvolveR~\citep{wu2025evolverselfevolvingllmagents} (which integrates raw trajectories into optimization), and Mem0~\citep{chhikara2025mem0buildingproductionreadyai}+GRPO and SimpleMem~\citep{liu2026simplememefficientlifelongmemory}+GRPO, (which incorporate persistent memory into the training process). We also compare against SkillRL~\citep{xia2026skillrlevolvingagentsrecursive}, a hybrid approach (supervised finetuning and RL) that induces actionable skills via a teacher model to guide the student's policy optimization, and GRPO with EMPG~\citep{wang2025harnessinguncertaintyentropymodulatedpolicy}, which utilizes entropy-modulated policy gradients for long-horizon optimization.
$(\RN{4})$~\textit{A Meta-RL framework~\citep{beck2025tutorial}}: We compare against \textsc{LaMer}~\citep{jiang2025metarlinducesexplorationlanguage}, which leverages a multi-episode structure to foster active exploration and robust adaptation within a meta-learning context.

\paragraph{Implementation Details.}
We evaluate \retroagent{} on \qwentwoseven{}~\citep{qwen2025qwen25technicalreport} and Llama-3.1-8B-Instruct~\citep{grattafiori2024llama3herdmodels}. Although \retroagent{} is generally compatible with various RL algorithms, we adopt GRPO as the default and implement our framework by adapting the open-source Verl training library~\citep{sheng2024hybridflow}. We employ the task prompts from~\citet{feng2025groupingroup} to enable decision-making via the \textit{ReAct} format~\citep{yao2022react}, in which the model generates step-by-step reasoning before its corresponding action. At training time, the agent distills lessons as memories from trajectories on the training set; at test time, the agent leverages these memories for task completion on the test set. Detailed hyperparameter settings and training configurations are provided in Appendix~\ref{app:implementation_details}.


\subsection{Main Results}
\label{sec:main_results}
\begin{table*}[t]
\caption{Main results across four benchmarks, averaged over three independent runs (mean $\pm$ standard deviation). All improvements are statistically significant with $p < 0.01$. Results marked with $^{*}$ are cited from prior work~\citep{xia2026skillrlevolvingagentsrecursive, feng2025groupingroup, wang2025harnessinguncertaintyentropymodulatedpolicy}. Unless otherwise specified, all training frameworks use the GRPO algorithm. ``Success'' and ``Score'' denote Success Rate and Task Score, respectively. \textcolor{blue}{w/ Teacher Model} indicates methods that require a teacher model for skill induction.}
\label{tab:main_results}
\centering
\small
\setlength{\tabcolsep}{2pt}
\resizebox{\textwidth}{!}{
\begin{tabular}{lccccc}
\toprule
\multirow{2}{*}{\textbf{Method}} & \textbf{ALFWorld} & \multicolumn{2}{c}{\textbf{WebShop}} & \textbf{Sokoban} & \textbf{MineSweeper} \\
\cmidrule(lr){2-2} \cmidrule(lr){3-4} \cmidrule(lr){5-5} \cmidrule(lr){6-6}
& Success (\%) & Score (\%) & Success (\%) & Success (\%) & Success (\%) \\ 
\midrule
\qwentwoseven{} (Zero-Shot) & $16.9 _{\pm1.8}$ & $4.5 _{\pm1.8}$ & $0.8 _{\pm0.0}$ & $2.6 _{\pm0.5}$ & $6.5 _{\pm1.6}$ \\
\multicolumn{6}{l}{\textit{\textcolor{gray}{Prompting-based Methods}}} \\
ReAct$^{*}$~\citep{yao2022react} & 31.2 & 46.2 & 19.5 & 3.9 & 7.0 \\
Reflexion$^{*}$~\citep{shinn2023reflexion} & 42.7 & 58.1 & 28.8 & 4.3 & 7.4 \\
\hdashline
\noalign{\vskip 0.07cm} 
\multicolumn{6}{l}{\textit{\textcolor{gray}{Fine-tuning with RL}}} \\
RLOO$^{*}$~\citep{kool2019buy} & $75.5 _{\pm4.6}$ & $80.3 _{\pm3.2}$ & $65.7 _{\pm4.0}$ & $9.9_{\pm1.6}$ & $32.8 _{\pm4.8}$ \\
GRPO~\citep{shao2024deepseekmath} & $77.3 _{\pm4.3}$ & $75.5 _{\pm3.6}$ & $66.9 _{\pm1.2}$ & $11.2 _{\pm2.5}$ & $39.3 _{\pm2.7}$ \\
GiGPO$^{*}$~\citep{feng2025groupingroup} & $90.8 _{\pm1.3}$ & $84.4 _{\pm2.9}$ & $72.8 _{\pm3.2}$ & $21.9 _{\pm2.8}$ & $41.1 _{\pm1.2}$ \\
\noalign{\vskip 0.07cm} 
\multicolumn{6}{l}{\textit{\textcolor{gray}{Fine-tuning with RL-based Frameworks}}} \\
MemRL$^{*}$~\citep{zhang2026memrlselfevolvingagentsruntime} & 21.4 & 29.5 & 9.2 & $4.2_{\pm3.2}$ & $7.0_{\pm1.4}$ \\
EvolveR$^{*}$~\citep{wu2025evolverselfevolvingllmagents} & 43.8 & 42.5 & 17.6 & $6.0_{\pm3.2}$ & $11.7_{\pm3.1}$ \\

Mem0~\citep{chhikara2025mem0buildingproductionreadyai}+GRPO$^{*}$ & 54.7 & 58.1 & 37.5 & --  & --\\
SimpleMem~\citep{liu2026simplememefficientlifelongmemory}+GRPO$^{*}$ & 62.5 & 67.8 & 46.9 & -- & -- \\
\textsc{SkillRL}$^{*}$~\citep{xia2026skillrlevolvingagentsrecursive} \textcolor{blue}{w/ Teacher Model} & 89.9 & 85.2 & 72.7 & -- & -- \\
GRPO w/ EMPG$^{*}$~\citep{wang2025harnessinguncertaintyentropymodulatedpolicy} & 78.5 & 81.0 & 69.3 & $12.8_{\pm2.3}$ & $40.1_{\pm3.6}$  \\
\noalign{\vskip 0.07cm} 
\multicolumn{6}{l}{\textit{\textcolor{gray}{Fine-tuning with Meta-RL Frameworks}}} \\
\textsc{LaMer}~\citep{jiang2025metarlinducesexplorationlanguage} & $82.3 _{\pm3.6}$ & -- & $61.7 _{\pm4.7}$ & $14.3 _{\pm1.2}$ & $33.3 _{\pm1.8}$ \\
\noalign{\vskip 0.07cm}
\multicolumn{6}{l}{\textit{\textcolor{gray}{RL Training with Extrinsic and Dual Intrinsic Feedback}}} \\
\textsc{RetroAgent} (In-Context Reflection) & $91.7 _{\pm1.2}$ & $87.6 _{\pm2.1}$ & $78.9 _{\pm3.6}$ & $32.6_{\pm4.6}$ & $47.9_{\pm2.0}$ \\
\textsc{RetroAgent} (RL-Trained Reflection) & $\mathbf{95.6_{\pm2.3}}$ & $\mathbf{88.9_{\pm1.3}}$ & $\mathbf{82.3_{\pm 1.6}}$ & $\mathbf{38.3_{\pm 3.4}}$ & $\mathbf{48.2 _{\pm2.0}}$ \\
\bottomrule
\end{tabular}
}
\end{table*}

We present the main results in Table~\ref{tab:main_results} and summarize the key findings below.

\noindent \textbf{Retrospective dual intrinsic feedback improves policy evolution.}
\retroagent{} achieves state-of-the-art performance on all four benchmarks, outperforming GRPO by \(+14.4\), \(+12.0\), \(+21.4\), and \(+8.6\) percentage points on ALFWorld, WebShop, Sokoban, and MineSweeper, respectively. These gains show that augmenting extrinsic rewards with hindsight-generated intrinsic feedback enables more effective learning than success-driven RL alone. On WebShop, \retroagent{} further surpasses the strongest baselines, GiGPO and \textsc{SkillRL}, by \(+6.1\)--\(6.2\%\), confirming the benefit of combining progress-aware exploration with explicit experience reuse.

\noindent \textbf{Numerical and language feedback are complementary.}
\retroagent{} consistently outperforms memory-augmented methods, including MemRL, EvolveR, SimpleMem+GRPO, and \textsc{SkillRL}, showing that textual memory alone is insufficient without progress-aware exploration signals. It also exceeds GRPO w/ EMPG, which uses uncertainty-based intrinsic rewards, indicating that numerical signals alone miss reusable trajectory-level knowledge. These results validate the complementary roles of intrinsic numerical feedback for exploration and intrinsic language feedback for experience reuse.

\noindent \textbf{Distilled lessons are more useful than raw trajectories.}
\retroagent{} substantially outperforms EvolveR, which reuses raw trajectories as in-context demonstrations; on WebShop, \retroagent{} achieves \(78.9\)--\(82.3\%\) success versus \(17.6\%\) for EvolveR. This suggests that raw trajectories may contain noisy or task-specific details, whereas retrospective self-reflection produces concise and transferable lessons for future decisions.

\noindent \textbf{RL-trained self-reflection further improves performance.}
Jointly training the self-reflection policy further improves \retroagent{}, reaching \(95.6\%\), \(82.3\%\), and \(38.3\%\) success on ALFWorld, WebShop, and Sokoban, respectively. This indicates that reflective feedback generation can itself be optimized to better support policy evolution.

\subsection{Test-Time Adaptation and Generalisation}
\label{sec:test_time_adaptation}
\begin{figure*}[th]
\centering
\begin{subfigure}[th]{0.48\linewidth}
    \centering
    \includegraphics[width=\linewidth]{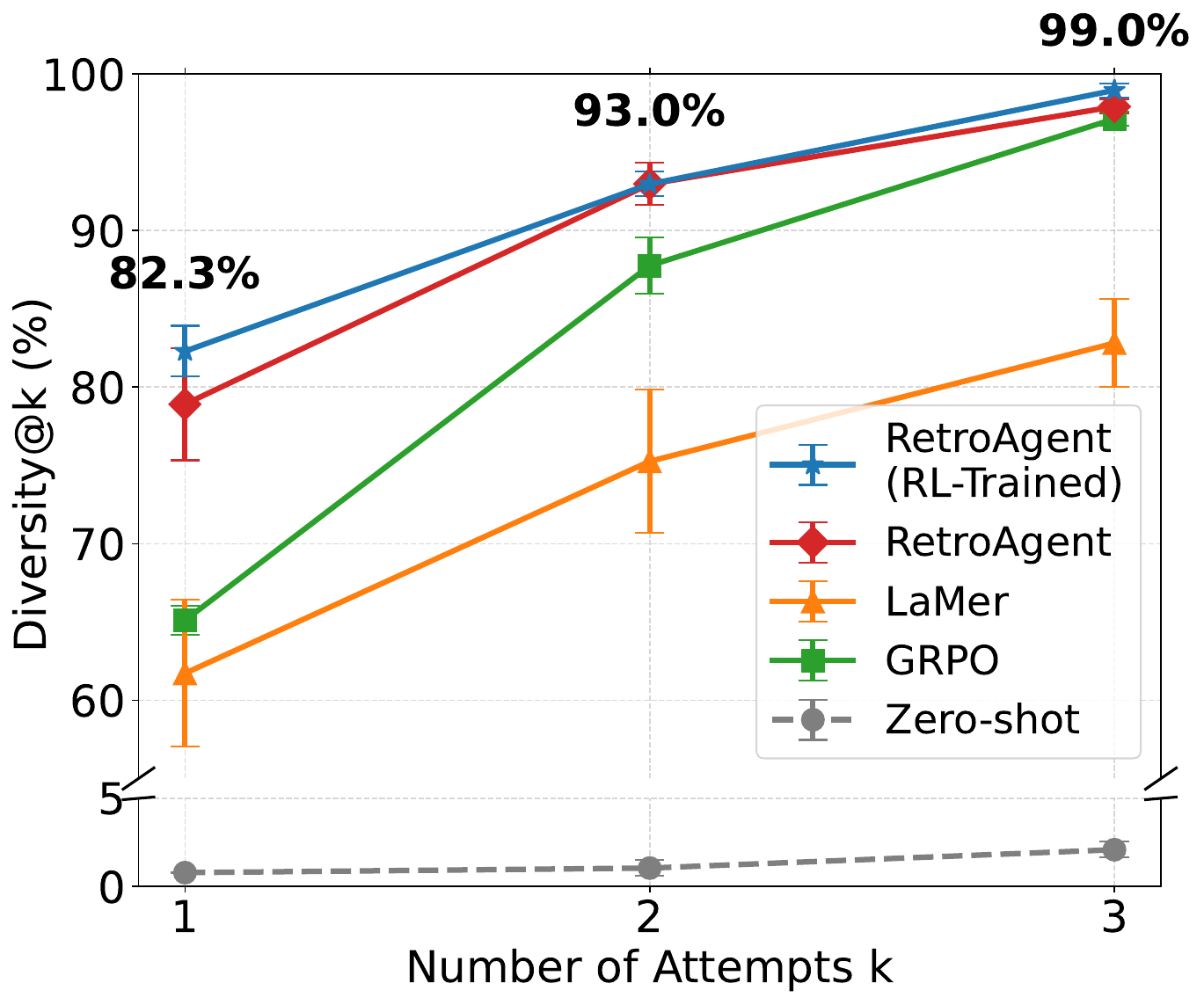}
    \caption{Test-time adaptation on WebShop (ID).}
    \label{fig:test_time_webshop}
\end{subfigure}
\hfill
\begin{subfigure}[th]{0.48\linewidth}
    \centering
    \includegraphics[width=\linewidth]{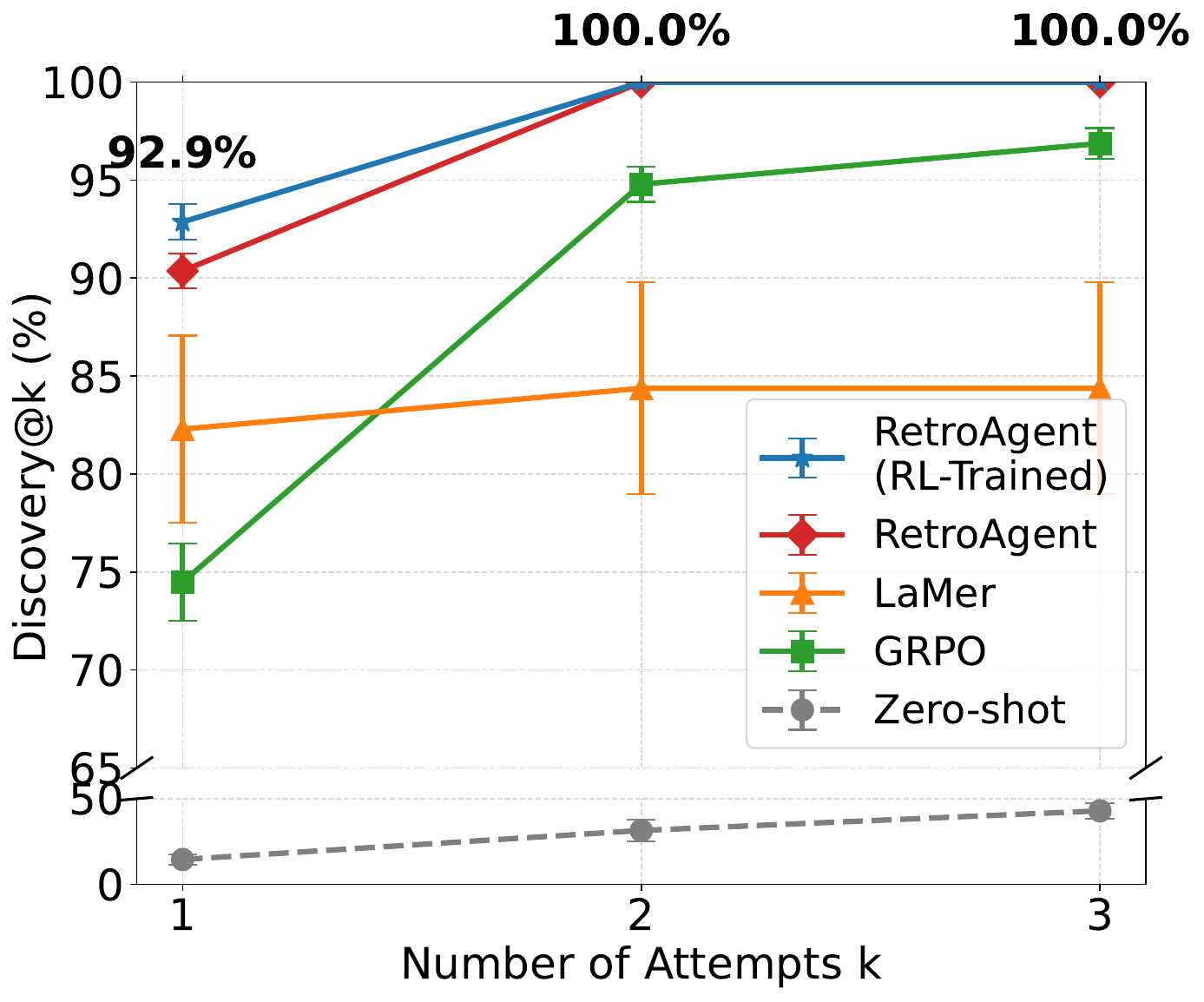}
    \caption{Test-time adaptation on ALFWorld (OOD).}
    \label{fig:test_time_alfworld}
\end{subfigure}
\caption{Test-time adaptation in an in-distribution (ID) setting on WebShop and an out-of-distribution (OOD) setting on ALFWorld.}
\label{fig:test_time_adaptation}
\end{figure*}

\paragraph{Test-Time Adaptation.}
Following~\citet{jiang2025metarlinducesexplorationlanguage}, we evaluate test-time adaptation with \(\mathrm{Discovery}@k\)~\citep{hübotter2026reinforcementlearningselfdistillation}, defined as the probability of solving a task within \(k\) attempts:
$\mathrm{Discovery}@k := P\!\left(\bigvee_{i=1}^{k} r(y_i \mid x)=1\right).$
Results are shown in Figure~\ref{fig:test_time_adaptation}.

\noindent \textbf{\retroagent{} enables rapid test-time adaptation.}
\retroagent{} reaches near-perfect discovery within three attempts in both in-distribution WebShop (\(82.3\% \to 99.0\%\)) and out-of-distribution ALFWorld (\(92.9\% \to 100.0\%\)) settings, consistently outperforming the Meta-RL baseline \textsc{LaMer}. The margin over \textsc{LaMer} further increases with \(k\) in OOD settings, suggesting that retrospective reasoning better exploits repeated attempts.

\begin{table*}[th]
\caption{Impact of memory retrieval on test-time adaptation.}
\label{tab:test_adaptation_self_reflection}
\centering
\begin{threeparttable}
\small
\setlength{\tabcolsep}{4pt}
\resizebox{\textwidth}{!}{
\begin{tabular}{lcccc}
\toprule
\multirow{2.5}{*}{\textbf{Method}}& \multirow{2.5}{*}{\makecell[c]{\textbf{Memory} \\ \textbf{Retrieval}}} & \multicolumn{3}{c}{\textbf{WebShop}} \\
\cmidrule(lr){3-5}
& & Discovery@1 (\%) & Discovery@2 (\%) & Discovery@3 (\%) \\
\midrule
GRPO (Baseline) & -- & $66.9 _{\pm1.2}$ & $87.8_{\pm 1.8}$ & $97.1_{\pm 0.5}$ \\
\textsc{RetroAgent} (In-Context) & $\times$ & $76.8_{\pm 1.6}$ & $91.9_{\pm 1.2}$ & $98.4_{\pm 0.0}$ \\
\textsc{RetroAgent} (RL-Trained) & $\times$ & $77.1_{\pm 1.6}$ & $91.7_{\pm 1.2}$ & $99.0_{\pm 0.5}$ \\
\hdashline
\noalign{\vskip 0.07cm} 
\textsc{RetroAgent} (In-Context) & $\checkmark$ & $78.9 _{\pm3.6}$ & $93.0_{\pm 1.4}$ & $97.9_{\pm 0.5}$ \\
\textsc{RetroAgent} (RL-Trained) & $\checkmark$ & $\mathbf{82.3_{\pm 1.6}}$ & $\mathbf{93.0_{\pm 0.8}}$ & $\mathbf{99.0_{\pm 0.5}}$ \\
\bottomrule
\end{tabular}
}
\end{threeparttable}
\end{table*}

\noindent \textbf{\retroagent{} internalizes intrinsic feedback during training.}
Table~\ref{tab:test_adaptation_self_reflection} ablates test-time memory retrieval. Removing retrieval causes only minor drops in \(\mathrm{Discovery}@1\) and \(\mathrm{Discovery}@2\) (\eg \(78.9\% \to 76.8\%\) for in-context self-reflection), while \(\mathrm{Discovery}@3\) remains unchanged. This suggests that dual intrinsic feedback is largely absorbed into the policy parameters, rather than relying on memory access at inference time.

\begin{figure*}[t]
\centering
\begin{subfigure}[t]{0.48\linewidth}
    \centering
    \includegraphics[width=\linewidth]{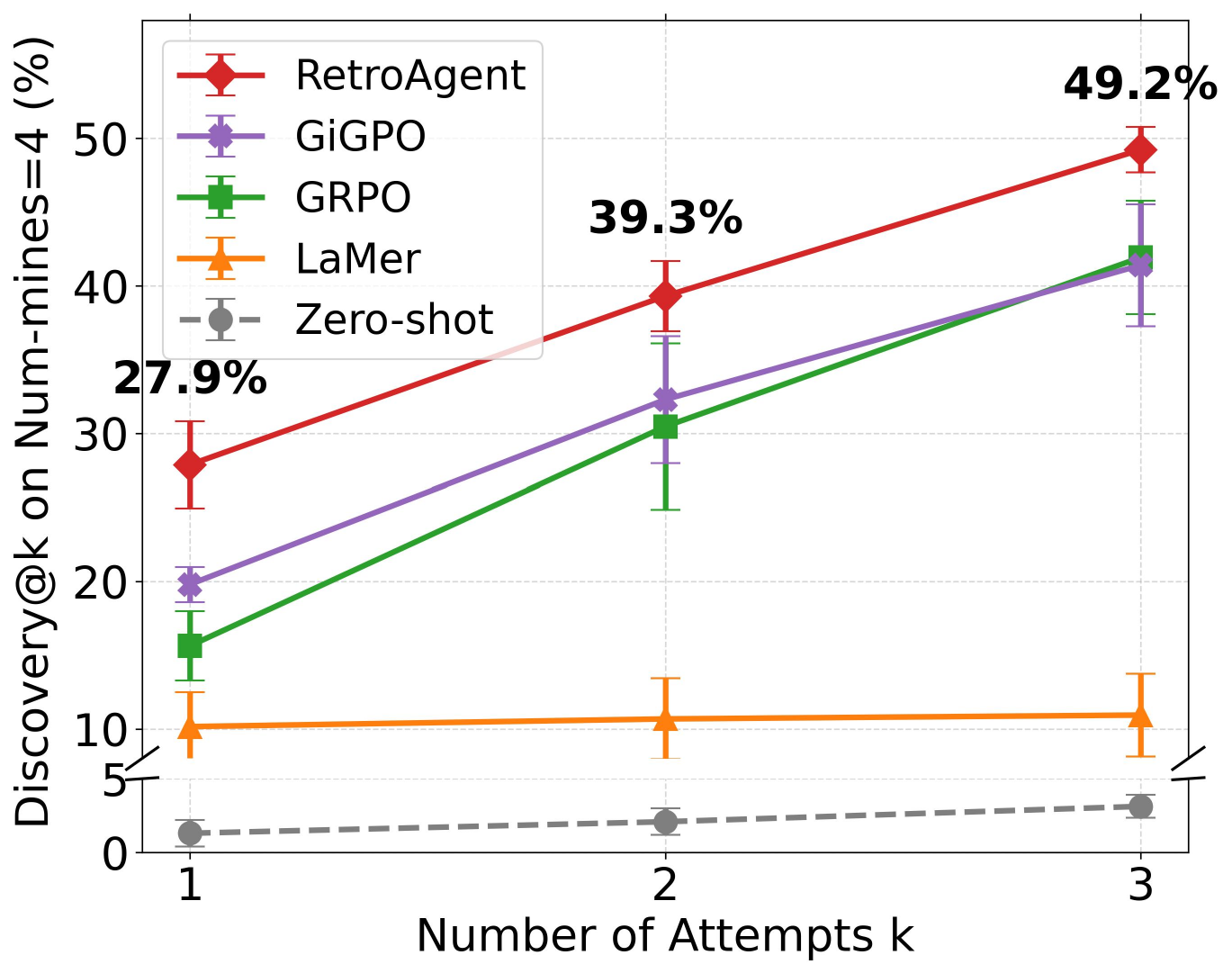}
    \caption{Test-time adaptation using $\text{Discovery}@k$ on harder instances (trained with 3~mines, evaluated with 4~mines).}
    \label{fig:multi_attempt_adaptation_on_difficult_task}
\end{subfigure}
\hfill
\begin{subfigure}[t]{0.48\linewidth}
    \centering
    \includegraphics[width=\linewidth]{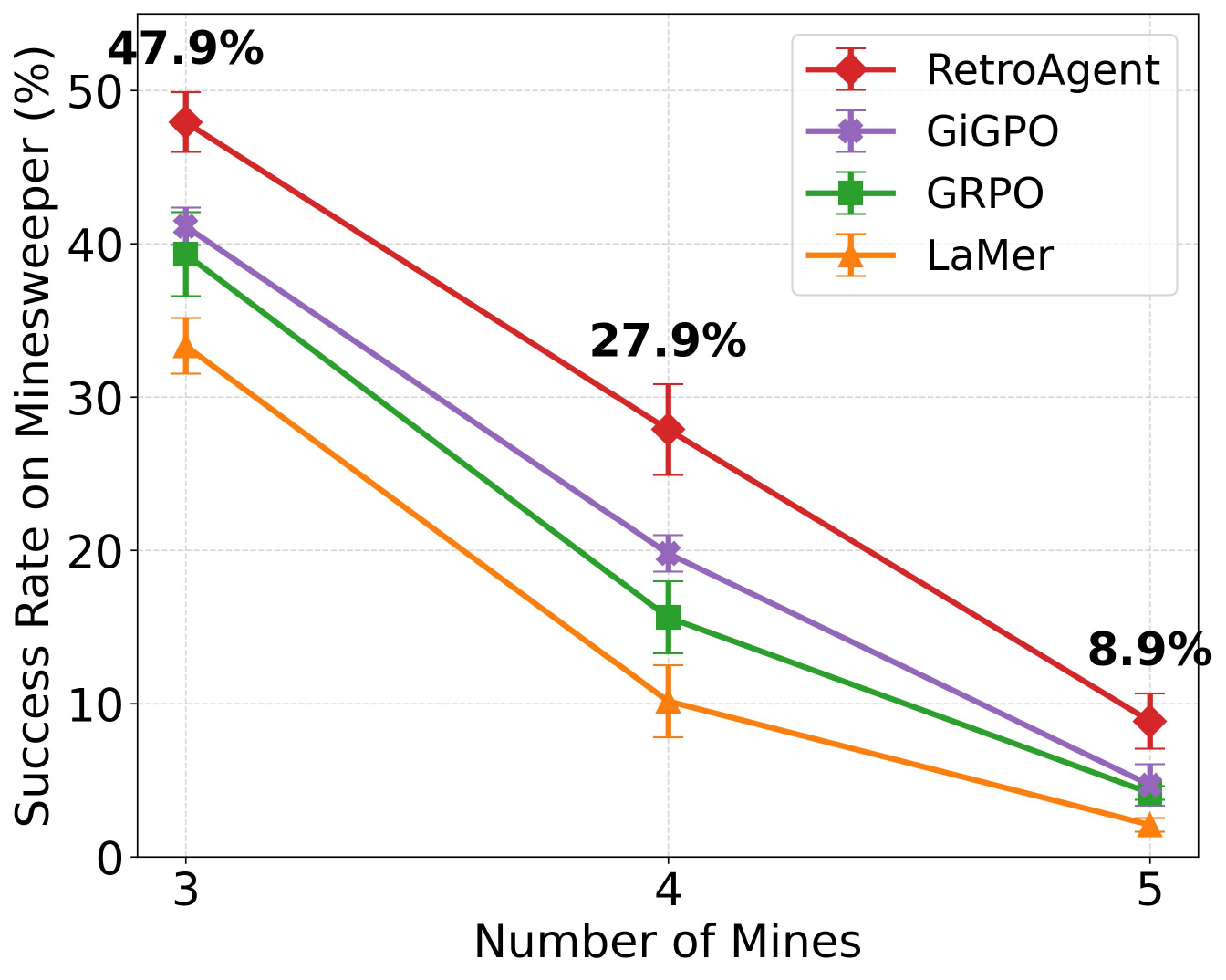}
    \caption{Generalization across increasing difficulty levels (evaluated with the number of mines ranging from 3 to 5).}
    \label{fig:adaptation_on_increasing_difficulty}
\end{subfigure}
\caption{Robustness to challenging tasks on MineSweeper.}
\label{fig:generalization_on_hard_tasks}
\end{figure*}

\paragraph{Robustness to Challenging Tasks.}
Following~\citet{jiang2025metarlinducesexplorationlanguage}, we assess robustness on MineSweeper under two harder-than-training settings (Figure~\ref{fig:generalization_on_hard_tasks}): \((\RN{1})\) increasing the mine count from \(3\) to \(4\), and \((\RN{2})\) varying the mine count from \(3\) to \(5\).

\noindent \textbf{\retroagent{} remains robust under increased difficulty.}
\retroagent{} consistently outperforms all baselines in both settings, showing rapid adaptation to harder instances (Figure~\ref{fig:multi_attempt_adaptation_on_difficult_task}) and graceful degradation as task difficulty increases (Figure~\ref{fig:adaptation_on_increasing_difficulty}).

\subsection{Analysis of In-Context Self-Reflection}
\label{sec:analysis_self_reflection}
\begin{figure*}[th]
\centering
\begin{subfigure}[t]{0.48\linewidth}
\centering
\includegraphics[width=\linewidth]{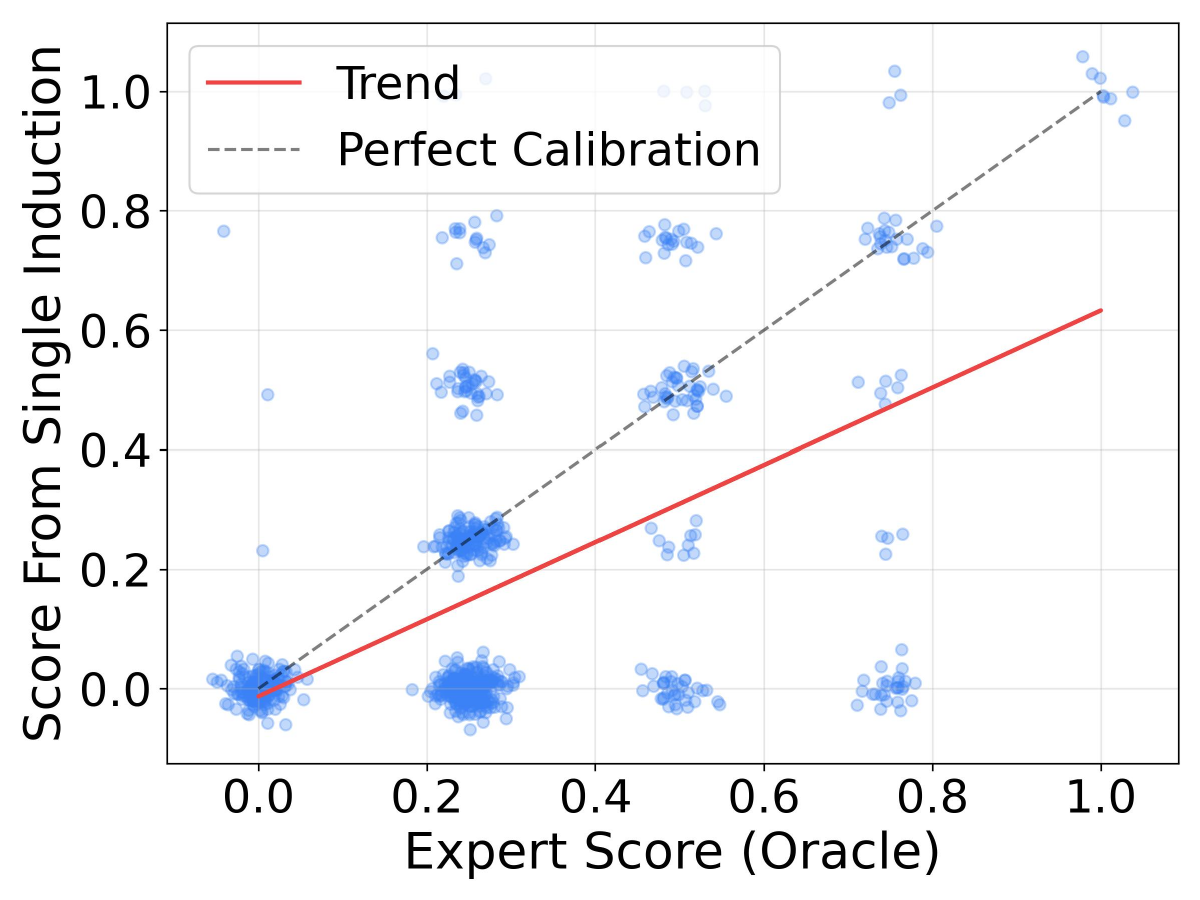}
\caption{Completion scores via single induction.}
\label{fig:reward_calibration_single_induction}
\end{subfigure}%
\hfill
\begin{subfigure}[t]{0.48\linewidth}
\centering
\includegraphics[width=\linewidth]{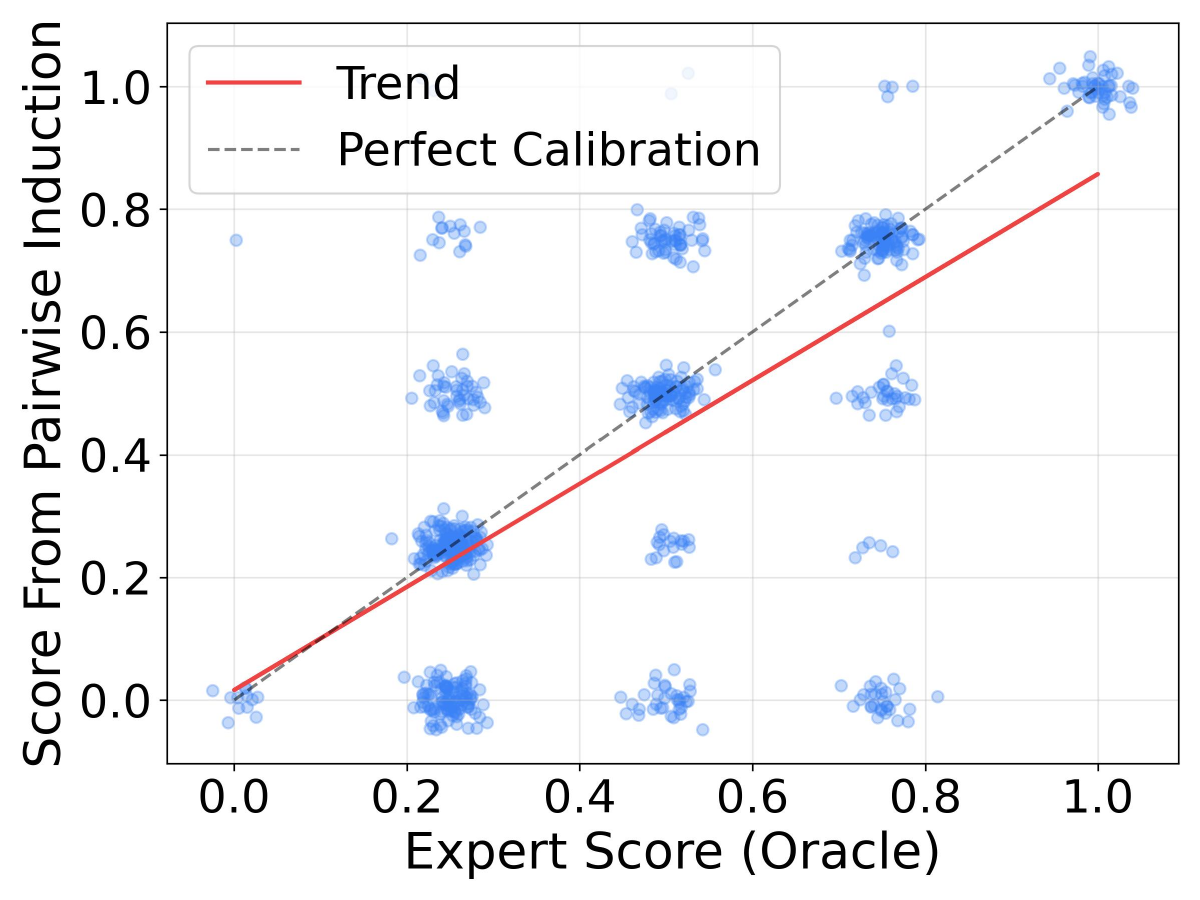}
\caption{Completion scores via pairwise induction.}
\label{fig:reward_calibration_pairwise_induction}
\end{subfigure}

\caption{Accuracy of subtask completion scores generated via single-trajectory (single) \textit{vs.}\ pairwise-trajectory  (pairwise) induction for \qwentwoseven{} on WebShop.}
\label{fig:comparison_intrinsic_reward_calibration}
\end{figure*}




\begin{table*}[th]
\caption{Quality of lessons (\ie memories) generated via single-trajectory \textit{vs.}\ pairwise-trajectory induction, as assessed by GPT-4o.}
\label{tab:reflection_quality_comparison}
\centering
\begin{threeparttable}
\small
\setlength\tabcolsep{4pt} 
\resizebox{\textwidth}{!}{
\begin{tabular}{lcccccccc}
\toprule
\multirow{3}{*}{\textbf{Method}} & 
\multicolumn{2}{c}{\textbf{Hallucination Rate (\%)}} & 
\multicolumn{6}{c}{\textbf{Estimated Utility Score (\%)}} \\
\cmidrule(lr){2-3} \cmidrule(lr){4-9} 
& \multirow{3}{*}{\makecell[c]{Failure \\ ($\downarrow$)}} & \multirow{3}{*}{\makecell[c]{Success \\ ($\downarrow$)}} & \multicolumn{3}{c}{Failure} & \multicolumn{3}{c}{Success} \\ \cmidrule(lr){4-6} \cmidrule(lr){7-9} 
&&& Low ($\downarrow$) & Med ($-$) & High ($\uparrow$) & Low ($\downarrow$) & Med ($-$) & High ($\uparrow$) \\ 
\midrule
Single Induction   & 8.8 & 15.1 & 8.8 & 78.2 & 12.9 & 12.2 & 75.6 & 12.2 \\ 
Pairwise Induction & \textbf{3.8} & \textbf{11.9} & \textbf{3.1} & 76.7 & \textbf{20.1} & \textbf{6.2} & 76.2 & \textbf{17.6}   \\
\bottomrule  
\end{tabular}
}
\end{threeparttable}
\end{table*}


\begin{table*}[th]
\caption{Effect of induction method and augmentation ratio on GRPO performance.
Augmentation Ratio denotes the fraction of sampled trajectories per prompt that receive
memory-augmented generation; the remaining trajectories are sampled without augmentation.}
\label{tab:comparison_single_pairwise_memory}
\centering
\begin{threeparttable}
\small
\setlength\tabcolsep{4pt}
\begin{tabular}{lccc}
\toprule
\textbf{Method} & \multirow{2}{*}{\makecell[c]{\textbf{Augmentation}\\\textbf{Ratio}}}  &
\multicolumn{2}{c}{\textbf{WebShop}} \\
\cmidrule(lr){3-4}
& & Task Score (\%) & Success Rate (\%) \\
\midrule
GRPO                & ---                  & 75.5$_{\pm 3.6}$ & 66.9$_{\pm 1.2}$ \\
\; + Single Induction  & 100\% (Full Group)   & 81.3$_{\pm 2.6}$ & 70.3$_{\pm 2.1}$ \\
\; + Pairwise Induction & 100\% (Full Group)   & 82.3$_{\pm 1.3}$ & 72.9$_{\pm 1.6}$ \\
\; + Pairwise Induction & \phantom{0}50\% (Half Group) & $\mathbf{82.4}_{\pm 2.9}$ & $\mathbf{75.3}_{\pm 4.3}$ \\
\bottomrule
\end{tabular}
\end{threeparttable}
\end{table*}

The effectiveness of \retroagent{} depends on the quality of self-reflection, which determines both the accuracy of intrinsic numerical feedback and the utility of intrinsic language feedback. We compare single-trajectory and pairwise-trajectory induction under the in-context self-reflection mechanism.

For numerical feedback, we use GPT-4o subtask completion scores~\citep{openai2024gpt4ocard} as oracle estimates and measure correlation with each induction method. For language feedback, we use GPT-4o to evaluate lesson quality. We further assess downstream impact by augmenting GRPO with lessons from each method, retrieved by semantic relevance to the task prompt (Table~\ref{tab:comparison_single_pairwise_memory}). Details are provided in Appendix~\ref{app:eval_pairwise_induction_single_induction}.

\noindent \textbf{Pairwise induction yields more accurate and useful self-reflection.}
Pairwise-trajectory induction produces intrinsic numerical feedback that better matches oracle subtask completion scores, as shown by its closer alignment with the dashed oracle curve in Figure~\ref{fig:comparison_intrinsic_reward_calibration}. It also generates higher-quality language feedback, with lower hallucination rates and higher estimated utility (Table~\ref{tab:reflection_quality_comparison}). These improvements translate into better downstream optimization: GRPO augmented with pairwise-induction lessons outperforms the single-induction variant (\(72.9\%\) vs.\ \(70.3\%\); Table~\ref{tab:comparison_single_pairwise_memory}).

\noindent \textbf{Preserving unguided exploration prevents premature convergence.}
Table~\ref{tab:comparison_single_pairwise_memory} shows that half-group memory augmentation outperforms full-group augmentation (\(75.3\%\) vs.\ \(72.9\%\) success). This indicates that applying memory-guided generation to all sampled trajectories reduces diversity and can promote premature convergence to suboptimal strategies.

\subsection{Effect of Intrinsic Numerical Feedback}
\label{sec:impact_intrinsic_rewards}
\begin{table*}[t]
\caption{Impact of discounted returns and intrinsic reward shaping on GRPO. Capability-evolution rewards denote the intrinsic numerical feedback described in Section~\ref{sec:intrinsic_numerical}.}
\label{tab:impact_intrinsic_rewards}
\centering
\begin{threeparttable}
\small
\setlength\tabcolsep{2pt}
\resizebox{\textwidth}{!}{
\begin{tabular}{lcccc}
\toprule
\multirow{2}{*}{\textbf{Method}} &
\multirow{2}{*}{\makecell[c]{\textbf{Discounted}\\\textbf{Returns}}} &
\multirow{2}{*}{\makecell[c]{\textbf{Reward}\\\textbf{Type}}} &
\multicolumn{2}{c}{\textbf{WebShop}} \\
\cmidrule(lr){4-5}
& & & Task Score (\%) & Success Rate (\%) \\
\midrule
GRPO (Baseline) & -- & Extrinsic & $75.5_{\pm 3.6}$ & $66.9_{\pm 1.2}$ \\
GRPO & \checkmark & Extrinsic & $84.2_{\pm 0.2}$ & $74.7_{\pm 2.7}$ \\
\quad + Progress-Guided Rewards & \checkmark & Extrinsic & $84.2_{\pm 1.7}$ & $75.0_{\pm 3.1}$ \\
\quad + Capability-Evolution Rewards & \checkmark & Extrinsic \& Intrinsic & $\mathbf{88.2}_{\pm 2.1}$ & $\mathbf{79.7}_{\pm 3.1}$ \\
\bottomrule
\end{tabular}
}
\end{threeparttable}
\end{table*}


\begin{figure*}[th]
\centering
\begin{subfigure}[t]{0.48\linewidth}
    \centering
    \includegraphics[width=\linewidth]{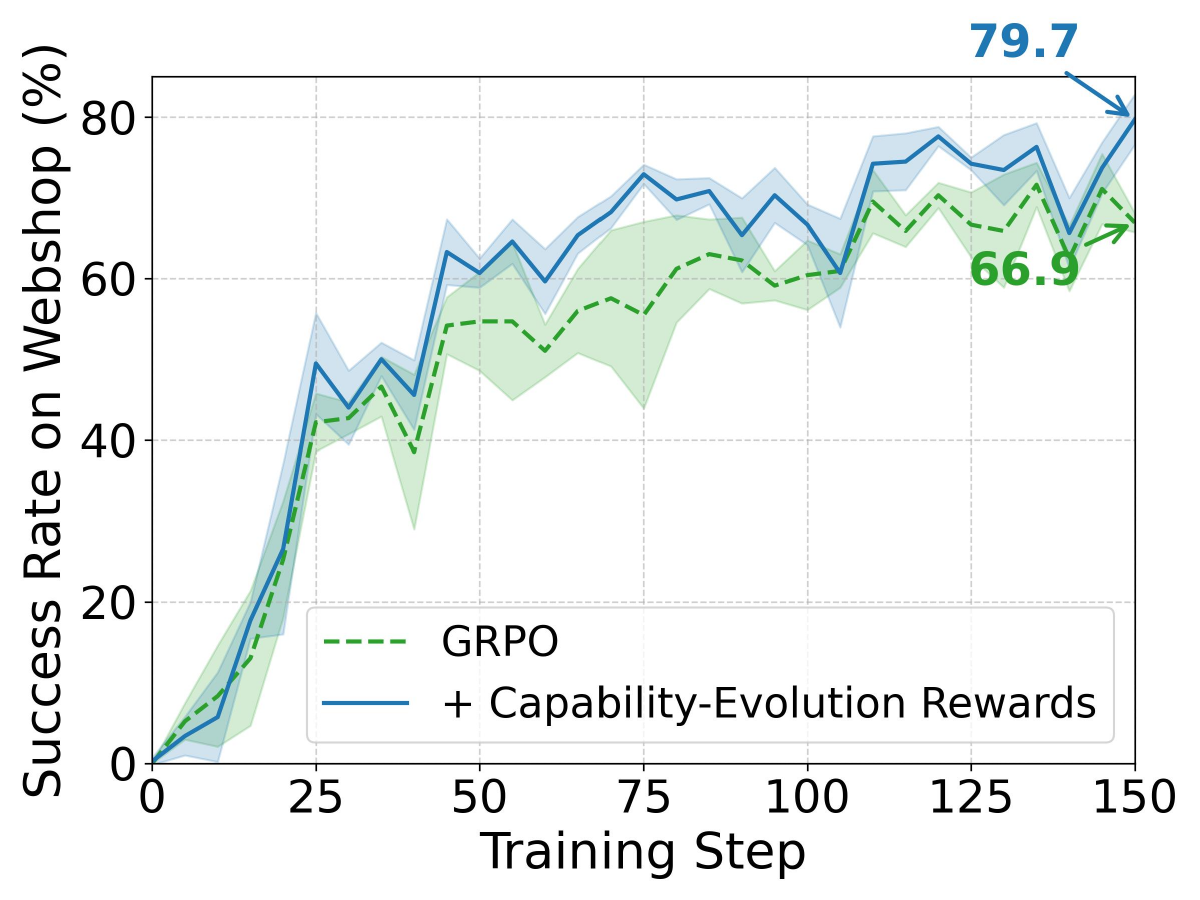}
    \caption{Impact of capability-evolution rewards.}
    \label{fig:intrinsic_rewards}
\end{subfigure}
\hfill
\begin{subfigure}[t]{0.48\linewidth}
    \centering
    \includegraphics[width=\linewidth]{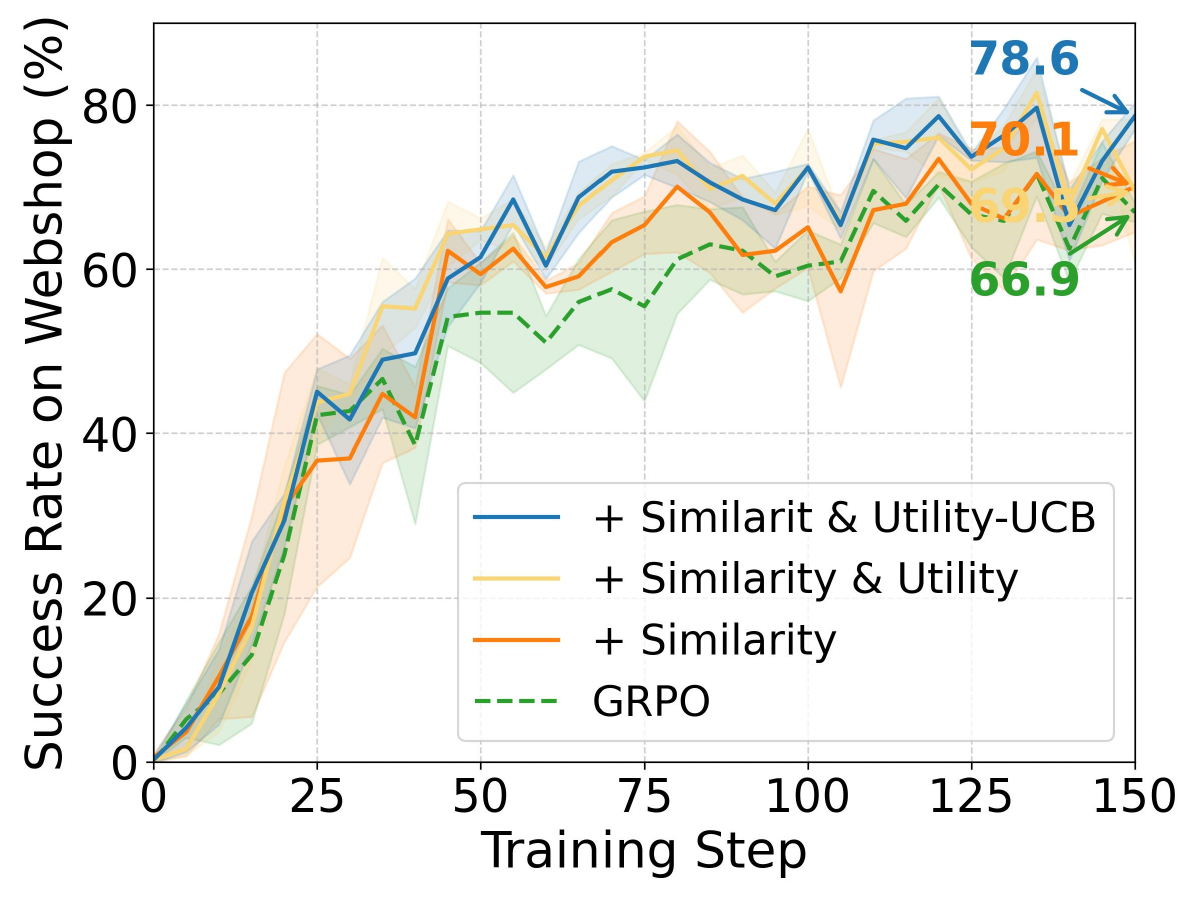}
    \caption{Impact of memory-retrieval strategies.}
    \label{fig:memory_comparison}
\end{subfigure}
\caption{Valid-set performance dynamics on WebShop when augmenting GRPO with intrinsic numerical feedback (a) or intrinsic language feedback (b).}
\label{fig:training_dynamics_individual_feedback}
\end{figure*}

We study the effects of discounted returns and intrinsic reward shaping on GRPO. Table~\ref{tab:impact_intrinsic_rewards} reports evaluation results, and Figure~\ref{fig:intrinsic_rewards} shows validation dynamics. As an additional baseline, we consider \emph{progress-guided rewards}, which replace the potential score \(\phi_{(x,\tau)}\) in Equation~\ref{eq:intrinsic_reward} with the binary environment success score \(I^{\mathrm{Ext}}\), thereby grounding the rectified gain in extrinsic outcomes rather than self-assessed progress.

\noindent \textbf{Intrinsic numerical feedback improves policy optimization.}
As shown in Table~\ref{tab:impact_intrinsic_rewards}, using discounted returns to compute trajectory-level advantages improves GRPO by \(+8.7\) percentage points in task score and \(+7.8\) in success rate on WebShop. Adding capability-evolution rewards further increases performance to \(88.2\%\) task score and \(79.7\%\) success rate, with consistent gains from step~25 onward (Figure~\ref{fig:intrinsic_rewards}). Capability-evolution rewards also outperform progress-guided rewards, indicating that self-reflection-based potential scores provide richer shaping signals than binary extrinsic outcomes alone.

\subsection{Effect of Intrinsic Language Feedback}
\label{sec:impact_intrinsic_language_feedback}
\begin{table*}[t]
\caption{Impact of intrinsic language feedback on GRPO using different memory-retrieval strategies. SimUtil-UCB denotes the our proposed memory retrieval strategy (Section~\ref{sec:intrinsic_language}).}
\label{tab:comparison_retrieved_memory}
\centering
\begin{threeparttable}
\small
\setlength\tabcolsep{2pt} 
\begin{tabular}{lcccc}
\toprule
    \multirow{2}{*}{\textbf{Method}}  & \multirow{2}{*}{\makecell[c]{\textbf{Discounted}\\\textbf{Returns}}}  & \multirow{2}{*}{\makecell[c]{\textbf{Retrieval}\\\textbf{Strategy}}}  & 
\multicolumn{2}{c}{\textbf{WebShop}}  \\
\cmidrule(lr){4-5} 
& & & Task Score (\%) & Success Rate (\%) \\ 
\midrule
GRPO (Baseline) & -- & -- & $75.5_{\pm 3.6}$ & $66.9_{\pm 1.2}$ \\
GRPO & \checkmark & -- & $84.2_{\pm 0.2}$ & $74.7_{\pm 2.7}$ \\
\quad + Memory Retrieval & \checkmark  & Similarity & $79.1_{\pm 7.1}$ & $70.1_{\pm 5.5}$ \\
\quad + Memory Retrieval & \checkmark  & Similarity \& Utility & $78.4_{\pm 11.4}$ & $69.5_{\pm 8.7}$  \\
\quad + Memory Retrieval & \checkmark  & SimUtil-UCB & $\mathbf{86.4 \pm 1.8}$ & $\mathbf{78.6 \pm 1.6}$ \\
\bottomrule  
\end{tabular}
\end{threeparttable}
\end{table*}

Having shown in Section~\ref{sec:analysis_self_reflection} that intrinsic language feedback improves policy optimization, we evaluate SimUtil-UCB against two ablations: similarity-based retrieval (Criterion~1 only) and similarity--utility retrieval (Criteria~1--2, without the exploration bonus). Table~\ref{tab:comparison_retrieved_memory} reports final performance, and Figure~\ref{fig:memory_comparison} shows validation dynamics. All experiments use half-group memory augmentation.

\noindent \textbf{Balancing relevance, utility, and exploration is critical.}
As shown in Table~\ref{tab:comparison_retrieved_memory}, discounted returns improve GRPO, but adding memories retrieved by similarity-only or similarity--utility criteria degrades performance. This contrasts with Table~\ref{tab:comparison_single_pairwise_memory}, where similarity-based retrieval improves standard GRPO without discounted returns, suggesting that discounted returns may amplify low-quality memory-guided behaviors. In contrast, SimUtil-UCB consistently improves performance, reaching an \(86.4\%\) task score and a \(78.6\%\) success rate (Table~\ref{tab:comparison_retrieved_memory}; Figure~\ref{fig:memory_comparison}). By adding the UCB exploration bonus (Equation~\ref{eq:ucb}), SimUtil-UCB avoids over-reliance on semantically similar or historically high-utility memories and promotes broader use of under-explored lessons.

Figure~\ref{fig:reflection_counts} further analyzes retrieval frequency. Each memory instance is initialized with count \(1\), which increments after retrieval. SimUtil-UCB distributes access more evenly, with most instances retrieved about \(5\) times (Figure~\ref{fig:similarity_utility_ucb_reflection_counts}), whereas similarity-based retrieval concentrates on a small subset, with many instances exceeding \(15\) retrievals (Figure~\ref{fig:similarity_reflection_counts}). This confirms that the UCB bonus diversifies memory usage and contributes to stronger final performance.
\begin{figure*}[th]
    \centering
    
    \begin{subfigure}[th]{0.32\linewidth}
        \centering
        \includegraphics[width=\linewidth]{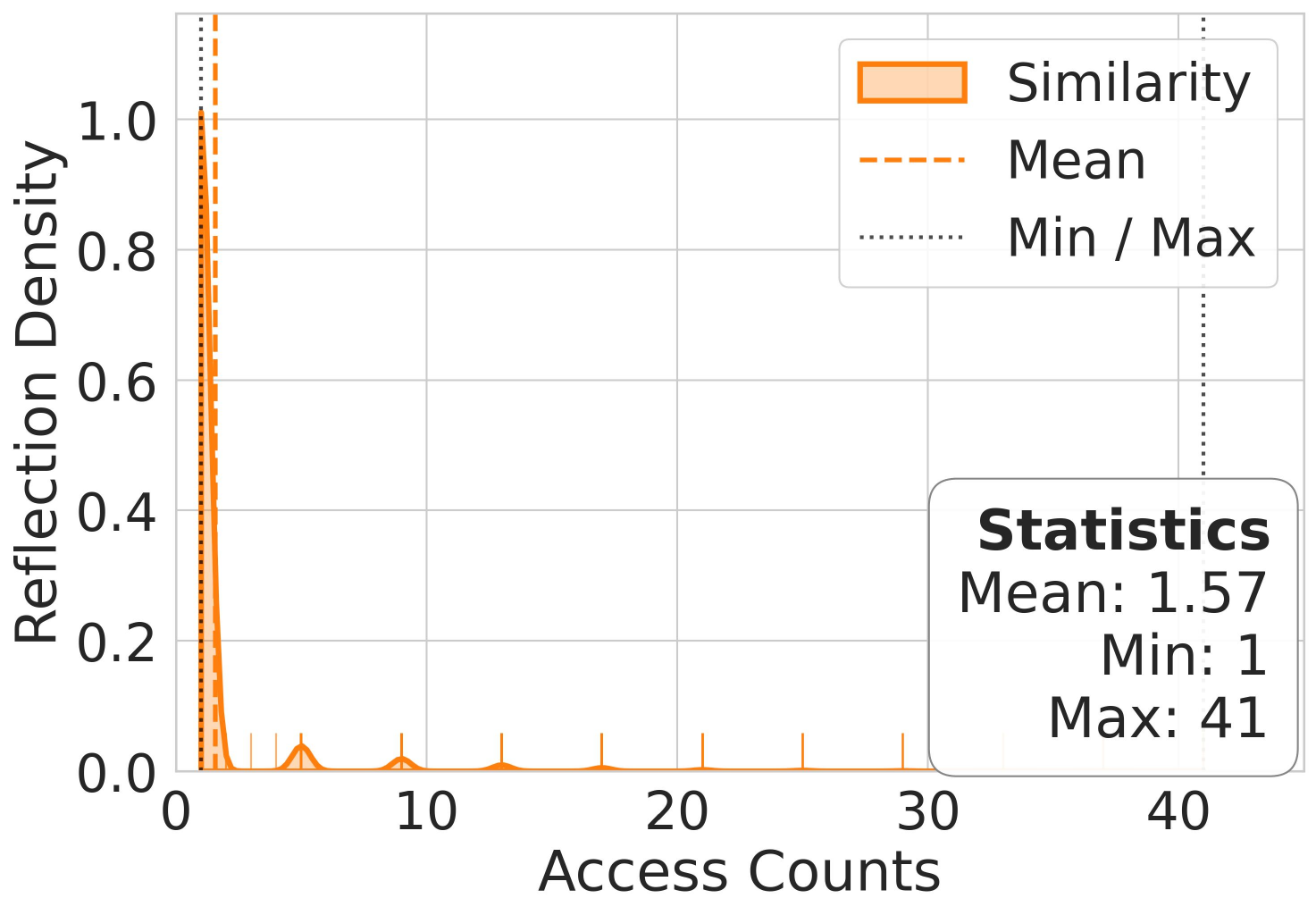}
        \caption{Similarity-based retrieval.}
        \label{fig:similarity_reflection_counts}
    \end{subfigure}
    \hfill
    \begin{subfigure}[th]{0.32\linewidth}
        \centering
        \includegraphics[width=\linewidth]{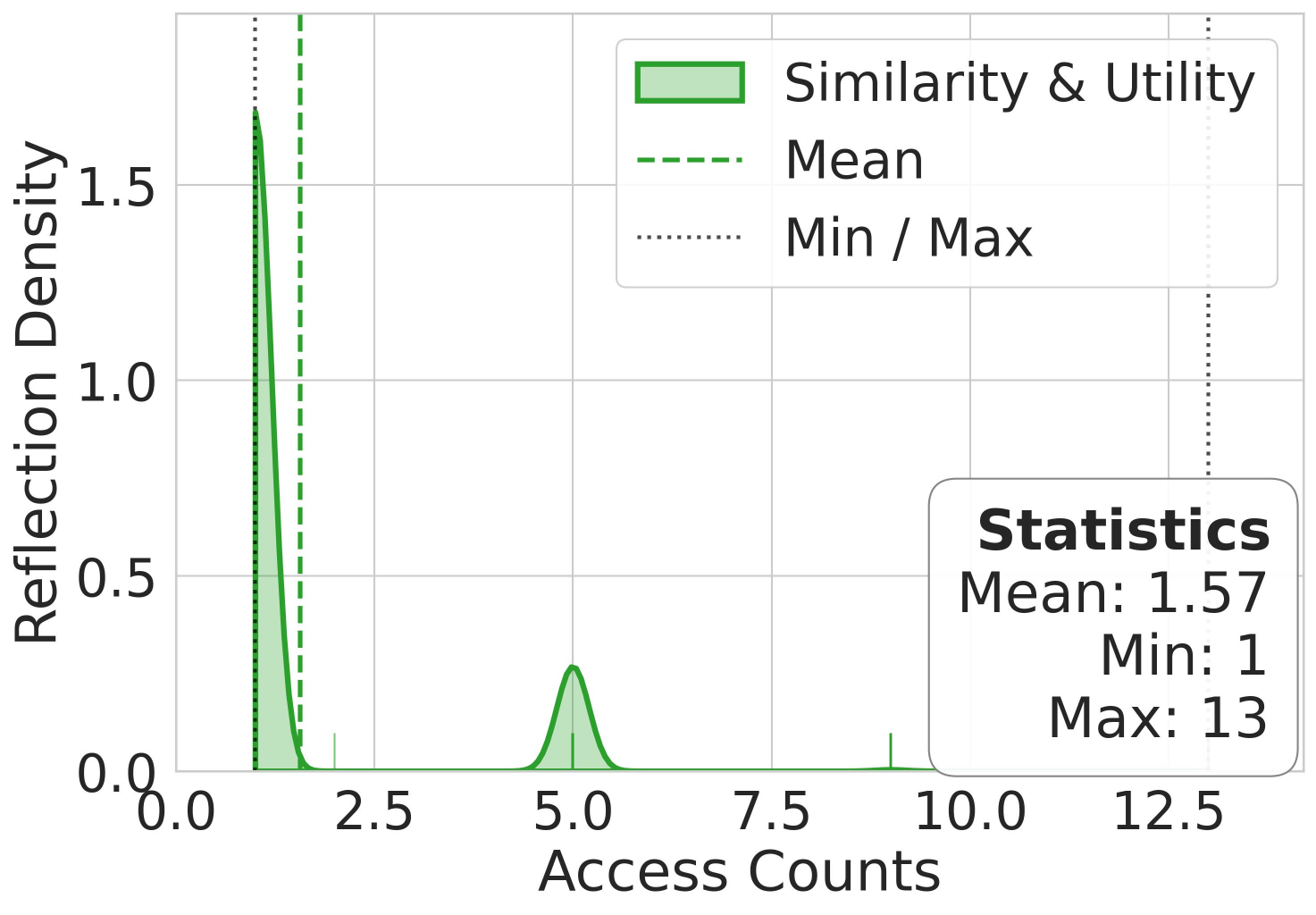}
        \caption{Similarity \& utility.}
        \label{fig:similarity_utility_reflection_counts}
    \end{subfigure}
    \hfill
    \begin{subfigure}[th]{0.32\linewidth}
        \centering
        \includegraphics[width=\linewidth]{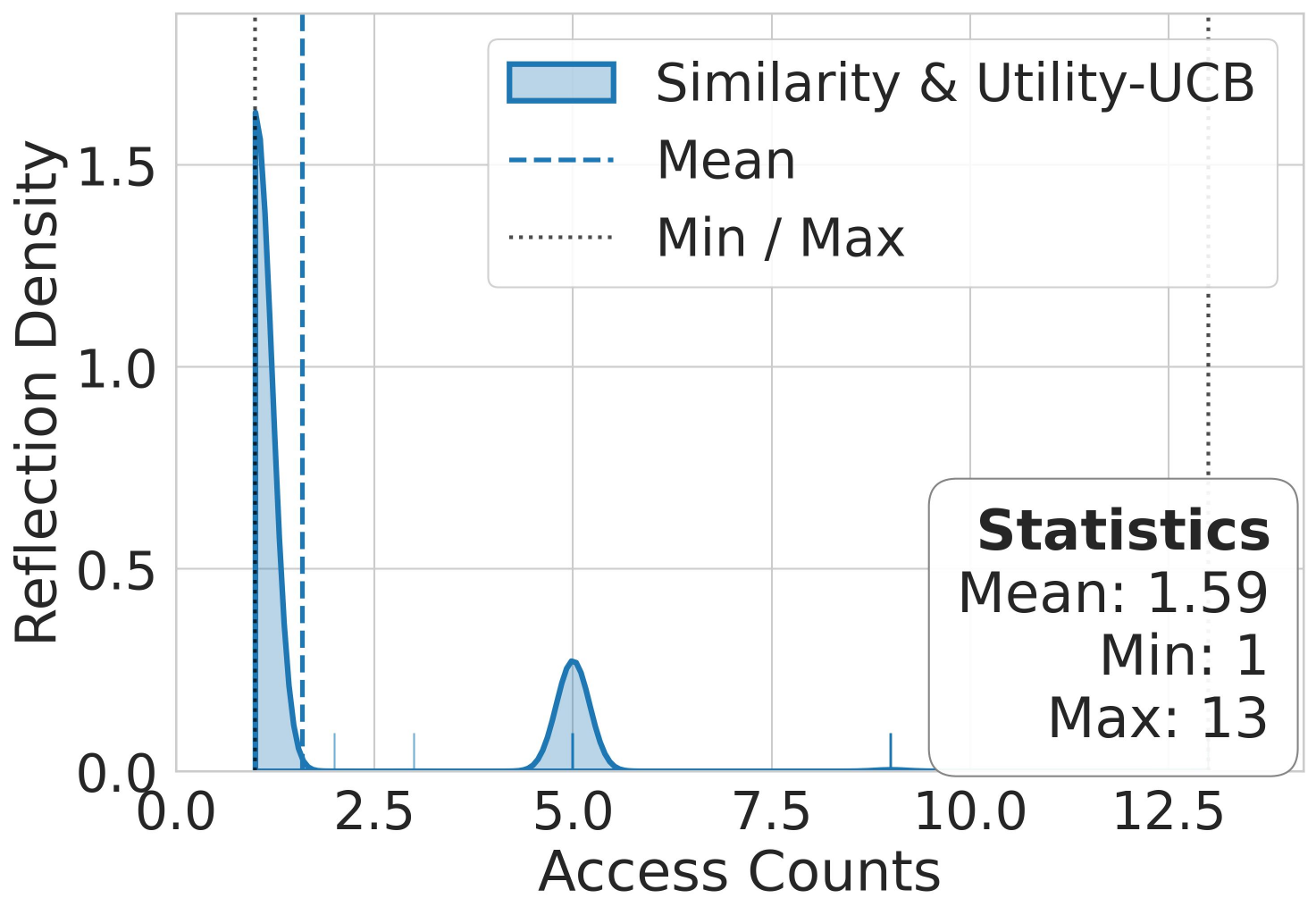}
        \caption{SimUtil-UCB.}
        \label{fig:similarity_utility_ucb_reflection_counts}
    \end{subfigure}
    
    \caption{Distribution of accumulated memory usage counts across retrieval strategies on WebShop, estimated via kernel density estimation (KDE)~\citep{chen2017tutorial}. Each panel shows how frequently stored memory instances are accessed under a given strategy.}
    \label{fig:reflection_counts}
\end{figure*}

\subsection{Effect of Combining Dual Intrinsic Feedback}
\label{sec:analysis_dual_intrinsic_feedback}
\begin{table*}[th]
\caption{Individual and combined effects of intrinsic numerical and language feedback under different self-reflection mechanisms on WebShop. Rows above the dashed line ablate each feedback type in isolation; rows below combine both (Dual).}
\label{tab:impact_combining_dual_intrinsic_feedback}
\centering
\begin{threeparttable}
\small
\setlength{\tabcolsep}{2pt}
\resizebox{\textwidth}{!}{
\begin{tabular}{lcccc}
\toprule
\multirow{2.5}{*}{\textbf{Method}} & \multirow{2.5}{*}{\makecell[c]{\textbf{Intrinsic} \\ \textbf{Feedback}}} & \multirow{2.5}{*}{\makecell[c]{\textbf{Self-Reflection} \\ \textbf{Mechanism}}} & \multicolumn{2}{c}{\textbf{WebShop}} \\
\cmidrule(lr){4-5}
& & & Task Score (\%) & Success Rate (\%) \\
\midrule
GRPO (Baseline) & -- & -- & $75.5 _{\pm3.6}$ & $66.9 _{\pm1.2}$ \\
\quad + Capability-Evolution Rewards & Numerical & Pairwise Induction & $88.2 _{\pm2.1}$ & $79.7 _{\pm3.1}$ \\
\quad + SimUtil-UCB Memory Retrieval & Language & Pairwise Induction & $86.4 _{\pm1.8}$ & $78.6 _{\pm1.6}$ \\
\hdashline
\noalign{\vskip 0.07cm} 
\textsc{RetroAgent} (In-Context) & Dual & Pairwise Induction & $87.6_{\pm 2.1}$ & $78.9_{\pm 3.6}$ \\
\textsc{RetroAgent} (RL-Trained) & Dual & Pairwise Induction & $87.0_{\pm 1.4}$ & $77.1_{\pm 1.0}$ \\
\textsc{RetroAgent} (RL-Trained) & Dual & Single Induction & $\mathbf{88.9_{\pm 1.3}}$ & $\mathbf{82.3_{\pm 1.6}}$ \\

\bottomrule
\end{tabular}
}
\end{threeparttable}
\end{table*}

\begin{figure*}[t]
\centering
\begin{subfigure}[t]{0.48\linewidth}
    \centering
    \includegraphics[width=\linewidth]{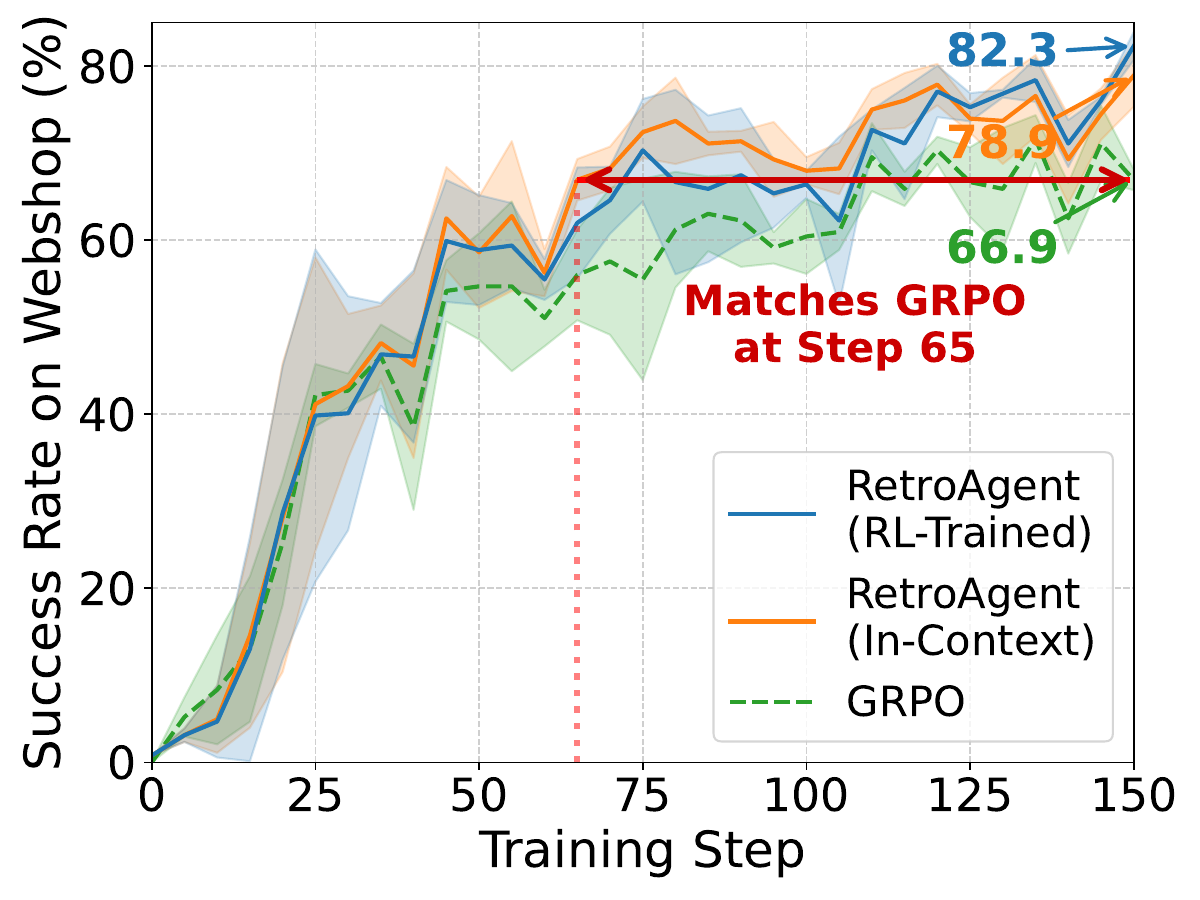}
    \caption{Valid-set performance over the course of training.}
    \label{fig:training_dynamics_reflection}
\end{subfigure}
\hfill
\begin{subfigure}[t]{0.48\linewidth}
    \centering
    \includegraphics[width=\linewidth]{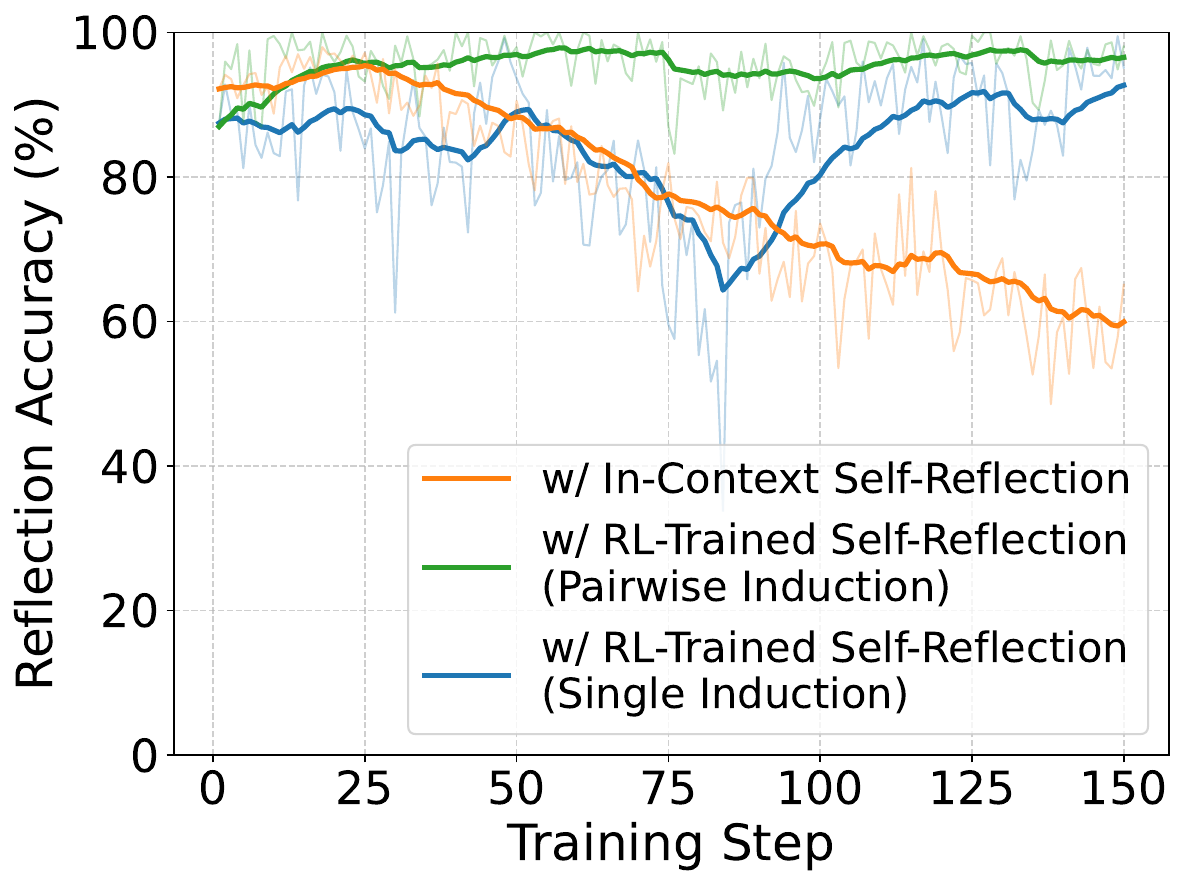}
    \caption{Reflection accuracy over the course of training, smoothed with exponential moving average (EMA)~\citep{klinker2011exponential}.}
    \label{fig:comparison_reflection_accuracy}
\end{subfigure}
\caption{In-context \textit{vs.} RL-trained self-reflection mechanisms in \textsc{RetroAgent} on WebShop.}
\label{fig:comparison_reflection_mechanisms}
\end{figure*}
We present results for combining intrinsic numerical and language feedback in
Table~\ref{tab:impact_combining_dual_intrinsic_feedback} and compare in-context versus
RL-trained reflection mechanisms in Figure~\ref{fig:comparison_reflection_mechanisms}.

\noindent \textbf{Combining dual intrinsic feedback facilitates superior agentic reasoning.} As shown in Table~\ref{tab:impact_combining_dual_intrinsic_feedback}, \textsc{RetroAgent} achieves notable performance gains (\eg $\approx +3\%$ success rate) by integrating dual intrinsic feedback compared to using either capability-evolution rewards or SimUtil-UCB memory retrieval in isolation. The
in-context variant, however, slightly underperforms GRPO with capability-evolution rewards
only, suggesting that simultaneous exploration signals from both feedback channels might
interfere with each other during action selection.

\noindent \textbf{Joint optimization preserves reflection quality and improves RL training.} In Figure~\ref{fig:comparison_reflection_accuracy}, the reflection accuracy of the
in-context variant declines steadily as the policy improves (orange curve), even though
extrinsic success signals remain available. In contrast, the RL-trained self-reflection
mechanism maintains accuracy throughout training (blue curve). Although accuracy dips
slightly before step~75---likely because decision-making policy improvement temporarily outpaces reflection
adaptation---it recovers and increases steadily thereafter. The initial gap relative to the
in-context baseline arises because the RL-trained variant uses single induction, which is
less informative than pairwise induction (consistent with
Section~\ref{sec:analysis_self_reflection}).

We validate the choice of single induction by comparing it against a pairwise variant that
conditions on a reference trajectory:
$z = f_{\text{reflect}}(\tau_{\text{ref}}, \tau)$. Although including
$\tau_{\text{ref}}$ yields the highest reflection accuracy (green curve,
Figure~\ref{fig:comparison_reflection_accuracy}), it does not improve task performance
(Table~\ref{tab:impact_combining_dual_intrinsic_feedback}). This discrepancy suggests that
contrastive comparison enables the reflector to infer outcomes from relative differences
between trajectories rather than developing robust standalone evaluation capability.

\subsection{Effect of Intrinsic Feedback on Exploration}
\label{sec:impact_exploration}
Both intrinsic numerical and language feedback are designed to improve RL by guiding exploration: capability-evolution rewards encourage promising action sequences, while retrieved lessons discourage repeated failures and reinforce effective strategies. We test this hypothesis on the WebShop test set under three settings: \((\RN{1})\) GRPO with capability-evolution rewards only, \((\RN{2})\) GRPO with SimUtil-UCB memory retrieval only, and \((\RN{3})\) \retroagent{} with in-context or RL-trained self-reflection. We quantify trajectory diversity using the Vendi Score~\citep{friedman2023vendi} over both successful and failed trajectories.

\begin{table*}[th]
\caption{Impact of intrinsic feedback on trajectory diversity on WebShop, measured by the
Vendi Score~\citep{friedman2023vendi}. A score of 0.00 for \qwentwoseven{} indicates that fewer than two successful trajectories were generated, precluding diversity measurement.}
\label{tab:exploration_intrinsic_feedback}
\centering
\begin{threeparttable}
\small
\resizebox{\textwidth}{!}{
\setlength{\tabcolsep}{8pt}
\begin{tabular}{lccc}
    \toprule
    \multirow{2}{*}{\textbf{Method}}
        & \multirow{2}{*}{\makecell[c]{\textbf{Intrinsic} \\ \textbf{Feedback}}}
        & \multicolumn{2}{c}{\textbf{Vendi Score ($\uparrow$)}} \\
    \cmidrule(lr){3-4}
    & & Successful Traj. & Failed Traj. \\
    \midrule
    \qwentwoseven{} & -- & 0.00\tnote{*} & 1.89 \\
    GRPO (Baseline) & -- & 1.85 & 1.71 \\
    \quad + Capability-Evolution Rewards & Numerical & 2.04 & 1.82 \\
    \quad + SimUtil-UCB Memory Retrieval & Language & 2.13 & 1.97 \\
    \noalign{\vskip 0.05cm}
    \hdashline
    \noalign{\vskip 0.07cm}
    \textsc{RetroAgent} (In-Context Self-Reflection) & Dual & 2.01 & 1.78 \\
    \textsc{RetroAgent} (RL-Trained Self-Reflection) & Dual & \textbf{2.20} & \textbf{1.94} \\
    \bottomrule
\end{tabular}
}
\end{threeparttable}
\end{table*}

\noindent \textbf{Intrinsic feedback promotes useful exploration.}
All intrinsic-feedback variants obtain higher Vendi Scores on successful trajectories than GRPO, indicating more diverse successful strategies. The in-context \retroagent{} variant shows slightly lower diversity than the single-feedback ablations, suggesting that the two feedback channels may partially constrain exploration when used jointly (Table~\ref{tab:impact_combining_dual_intrinsic_feedback}).

\begin{figure*}[th]
\centering
\begin{subfigure}[t]{0.48\linewidth}
    \centering
    \includegraphics[width=\linewidth]{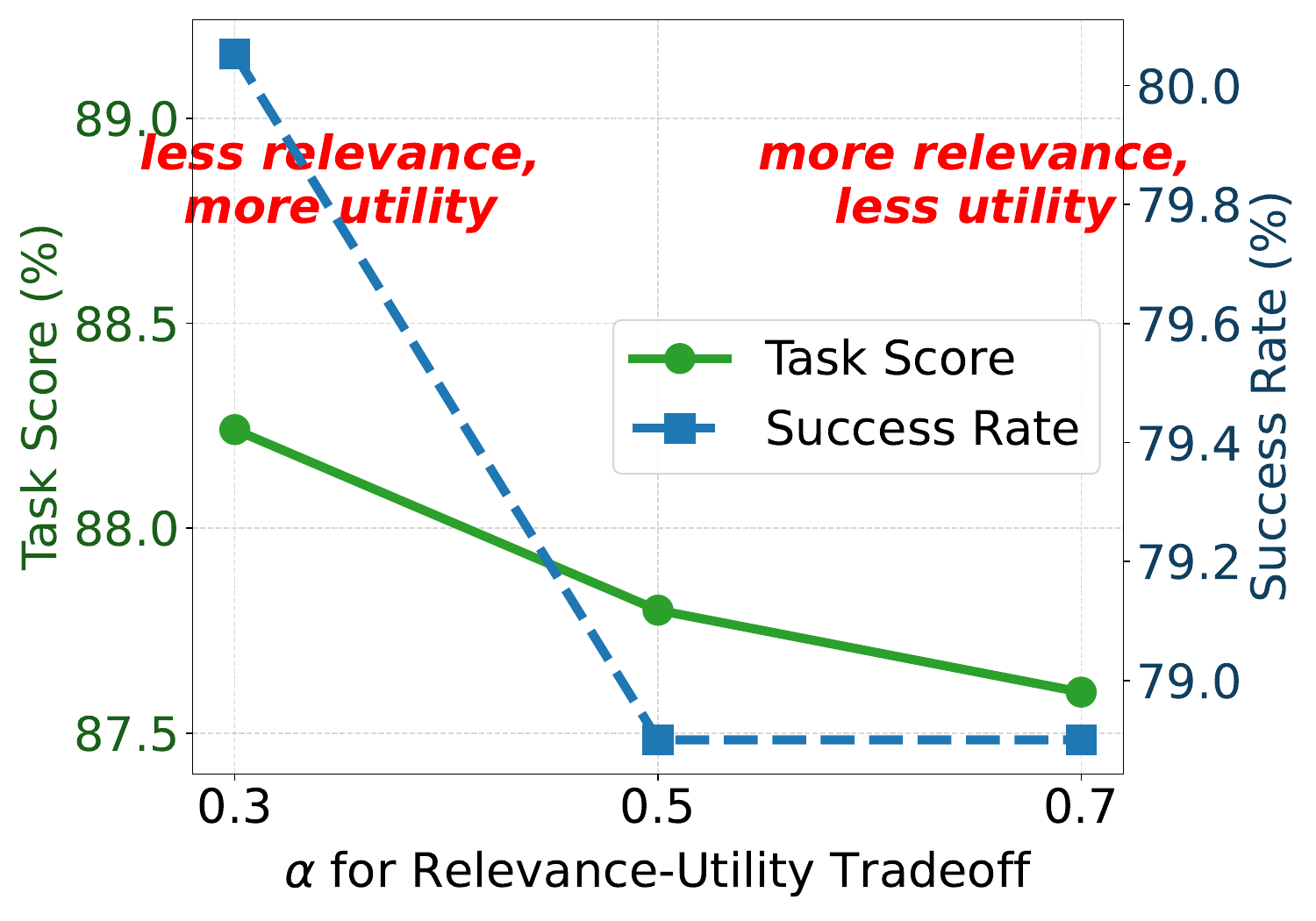}
    \caption{Impact of the relevance--utility tradeoff coefficient~\(\alpha\) on \retroagent{} (in-context self-reflection).}
    \label{fig:relevance_utility_tradeoff}
\end{subfigure}
\hfill
\begin{subfigure}[t]{0.48\linewidth}
    \centering
    \includegraphics[width=\linewidth]{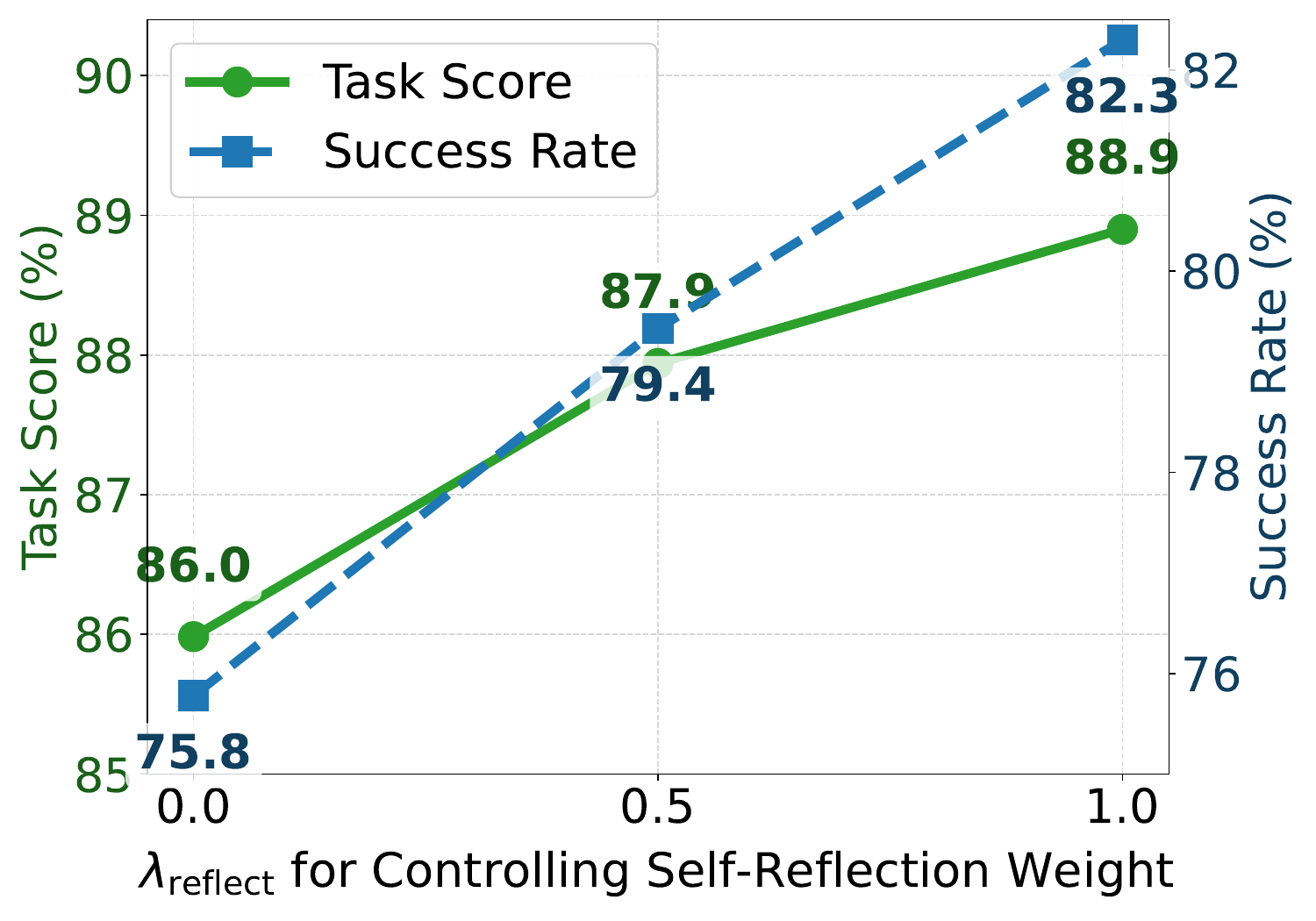}
    \caption{Impact of the self-reflection objective weight~\(\lambda_{\text{reflect}}\) on \retroagent{} (RL-trained self-reflection).}
    \label{fig:reflect_loss_weight}
\end{subfigure}
\caption{Sensitivity of \retroagent{} to two key coefficients on the
WebShop test set, reported in terms of task score and success rate
(averaged over three runs).}
\label{fig:coefficient_changes}
\end{figure*}



\subsection{Training Efficiency}
\label{sec:training_efficiency}
\begin{wrapfigure}{r}{0.48\linewidth}
    \centering
    \includegraphics[width=\linewidth]{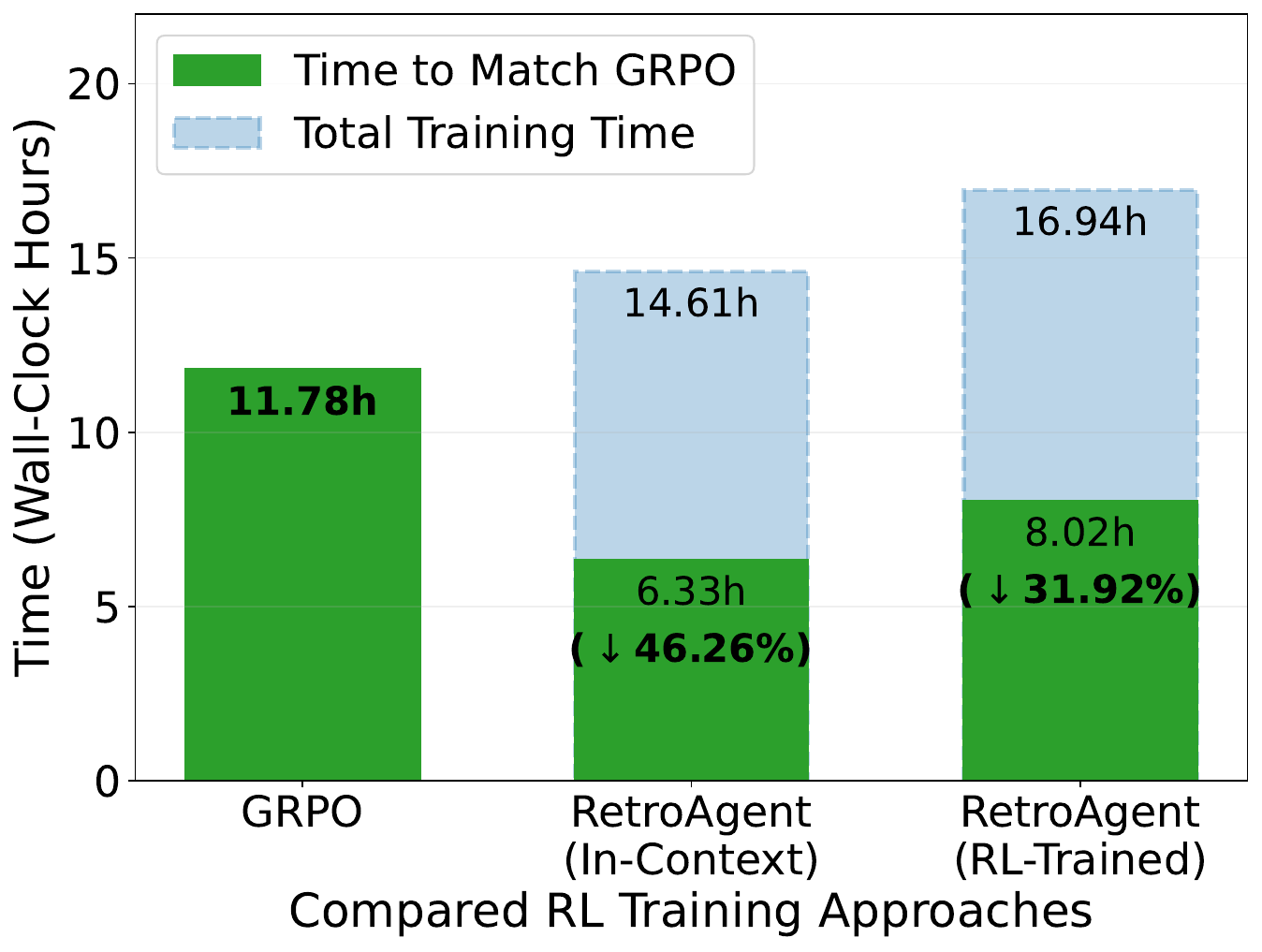}
    \caption{Training time (wall-clock hours) on WebShop. ``Time to Match GRPO'' denotes
    the time required for each \textsc{RetroAgent} variant to reach the peak performance of
    the GRPO baseline.}
    \label{fig:training_efficiency_comparison}
\end{wrapfigure}
We evaluate training efficiency by comparing \retroagent{} with the GRPO baseline in terms of training time (Figure~\ref{fig:training_efficiency_comparison}).

\noindent \textbf{Intrinsic feedback accelerates convergence.}
Although \retroagent{} incurs higher total training cost than GRPO, it reaches GRPO's peak performance substantially earlier. The in-context variant matches the GRPO peak at step~65, while the RL-trained variant does so at step~73 (Figure~\ref{fig:training_dynamics_reflection}), reducing the time required to reach this performance by \(46\%\) and \(32\%\), respectively. The RL-trained variant converges slightly more slowly, likely due to the additional optimization of the reflection objective.

\subsection{Sensitivity to the Relevance--Utility Tradeoff Coefficient}
\label{sec:relevance_utility_tradeoff}
We examine the relevance--utility tradeoff in memory retrieval for \retroagent{} with in-context self-reflection by varying the coefficient \(\alpha\) from \(0.3\) (utility-oriented) to \(0.7\) (relevance-oriented). As shown in Figure~\ref{fig:relevance_utility_tradeoff}, \retroagent{} achieves higher WebShop task scores and success rates when utility is prioritized (\(\alpha=0.3\)). This highlights the importance of accounting for historical utility, rather than relying solely on semantic relevance, when retrieving experiential memories.

\subsection{Sensitivity to the Self-Reflection Objective Weight}
\label{sec:self_reflection_weight}


We examine the effect of the self-reflection objective
weight~$\lambda_{\text{reflect}}$ on the final performance of
\retroagent{} with RL-trained self-reflection mechanism by varying
$\lambda_{\text{reflect}}$ from~$0$ (self-reflection loss disabled)
to~$1$. As shown in Figure~\ref{fig:reflect_loss_weight}, increasing
$\lambda_{\text{reflect}}$ consistently improves performance on
WebShop, raising the success rate from 75.8\% to 82.3\% and the task
score from 86.0\% to 88.9\%.

\subsection{Generalization Across Model Architectures}
\label{sec:generalization_scales}
\begin{table*}[th]
\caption{Performance of \textsc{RetroAgent} on Llama-3.1-8B-Instruct across four agentic benchmarks. All improvements are statistically significant ($p < 0.01$).}
\label{tab:llama_results}
\centering
\setlength{\tabcolsep}{2pt}
\resizebox{\textwidth}{!}{
\begin{tabular}{lccccc}
\toprule
\multirow{2}{*}{\textbf{Method}} & \textbf{ALFWorld} & \multicolumn{2}{c}{\textbf{WebShop}} & \textbf{Sokoban} & \textbf{MineSweeper} \\
\cmidrule(lr){2-2} \cmidrule(lr){3-4} \cmidrule(lr){5-5} \cmidrule(lr){6-6}
& Success (\%) & Score (\%) & Success (\%) & Success (\%) & Success (\%) \\ 
\midrule
Llama-3.1-8B-Instruct (Zero-shot) & $29.2 _{\pm0.9}$ & $0.2 _{\pm0.4}$ & $0.1 _{\pm0.2}$ & $5.7_{\pm0.5}$ & $7.0_{\pm0.8}$ \\
GRPO (Baseline) & $72.7 _{\pm2.3}$ & $78.0 _{\pm2.3}$ & $67.6 _{\pm2.8}$ & $12.2_{\pm1.2}$ & $42.4_{\pm2.5}$ \\
\textsc{LaMer}~\citep{jiang2025metarlinducesexplorationlanguage} & $76.0 _{\pm1.8}$ & - & $70.3 _{\pm3.6}$ & $15.9 _{\pm2.4}$ & $32.0 _{\pm3.4}$ \\
GiGPO~\citep{feng2025groupingroup} & $90.9_{\pm3.6}$ & $87.8_{\pm2.3}$ & $77.7_{\pm3.9}$ & $13.5_{\pm1.2}$ & $48.2_{\pm2.0}$ \\
\textsc{RetroAgent} (In-Context) & $\mathbf{93.1_{\pm1.5}}$ & $87.8_{\pm1.8}$ & $71.9_{\pm3.6}$ & $\mathbf{39.1_{\pm1.3}}$  & $52.3_{\pm1.6}$\\
\textsc{RetroAgent} (RL-Trained) & $91.4_{\pm1.4}$  & $\mathbf{89.5_{\pm2.1}}$ & $\mathbf{80.5_{\pm2.2}}$ & $24.5_{\pm2.8}$ & $\mathbf{59.9_{\pm3.2}}$  \\
\bottomrule
\end{tabular}
}
\end{table*}



\noindent \textbf{\retroagent{} generalizes across model architectures.}
To assess architectural generality, we evaluate \retroagent{} with Llama-3.1-8B-Instruct~\citep{grattafiori2024llama3herdmodels}. As shown in Table~\ref{tab:llama_results}, \retroagent{} consistently achieves state-of-the-art performance across all four tasks. The RL-trained self-reflection variant slightly underperforms the in-context variant on ALFWorld and Sokoban, likely due to interference between reflection and decision-making objectives during joint optimization. Specifically, the auxiliary reflection loss may weaken the primary policy-gradient signal. We leave improved multi-objective balancing to future work.


\subsection{Scaling Across Model Sizes}
\label{sec:scales}



\begin{figure*}[ht]
    \centering
    \begin{subfigure}[t]{0.48\linewidth}
        \centering
        \includegraphics[width=\linewidth]{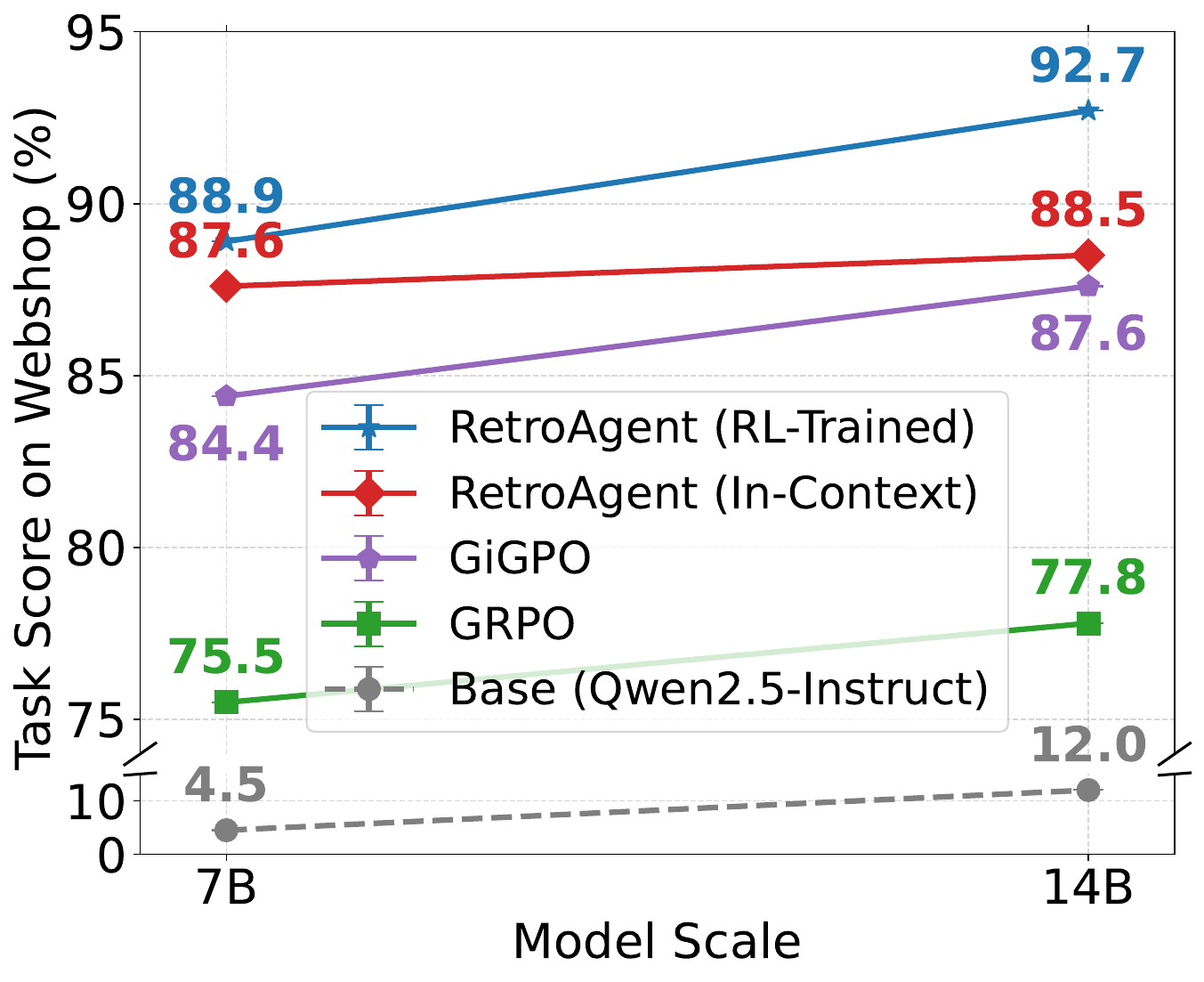}
        \caption{Task score comparison.}
        \label{fig:task_score_model_scales}
    \end{subfigure}
    \hfill
    \begin{subfigure}[t]{0.48\linewidth}
        \centering
        \includegraphics[width=\linewidth]{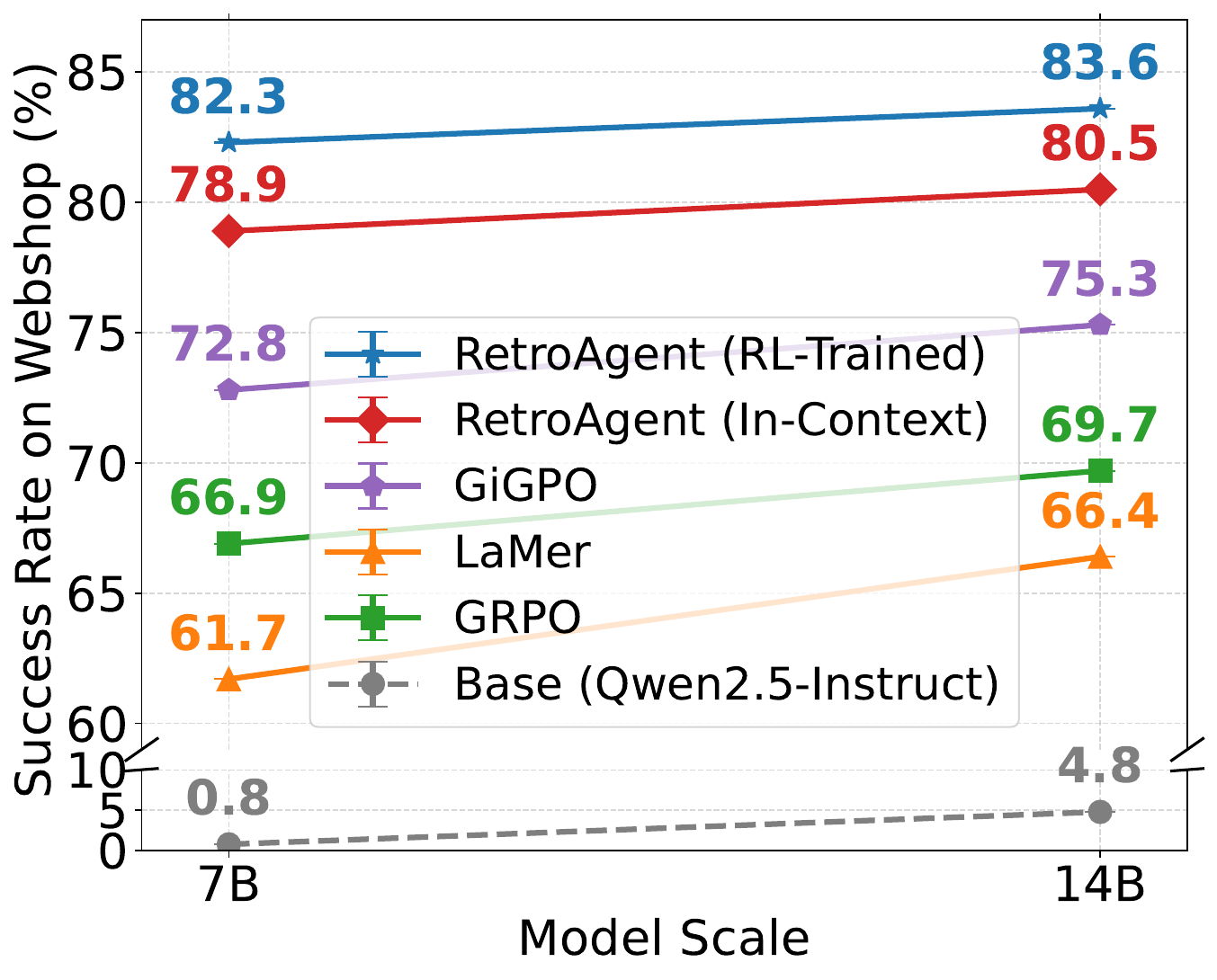}
        \caption{Success rate comparison.}
        \label{fig:success_rate_model_scales}
    \end{subfigure}
    \caption{Performance comparison of various methods across different model scales.}
    \label{fig:performance_model_scales}
\end{figure*}

\noindent \textbf{\retroagent{} generalizes across model scales.}
We evaluate \retroagent{} on WebShop using Qwen2.5-Instruct models at 7B and 14B scales. As shown in Figure~\ref{fig:performance_model_scales}, \retroagent{} consistently outperforms competitive baselines at both scales. Scaling from 7B to 14B yields modest gains: task score improves by \(+0.9\%\)--\(+3.8\%\), and success rate by \(+1.3\%\)--\(+1.6\%\). These results suggest that larger models can better encode and exploit experiential knowledge, but that model scaling alone provides limited benefits when the compared models share similar pretraining distributions and thus similar underlying reasoning capabilities.

\subsection{Qualitative Analysis}
\label{sec:qualitative_analysis}
\begin{figure}[th]
    \centering
    \includegraphics[width=\linewidth]{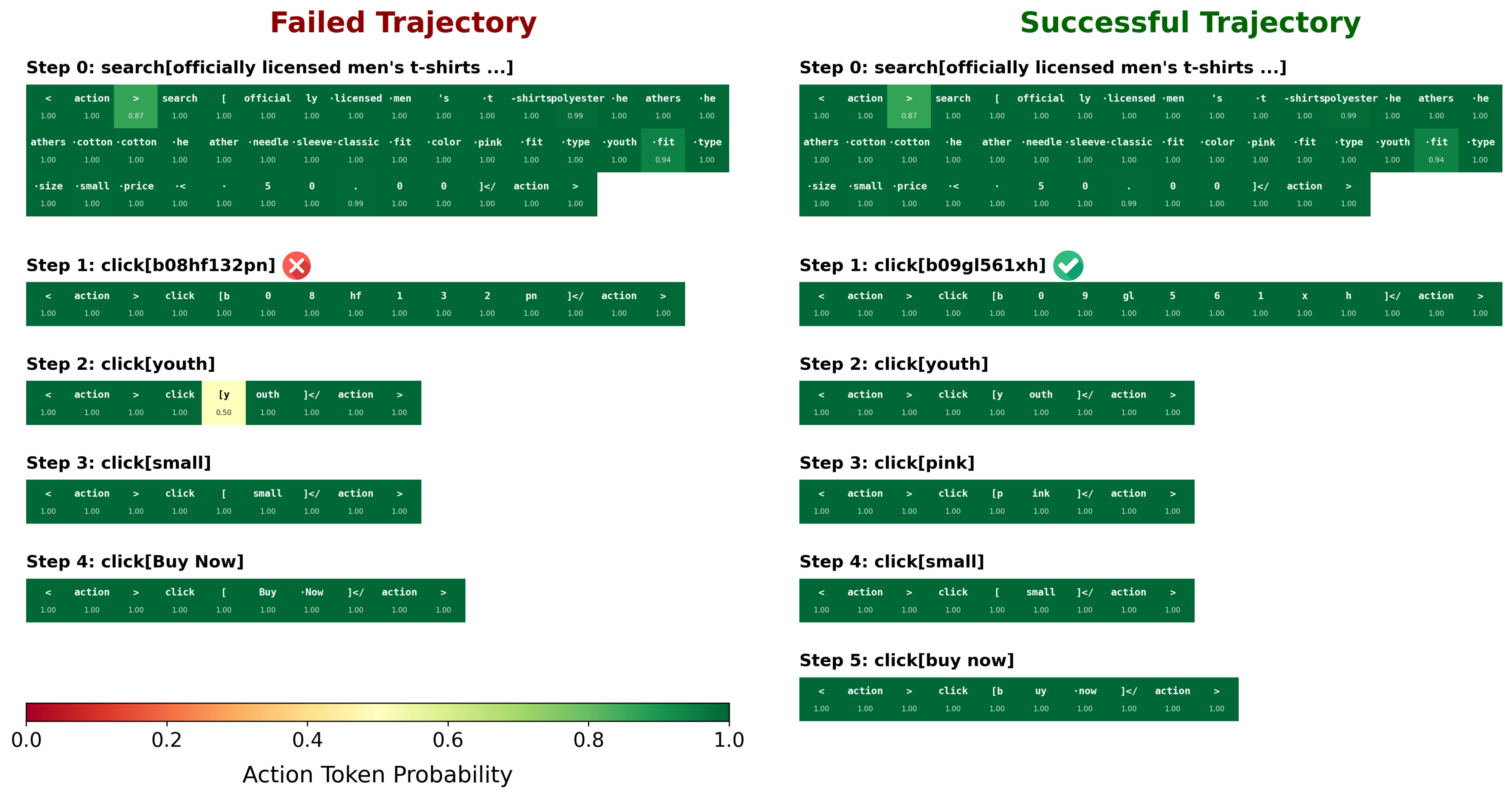}
    \caption{Qualitative comparison of \retroagent{} (in-context
    self-reflection) on the WebShop validation set between training
    step~65 (failed trajectory, left) and training step~150 (successful
    trajectory, right). For conciseness, only action tokens and their
    corresponding probabilities are shown at each decision step.}
    \label{fig:compared_traj_token_prob}
\end{figure}
We qualitatively examine \retroagent{}'s policy evolution by analyzing how lessons distilled from similar past tasks guide decision-making during training. We compare an early failed trajectory from \retroagent{} with in-context self-reflection at step~65 and a later successful trajectory at step~150 on the WebShop validation set. As shown in Figure~\ref{fig:compared_traj_token_prob}, at step~65, \retroagent{} selects an incorrect item at decision Step~1, fails to choose the required pink variant, and assigns low token-level confidence to the correct category ``youth.'' In contrast, at step~150, \retroagent{} retrieves relevant lessons from memory and selects the correct item with the required attributes at higher confidence. Full trajectories are provided in Appendix~\ref{app:generated_trajectories}.

\section{Conclusion}

We present \retroagent{}, an online RL framework that bridges one-off task solving and continuous adaptation. Through a hindsight self-reflection mechanism, \retroagent{} generates dual intrinsic feedback: $(\RN{1})$ intrinsic numerical feedback that rewards promising exploration by tracking incremental subtask completion, and $(\RN{2})$ intrinsic language feedback that distills reusable lessons into a memory buffer. This memory is retrieved via SimUtil-UCB, which balances relevance, utility, and exploration to leverage prior experience effectively. By jointly learning from extrinsic task-success rewards and retrospective dual intrinsic feedback, \retroagent{} enables efficient experiential learning. Experiments across four diverse agentic tasks show that \retroagent{} consistently achieves SOTA performance while exhibiting strong test-time adaptation and out-of-distribution generalization. These results suggest that dual intrinsic feedback is a promising direction for building continuously adaptive agents. Future work includes developing more effective multi-objective optimization strategies for jointly training self-reflection and decision-making, and extending \retroagent{} to multi-agent and open-ended settings.

\clearpage


\section*{Ethics Statement}
Throughout this research, we have adhered to ethical guidelines that prioritize privacy, fairness, and the responsible development of autonomous agents. All experiments are conducted on publicly available interactive decision-making benchmarks (ALFWorld, WebShop, Sokoban, and MineSweeper) within fully simulated, sandboxed environments. Our study does not involve human subjects, personally identifiable information, or private data, thereby safeguarding user privacy by design. The proposed intrinsic feedback mechanisms---capability-evolution rewards and memory-based language feedback---operate entirely within these simulated environments and do not interact with real users or real-world systems. All training is performed using open-source base models, ensuring transparency and reproducibility. These measures were implemented to minimize potential negative impacts on users and society.


\bibliography{main}

@inproceedings{
liu2025understanding,
title={Understanding R1-Zero-Like Training: A Critical Perspective},
author={Zichen Liu and Changyu Chen and Wenjun Li and Penghui Qi and Tianyu Pang and Chao Du and Wee Sun Lee and Min Lin},
booktitle={Second Conference on Language Modeling},
year={2025},
url={https://openreview.net/forum?id=5PAF7PAY2Y}
}

@inproceedings{
tajwar2025training,
title={Training a Generally Curious Agent},
author={Fahim Tajwar and Yiding Jiang and Abitha Thankaraj and Sumaita Sadia Rahman and J Zico Kolter and Jeff Schneider and Russ Salakhutdinov},
booktitle={Forty-second International Conference on Machine Learning},
year={2025},
url={https://openreview.net/forum?id=UeB3Hdrhda}
}

@misc{jiang2025metarlinducesexplorationlanguage,
      title={Meta-RL Induces Exploration in Language Agents}, 
      author={Yulun Jiang and Liangze Jiang and Damien Teney and Michael Moor and Maria Brbic},
      year={2025},
      eprint={2512.16848},
      archivePrefix={arXiv},
      primaryClass={cs.LG},
      url={https://arxiv.org/abs/2512.16848}, 
}

@inproceedings{yao2022react,
  title={React: Synergizing reasoning and acting in language models},
  author={Yao, Shunyu and Zhao, Jeffrey and Yu, Dian and Du, Nan and Shafran, Izhak and Narasimhan, Karthik R and Cao, Yuan},
  booktitle={The eleventh international conference on learning representations},
  year={2022}
}

@inproceedings{
feng2025groupingroup,
title={Group-in-Group Policy Optimization for {LLM} Agent Training},
author={Lang Feng and Zhenghai Xue and Tingcong Liu and Bo An},
booktitle={The Thirty-ninth Annual Conference on Neural Information Processing Systems},
year={2025},
url={https://openreview.net/forum?id=QXEhBMNrCW}
}

@inproceedings{
shinn2023reflexion,
title={Reflexion: language agents with verbal reinforcement learning},
author={Noah Shinn and Federico Cassano and Ashwin Gopinath and Karthik R Narasimhan and Shunyu Yao},
booktitle={Thirty-seventh Conference on Neural Information Processing Systems},
year={2023},
url={https://openreview.net/forum?id=vAElhFcKW6}
}

@inproceedings{park2023generative,
  title={Generative agents: Interactive simulacra of human behavior},
  author={Park, Joon Sung and O'Brien, Joseph and Cai, Carrie Jun and Morris, Meredith Ringel and Liang, Percy and Bernstein, Michael S},
  booktitle={Proceedings of the 36th annual acm symposium on user interface software and technology},
  pages={1--22},
  year={2023}
}

@article{
wang2024voyager,
title={Voyager: An Open-Ended Embodied Agent with Large Language Models},
author={Guanzhi Wang and Yuqi Xie and Yunfan Jiang and Ajay Mandlekar and Chaowei Xiao and Yuke Zhu and Linxi Fan and Anima Anandkumar},
journal={Transactions on Machine Learning Research},
issn={2835-8856},
year={2024},
url={https://openreview.net/forum?id=ehfRiF0R3a},
note={}
}

@inproceedings{
wang2024mobileagentv,
title={Mobile-Agent-v2: Mobile Device Operation Assistant with Effective Navigation via Multi-Agent Collaboration},
author={Junyang Wang and Haiyang Xu and Haitao Jia and Xi Zhang and Ming Yan and Weizhou Shen and Ji Zhang and Fei Huang and Jitao Sang},
booktitle={The Thirty-eighth Annual Conference on Neural Information Processing Systems},
year={2024},
url={https://openreview.net/forum?id=O0nBMRlkc8}
}

@article{schick2023toolformer,
  title={Toolformer: Language models can teach themselves to use tools},
  author={Schick, Timo and Dwivedi-Yu, Jane and Dess{\`\i}, Roberto and Raileanu, Roberta and Lomeli, Maria and Hambro, Eric and Zettlemoyer, Luke and Cancedda, Nicola and Scialom, Thomas},
  journal={Advances in Neural Information Processing Systems},
  volume={36},
  pages={68539--68551},
  year={2023}
}

@misc{
putta2025agent,
title={Agent Q: Advanced Reasoning and Learning for Autonomous {AI} Agents},
author={Pranav Putta and Edmund Mills and Naman Garg and Sumeet Ramesh Motwani and Elan Sopher Markowitz and Julia Kiseleva and Chelsea Finn and Divyansh Garg and Rafael Rafailov},
year={2025},
url={https://openreview.net/forum?id=LuytzzohTa}
}

@misc{shao2024deepseekmathpushinglimitsmathematical,
      title={DeepSeekMath: Pushing the Limits of Mathematical Reasoning in Open Language Models}, 
      author={Zhihong Shao and Peiyi Wang and Qihao Zhu and Runxin Xu and Junxiao Song and Xiao Bi and Haowei Zhang and Mingchuan Zhang and Y. K. Li and Y. Wu and Daya Guo},
      year={2024},
      eprint={2402.03300},
      archivePrefix={arXiv},
      primaryClass={cs.CL},
      url={https://arxiv.org/abs/2402.03300}, 
}

@article{trivedi2024appworld,
  title={Appworld: A controllable world of apps and people for benchmarking interactive coding agents},
  author={Trivedi, Harsh and Khot, Tushar and Hartmann, Mareike and Manku, Ruskin and Dong, Vinty and Li, Edward and Gupta, Shashank and Sabharwal, Ashish and Balasubramanian, Niranjan},
  journal={arXiv preprint arXiv:2407.18901},
  year={2024}
}

@article{schulman2017proximal,
  title={Proximal policy optimization algorithms},
  author={Schulman, John and Wolski, Filip and Dhariwal, Prafulla and Radford, Alec and Klimov, Oleg},
  journal={arXiv preprint arXiv:1707.06347},
  year={2017}
}

@article{friedman2023vendi,
  title={The Vendi Score: A Diversity Evaluation Metric for Machine Learning},
  author={Friedman, Dan and Dieng, Adji Bousso},
  journal={Transactions on Machine Learning Research},
  issn={2835-8856},
  year={2023}
}

@article{lee2023rlaif,
  title={Rlaif: Scaling reinforcement learning from human feedback with ai feedback},
  author={Lee, Harrison and Phatale, Samrat and Mansoor, Hassan and Lu, Kellie Ren and Mesnard, Thomas and Ferret, Johan and Bishop, Colton and Hall, Ethan and Carbune, Victor and Rastogi, Abhinav},
  year={2023}
}

@misc{liu2026simplememefficientlifelongmemory,
      title={SimpleMem: Efficient Lifelong Memory for LLM Agents}, 
      author={Jiaqi Liu and Yaofeng Su and Peng Xia and Siwei Han and Zeyu Zheng and Cihang Xie and Mingyu Ding and Huaxiu Yao},
      year={2026},
      eprint={2601.02553},
      archivePrefix={arXiv},
      primaryClass={cs.AI},
      url={https://arxiv.org/abs/2601.02553}, 
}

@misc{hübotter2026reinforcementlearningselfdistillation,
      title={Reinforcement Learning via Self-Distillation}, 
      author={Jonas Hübotter and Frederike Lübeck and Lejs Behric and Anton Baumann and Marco Bagatella and Daniel Marta and Ido Hakimi and Idan Shenfeld and Thomas Kleine Buening and Carlos Guestrin and Andreas Krause},
      year={2026},
      eprint={2601.20802},
      archivePrefix={arXiv},
      primaryClass={cs.LG},
      url={https://arxiv.org/abs/2601.20802}, 
}

@misc{fang2026mempexploringagentprocedural,
      title={Memp: Exploring Agent Procedural Memory}, 
      author={Runnan Fang and Yuan Liang and Xiaobin Wang and Jialong Wu and Shuofei Qiao and Pengjun Xie and Fei Huang and Huajun Chen and Ningyu Zhang},
      year={2026},
      eprint={2508.06433},
      archivePrefix={arXiv},
      primaryClass={cs.CL},
      url={https://arxiv.org/abs/2508.06433}, 
}

@misc{chhikara2025mem0buildingproductionreadyai,
      title={Mem0: Building Production-Ready AI Agents with Scalable Long-Term Memory}, 
      author={Prateek Chhikara and Dev Khant and Saket Aryan and Taranjeet Singh and Deshraj Yadav},
      year={2025},
      eprint={2504.19413},
      archivePrefix={arXiv},
      primaryClass={cs.CL},
      url={https://arxiv.org/abs/2504.19413}, 
}

@misc{wu2025evolverselfevolvingllmagents,
      title={EvolveR: Self-Evolving LLM Agents through an Experience-Driven Lifecycle}, 
      author={Rong Wu and Xiaoman Wang and Jianbiao Mei and Pinlong Cai and Daocheng Fu and Cheng Yang and Licheng Wen and Xuemeng Yang and Yufan Shen and Yuxin Wang and Botian Shi},
      year={2025},
      eprint={2510.16079},
      archivePrefix={arXiv},
      primaryClass={cs.CL},
      url={https://arxiv.org/abs/2510.16079}, 
}

@misc{xia2026skillrlevolvingagentsrecursive,
      title={SkillRL: Evolving Agents via Recursive Skill-Augmented Reinforcement Learning}, 
      author={Peng Xia and Jianwen Chen and Hanyang Wang and Jiaqi Liu and Kaide Zeng and Yu Wang and Siwei Han and Yiyang Zhou and Xujiang Zhao and Haifeng Chen and Zeyu Zheng and Cihang Xie and Huaxiu Yao},
      year={2026},
      eprint={2602.08234},
      archivePrefix={arXiv},
      primaryClass={cs.LG},
      url={https://arxiv.org/abs/2602.08234}, 
}

@misc{zhang2026memrlselfevolvingagentsruntime,
      title={MemRL: Self-Evolving Agents via Runtime Reinforcement Learning on Episodic Memory}, 
      author={Shengtao Zhang and Jiaqian Wang and Ruiwen Zhou and Junwei Liao and Yuchen Feng and Zhuo Li and Yujie Zheng and Weinan Zhang and Ying Wen and Zhiyu Li and Feiyu Xiong and Yutao Qi and Bo Tang and Muning Wen},
      year={2026},
      eprint={2601.03192},
      archivePrefix={arXiv},
      primaryClass={cs.CL},
      url={https://arxiv.org/abs/2601.03192}, 
}

@article{williams1992simple,
  title={Simple statistical gradient-following algorithms for connectionist reinforcement learning},
  author={Williams, Ronald J},
  journal={Machine learning},
  volume={8},
  number={3},
  pages={229--256},
  year={1992},
  publisher={Springer}
}

@article{shao2024deepseekmath,
  title={Deepseekmath: Pushing the limits of mathematical reasoning in open language models},
  author={Shao, Zhihong and Wang, Peiyi and Zhu, Qihao and Xu, Runxin and Song, Junxiao and Bi, Xiao and Zhang, Haowei and Zhang, Mingchuan and Li, YK and Wu, Yang and others},
  journal={arXiv preprint arXiv:2402.03300},
  year={2024}
}

@article{klinker2011exponential,
  title={Exponential moving average versus moving exponential average},
  author={Klinker, Frank},
  journal={Mathematische Semesterberichte},
  volume={58},
  number={1},
  pages={97--107},
  year={2011},
  publisher={Springer}
}

@misc{xu2025provablylearninglanguagefeedback,
      title={Provably Learning from Language Feedback}, 
      author={Wanqiao Xu and Allen Nie and Ruijie Zheng and Aditya Modi and Adith Swaminathan and Ching-An Cheng},
      year={2025},
      eprint={2506.10341},
      archivePrefix={arXiv},
      primaryClass={cs.LG},
      url={https://arxiv.org/abs/2506.10341}, 
}

@article{lyons2011monitoring,
  title={Monitoring, metacognition, and executive function: Elucidating the role of self-reflection in the development of self-regulation},
  author={Lyons, Kristen E and Zelazo, Philip David},
  journal={Advances in child development and behavior},
  volume={40},
  pages={379--412},
  year={2011},
  publisher={Elsevier}
}

@article{chen2017tutorial,
  title={A tutorial on kernel density estimation and recent advances},
  author={Chen, Yen-Chi},
  journal={Biostatistics \& Epidemiology},
  volume={1},
  number={1},
  pages={161--187},
  year={2017},
  publisher={Taylor \& Francis}
}

@book{sutton1998reinforcement,
  title={Reinforcement learning: An introduction},
  author={Sutton, Richard S and Barto, Andrew G and others},
  volume={1},
  number={1},
  year={1998},
  publisher={MIT press Cambridge}
}

@misc{zhou2025mementofinetuningllmagents,
      title={Memento: Fine-tuning LLM Agents without Fine-tuning LLMs}, 
      author={Huichi Zhou and Yihang Chen and Siyuan Guo and Xue Yan and Kin Hei Lee and Zihan Wang and Ka Yiu Lee and Guchun Zhang and Kun Shao and Linyi Yang and Jun Wang},
      year={2025},
      eprint={2508.16153},
      archivePrefix={arXiv},
      primaryClass={cs.LG},
      url={https://arxiv.org/abs/2508.16153}, 
}

@article{zhang2025critique,
  title={Critique-grpo: Advancing llm reasoning with natural language and numerical feedback},
  author={Zhang, Xiaoying and Zhang, Yipeng and Sun, Hao and Feng, Kaituo and Lu, Chaochao and Yang, Chao and Meng, Helen},
  journal={arXiv preprint arXiv:2506.03106},
  year={2025}
}

@article{racaniere2017imagination,
  title={Imagination-augmented agents for deep reinforcement learning},
  author={Racani{\`e}re, S{\'e}bastien and Weber, Th{\'e}ophane and Reichert, David and Buesing, Lars and Guez, Arthur and Jimenez Rezende, Danilo and Puigdom{\`e}nech Badia, Adri{\`a} and Vinyals, Oriol and Heess, Nicolas and Li, Yujia and others},
  journal={Advances in neural information processing systems},
  volume={30},
  year={2017}
}

@inproceedings{li2024assessing,
  title={Assessing logical puzzle solving in large language models: Insights from a minesweeper case study},
  author={Li, Yinghao and Wang, Haorui and Zhang, Chao},
  booktitle={Proceedings of the 2024 Conference of the North American Chapter of the Association for Computational Linguistics: Human Language Technologies (Volume 1: Long Papers)},
  pages={59--81},
  year={2024}
}

@inproceedings{DBLP:conf/nips/Yao0YN22,
  author       = {Shunyu Yao and
                  Howard Chen and
                  John Yang and
                  Karthik Narasimhan},
  editor       = {Sanmi Koyejo and
                  S. Mohamed and
                  A. Agarwal and
                  Danielle Belgrave and
                  K. Cho and
                  A. Oh},
  title        = {WebShop: Towards Scalable Real-World Web Interaction with Grounded
                  Language Agents},
  booktitle    = {Advances in Neural Information Processing Systems 35: Annual Conference
                  on Neural Information Processing Systems 2022, NeurIPS 2022, New Orleans,
                  LA, USA, November 28 - December 9, 2022},
  year         = {2022},
  url          = {http://papers.nips.cc/paper\_files/paper/2022/hash/82ad13ec01f9fe44c01cb91814fd7b8c-Abstract-Conference.html},
  bibsource    = {dblp computer science bibliography, https://dblp.org}
}

@inproceedings{DBLP:conf/iclr/ShridharYCBTH21,
  author       = {Mohit Shridhar and
                  Xingdi Yuan and
                  Marc{-}Alexandre C{\^{o}}t{\'{e}} and
                  Yonatan Bisk and
                  Adam Trischler and
                  Matthew J. Hausknecht},
  title        = {ALFWorld: Aligning Text and Embodied Environments for Interactive
                  Learning},
  booktitle    = {9th International Conference on Learning Representations, {ICLR} 2021,
                  Virtual Event, Austria, May 3-7, 2021},
  publisher    = {OpenReview.net},
  year         = {2021},
  url          = {https://openreview.net/forum?id=0IOX0YcCdTn},
  bibsource    = {dblp computer science bibliography, https://dblp.org}
}

@article{DBLP:journals/tmlr/WeiTBRZBYBZMCHVLDF22,
  author       = {Jason Wei and
                  Yi Tay and
                  Rishi Bommasani and
                  Colin Raffel and
                  Barret Zoph and
                  Sebastian Borgeaud and
                  Dani Yogatama and
                  Maarten Bosma and
                  Denny Zhou and
                  Donald Metzler and
                  Ed H. Chi and
                  Tatsunori Hashimoto and
                  Oriol Vinyals and
                  Percy Liang and
                  Jeff Dean and
                  William Fedus},
  title        = {Emergent Abilities of Large Language Models},
  journal      = {Trans. Mach. Learn. Res.},
  volume       = {2022},
  year         = {2022},
  url          = {https://openreview.net/forum?id=yzkSU5zdwD},
  bibsource    = {dblp computer science bibliography, https://dblp.org}
}

@inproceedings{
yao2024retroformer,
title={Retroformer: Retrospective Large Language Agents with Policy Gradient Optimization},
author={Weiran Yao and Shelby Heinecke and Juan Carlos Niebles and Zhiwei Liu and Yihao Feng and Le Xue and Rithesh R N and Zeyuan Chen and Jianguo Zhang and Devansh Arpit and Ran Xu and Phil L Mui and Huan Wang and Caiming Xiong and Silvio Savarese},
booktitle={The Twelfth International Conference on Learning Representations},
year={2024},
url={https://openreview.net/forum?id=KOZu91CzbK}
}

@article{madaan2023self,
  title={Self-refine: Iterative refinement with self-feedback},
  author={Madaan, Aman and Tandon, Niket and Gupta, Prakhar and Hallinan, Skyler and Gao, Luyu and Wiegreffe, Sarah and Alon, Uri and Dziri, Nouha and Prabhumoye, Shrimai and Yang, Yiming and others},
  journal={Advances in Neural Information Processing Systems},
  volume={36},
  pages={46534--46594},
  year={2023}
}

@misc{
kool2019buy,
title={Buy 4 {REINFORCE} Samples, Get a Baseline for Free!},
author={Wouter Kool and Herke van Hoof and Max Welling},
year={2019},
url={https://openreview.net/forum?id=r1lgTGL5DE}
}

@inproceedings{
yao2022webshop,
title={WebShop: Towards Scalable Real-World Web Interaction with Grounded Language Agents},
author={Shunyu Yao and Howard Chen and John Yang and Karthik R Narasimhan},
booktitle={Advances in Neural Information Processing Systems},
editor={Alice H. Oh and Alekh Agarwal and Danielle Belgrave and Kyunghyun Cho},
year={2022},
url={https://openreview.net/forum?id=R9KnuFlvnU}
}

@inproceedings{
zhou2024archer,
title={Ar{CH}er: Training Language Model Agents via Hierarchical Multi-Turn {RL}},
author={Yifei Zhou and Andrea Zanette and Jiayi Pan and Sergey Levine and Aviral Kumar},
booktitle={Forty-first International Conference on Machine Learning},
year={2024},
url={https://openreview.net/forum?id=b6rA0kAHT1}
}

@misc{zhang2025agentrlscalingagenticreinforcement,
      title={AgentRL: Scaling Agentic Reinforcement Learning with a Multi-Turn, Multi-Task Framework}, 
      author={Hanchen Zhang and Xiao Liu and Bowen Lv and Xueqiao Sun and Bohao Jing and Iat Long Iong and Zhenyu Hou and Zehan Qi and Hanyu Lai and Yifan Xu and Rui Lu and Hongning Wang and Jie Tang and Yuxiao Dong},
      year={2025},
      eprint={2510.04206},
      archivePrefix={arXiv},
      primaryClass={cs.AI},
      url={https://arxiv.org/abs/2510.04206}, 
}

@misc{wang2025cogitoergoludoagent,
      title={Cogito, Ergo Ludo: An Agent that Learns to Play by Reasoning and Planning}, 
      author={Sai Wang and Yu Wu and Zhongwen Xu},
      year={2025},
      eprint={2509.25052},
      archivePrefix={arXiv},
      primaryClass={cs.AI},
      url={https://arxiv.org/abs/2509.25052}, 
}

@article{comanici2025gemini,
  title={Gemini 2.5: Pushing the frontier with advanced reasoning, multimodality, long context, and next generation agentic capabilities},
  author={Comanici, Gheorghe and Bieber, Eric and Schaekermann, Mike and Pasupat, Ice and Sachdeva, Noveen and Dhillon, Inderjit and Blistein, Marcel and Ram, Ori and Zhang, Dan and Rosen, Evan and others},
  journal={arXiv preprint arXiv:2507.06261},
  year={2025}
}

@article{anthrophic2025skills,
  title={Introducing Agent Skills},
  author={Anthropic},
  journal={Claude Blog},
  year={2025}
}

@article{auer2002finite,
  title={Finite-time analysis of the multiarmed bandit problem},
  author={Auer, Peter and Cesa-Bianchi, Nicolo and Fischer, Paul},
  journal={Machine learning},
  volume={47},
  number={2},
  pages={235--256},
  year={2002},
  publisher={Springer}
}

@inproceedings{
liu2026exploratory,
title={Exploratory Memory-Augmented {LLM} Agent via Hybrid On- and Off-Policy Optimization},
author={Zeyuan Liu and Jeonghye Kim and Xufang Luo and Dongsheng Li and Yuqing Yang},
booktitle={The Fourteenth International Conference on Learning Representations},
year={2026},
url={https://openreview.net/forum?id=UOzxviKVFO}
}

@inproceedings{zhang-etal-2022-toward-self,
    title = "Toward Self-Learning End-to-End Task-oriented Dialog Systems",
    author = "Zhang, Xiaoying  and
      Peng, Baolin  and
      Gao, Jianfeng  and
      Meng, Helen",
    editor = "Lemon, Oliver  and
      Hakkani-Tur, Dilek  and
      Li, Junyi Jessy  and
      Ashrafzadeh, Arash  and
      Garcia, Daniel Hern{\'a}ndez  and
      Alikhani, Malihe  and
      Vandyke, David  and
      Du{\v{s}}ek, Ond{\v{r}}ej",
    booktitle = "Proceedings of the 23rd Annual Meeting of the Special Interest Group on Discourse and Dialogue",
    month = sep,
    year = "2022",
    address = "Edinburgh, UK",
    publisher = "Association for Computational Linguistics",
    url = "https://aclanthology.org/2022.sigdial-1.49/",
    doi = "10.18653/v1/2022.sigdial-1.49",
    pages = "516--530"
}

@article{liu2025gem,
  title={Gem: A gym for agentic llms},
  author={Liu, Zichen and Sims, Anya and Duan, Keyu and Chen, Changyu and Yu, Simon and Zhou, Xiangxin and Xu, Haotian and Xiong, Shaopan and Liu, Bo and Tan, Chenmien and others},
  journal={arXiv preprint arXiv:2510.01051},
  year={2025}
}

@article{abel2023definition,
  title={A definition of continual reinforcement learning},
  author={Abel, David and Barreto, Andr{\'e} and Van Roy, Benjamin and Precup, Doina and van Hasselt, Hado P and Singh, Satinder},
  journal={Advances in Neural Information Processing Systems},
  volume={36},
  pages={50377--50407},
  year={2023}
}

@article{lin1992self,
  title={Self-improving reactive agents based on reinforcement learning, planning and teaching},
  author={Lin, Long-Ji},
  journal={Machine learning},
  volume={8},
  number={3},
  pages={293--321},
  year={1992},
  publisher={Springer}
}

@article{graves2014neural,
  title={Neural turing machines},
  author={Graves, Alex and Wayne, Greg and Danihelka, Ivo},
  journal={arXiv preprint arXiv:1410.5401},
  year={2014}
}

@inproceedings{DBLP:conf/nips/Ouyang0JAWMZASR22,
  author       = {Long Ouyang and
                  Jeffrey Wu and
                  Xu Jiang and
                  Diogo Almeida and
                  Carroll L. Wainwright and
                  Pamela Mishkin and
                  Chong Zhang and
                  Sandhini Agarwal and
                  Katarina Slama and
                  Alex Ray and
                  John Schulman and
                  Jacob Hilton and
                  Fraser Kelton and
                  Luke Miller and
                  Maddie Simens and
                  Amanda Askell and
                  Peter Welinder and
                  Paul F. Christiano and
                  Jan Leike and
                  Ryan Lowe},
  editor       = {Sanmi Koyejo and
                  S. Mohamed and
                  A. Agarwal and
                  Danielle Belgrave and
                  K. Cho and
                  A. Oh},
  title        = {Training language models to follow instructions with human feedback},
  booktitle    = {Advances in Neural Information Processing Systems 35: Annual Conference
                  on Neural Information Processing Systems 2022, NeurIPS 2022, New Orleans,
                  LA, USA, November 28 - December 9, 2022},
  year         = {2022},
  url          = {http://papers.nips.cc/paper\_files/paper/2022/hash/b1efde53be364a73914f58805a001731-Abstract-Conference.html},
  bibsource    = {dblp computer science bibliography, https://dblp.org}
}

@article{rafiee2026toward,
  title={Toward Enactive Artificial Intelligence},
  author={Rafiee, Banafsheh and Sutton, Richard},
  journal={arXiv preprint arXiv:2605.24238},
  year={2026}
}

@incollection{lehman2023evolution,
  title={Evolution through large models},
  author={Lehman, Joel and Gordon, Jonathan and Jain, Shawn and Ndousse, Kamal and Yeh, Cathy and Stanley, Kenneth O},
  booktitle={Handbook of evolutionary machine learning},
  pages={331--366},
  year={2023},
  publisher={Springer}
}

@article{hinton1987learning,
  title={How learning can guide evolution},
  author={Hinton, Geoffrey E and Nowlan, Steven J and others},
  journal={Complex systems},
  volume={1},
  number={3},
  pages={495--502},
  year={1987}
}

@misc{weng2026learning_beyond_gradients,
  title        = {Learning Beyond Gradients},
  author       = {Weng, Jiayi},
  year         = {2026},
  month        = may,
  howpublished = {\url{https://trinkle23897.github.io/learning-beyond-gradients/}},
  note         = {Blog post}
}

@inproceedings{
kirk2024understanding,
title={Understanding the Effects of {RLHF} on {LLM} Generalisation and Diversity},
author={Robert Kirk and Ishita Mediratta and Christoforos Nalmpantis and Jelena Luketina and Eric Hambro and Edward Grefenstette and Roberta Raileanu},
booktitle={The Twelfth International Conference on Learning Representations},
year={2024},
url={https://openreview.net/forum?id=PXD3FAVHJT}
}

@inproceedings{goyal2022retrieval,
  title={Retrieval-augmented reinforcement learning},
  author={Goyal, Anirudh and Friesen, Abram and Banino, Andrea and Weber, Theophane and Ke, Nan Rosemary and Badia, Adria Puigdomenech and Guez, Arthur and Mirza, Mehdi and Humphreys, Peter C and Konyushova, Ksenia and others},
  booktitle={International Conference on Machine Learning},
  pages={7740--7765},
  year={2022},
  organization={PMLR}
}

@article{singh2025openai,
  title={Openai gpt-5 system card},
  author={Singh, Aaditya and Fry, Adam and Perelman, Adam and Tart, Adam and Ganesh, Adi and El-Kishky, Ahmed and McLaughlin, Aidan and Low, Aiden and Ostrow, AJ and Ananthram, Akhila and others},
  journal={arXiv preprint arXiv:2601.03267},
  year={2025}
}

@inproceedings{song-etal-2024-trial,
    title = "Trial and Error: Exploration-Based Trajectory Optimization of {LLM} Agents",
    author = "Song, Yifan  and
      Yin, Da  and
      Yue, Xiang  and
      Huang, Jie  and
      Li, Sujian  and
      Lin, Bill Yuchen",
    editor = "Ku, Lun-Wei  and
      Martins, Andre  and
      Srikumar, Vivek",
    booktitle = "Proceedings of the 62nd Annual Meeting of the Association for Computational Linguistics (Volume 1: Long Papers)",
    month = aug,
    year = "2024",
    address = "Bangkok, Thailand",
    publisher = "Association for Computational Linguistics",
    url = "https://aclanthology.org/2024.acl-long.409/",
    doi = "10.18653/v1/2024.acl-long.409",
    pages = "7584--7600"
}

@inproceedings{xi-etal-2025-agentgym,
    title = "{A}gent{G}ym: Evaluating and Training Large Language Model-based Agents across Diverse Environments",
    author = "Xi, Zhiheng  and
      Ding, Yiwen  and
      Chen, Wenxiang  and
      Hong, Boyang  and
      Guo, Honglin  and
      Wang, Junzhe  and
      Guo, Xin  and
      Yang, Dingwen  and
      Liao, Chenyang  and
      He, Wei  and
      Gao, Songyang  and
      Chen, Lu  and
      Zheng, Rui  and
      Zou, Yicheng  and
      Gui, Tao  and
      Zhang, Qi  and
      Qiu, Xipeng  and
      Huang, Xuanjing  and
      Wu, Zuxuan  and
      Jiang, Yu-Gang",
    editor = "Che, Wanxiang  and
      Nabende, Joyce  and
      Shutova, Ekaterina  and
      Pilehvar, Mohammad Taher",
    booktitle = "Proceedings of the 63rd Annual Meeting of the Association for Computational Linguistics (Volume 1: Long Papers)",
    month = jul,
    year = "2025",
    address = "Vienna, Austria",
    publisher = "Association for Computational Linguistics",
    url = "https://aclanthology.org/2025.acl-long.1355/",
    doi = "10.18653/v1/2025.acl-long.1355",
    pages = "27914--27961",
    ISBN = "979-8-89176-251-0"
}

@article{xie2024osworld,
  title={Osworld: Benchmarking multimodal agents for open-ended tasks in real computer environments},
  author={Xie, Tianbao and Zhang, Danyang and Chen, Jixuan and Li, Xiaochuan and Zhao, Siheng and Cao, Ruisheng and Hua, Toh J and Cheng, Zhoujun and Shin, Dongchan and Lei, Fangyu and others},
  journal={Advances in Neural Information Processing Systems},
  volume={37},
  pages={52040--52094},
  year={2024}
}

@inproceedings{zhang2025ufo,
  title={Ufo: A ui-focused agent for windows os interaction},
  author={Zhang, Chaoyun and Li, Liqun and He, Shilin and Zhang, Xu and Qiao, Bo and Qin, Si and Ma, Minghua and Kang, Yu and Lin, Qingwei and Rajmohan, Saravan and others},
  booktitle={Proceedings of the 2025 Conference of the Nations of the Americas Chapter of the Association for Computational Linguistics: Human Language Technologies (Volume 1: Long Papers)},
  pages={597--622},
  year={2025}
}

@misc{chen2025reinforcementlearninglonghorizoninteractive,
      title={Reinforcement Learning for Long-Horizon Interactive LLM Agents}, 
      author={Kevin Chen and Marco Cusumano-Towner and Brody Huval and Aleksei Petrenko and Jackson Hamburger and Vladlen Koltun and Philipp Krähenbühl},
      year={2025},
      eprint={2502.01600},
      archivePrefix={arXiv},
      primaryClass={cs.LG},
      url={https://arxiv.org/abs/2502.01600}, 
}

@misc{openai2024gpt4ocard,
      title={GPT-4o System Card}, 
      author={OpenAI and : and Aaron Hurst and Adam Lerer and Adam P. Goucher and Adam Perelman and Aditya Ramesh and Aidan Clark and AJ Ostrow and Akila Welihinda and Alan Hayes and Alec Radford and Aleksander Mądry and Alex Baker-Whitcomb and Alex Beutel and Alex Borzunov and Alex Carney and Alex Chow and Alex Kirillov and Alex Nichol and Alex Paino and Alex Renzin and Alex Tachard Passos and Alexander Kirillov and Alexi Christakis and Alexis Conneau and Ali Kamali and Allan Jabri and Allison Moyer and Allison Tam and Amadou Crookes and Amin Tootoochian and Amin Tootoonchian and Ananya Kumar and Andrea Vallone and Andrej Karpathy and Andrew Braunstein and Andrew Cann and Andrew Codispoti and Andrew Galu and Andrew Kondrich and Andrew Tulloch and Andrey Mishchenko and Angela Baek and Angela Jiang and Antoine Pelisse and Antonia Woodford and Anuj Gosalia and Arka Dhar and Ashley Pantuliano and Avi Nayak and Avital Oliver and Barret Zoph and Behrooz Ghorbani and Ben Leimberger and Ben Rossen and Ben Sokolowsky and Ben Wang and Benjamin Zweig and Beth Hoover and Blake Samic and Bob McGrew and Bobby Spero and Bogo Giertler and Bowen Cheng and Brad Lightcap and Brandon Walkin and Brendan Quinn and Brian Guarraci and Brian Hsu and Bright Kellogg and Brydon Eastman and Camillo Lugaresi and Carroll Wainwright and Cary Bassin and Cary Hudson and Casey Chu and Chad Nelson and Chak Li and Chan Jun Shern and Channing Conger and Charlotte Barette and Chelsea Voss and Chen Ding and Cheng Lu and Chong Zhang and Chris Beaumont and Chris Hallacy and Chris Koch and Christian Gibson and Christina Kim and Christine Choi and Christine McLeavey and Christopher Hesse and Claudia Fischer and Clemens Winter and Coley Czarnecki and Colin Jarvis and Colin Wei and Constantin Koumouzelis and Dane Sherburn and Daniel Kappler and Daniel Levin and Daniel Levy and David Carr and David Farhi and David Mely and David Robinson and David Sasaki and Denny Jin and Dev Valladares and Dimitris Tsipras and Doug Li and Duc Phong Nguyen and Duncan Findlay and Edede Oiwoh and Edmund Wong and Ehsan Asdar and Elizabeth Proehl and Elizabeth Yang and Eric Antonow and Eric Kramer and Eric Peterson and Eric Sigler and Eric Wallace and Eugene Brevdo and Evan Mays and Farzad Khorasani and Felipe Petroski Such and Filippo Raso and Francis Zhang and Fred von Lohmann and Freddie Sulit and Gabriel Goh and Gene Oden and Geoff Salmon and Giulio Starace and Greg Brockman and Hadi Salman and Haiming Bao and Haitang Hu and Hannah Wong and Haoyu Wang and Heather Schmidt and Heather Whitney and Heewoo Jun and Hendrik Kirchner and Henrique Ponde de Oliveira Pinto and Hongyu Ren and Huiwen Chang and Hyung Won Chung and Ian Kivlichan and Ian O'Connell and Ian O'Connell and Ian Osband and Ian Silber and Ian Sohl and Ibrahim Okuyucu and Ikai Lan and Ilya Kostrikov and Ilya Sutskever and Ingmar Kanitscheider and Ishaan Gulrajani and Jacob Coxon and Jacob Menick and Jakub Pachocki and James Aung and James Betker and James Crooks and James Lennon and Jamie Kiros and Jan Leike and Jane Park and Jason Kwon and Jason Phang and Jason Teplitz and Jason Wei and Jason Wolfe and Jay Chen and Jeff Harris and Jenia Varavva and Jessica Gan Lee and Jessica Shieh and Ji Lin and Jiahui Yu and Jiayi Weng and Jie Tang and Jieqi Yu and Joanne Jang and Joaquin Quinonero Candela and Joe Beutler and Joe Landers and Joel Parish and Johannes Heidecke and John Schulman and Jonathan Lachman and Jonathan McKay and Jonathan Uesato and Jonathan Ward and Jong Wook Kim and Joost Huizinga and Jordan Sitkin and Jos Kraaijeveld and Josh Gross and Josh Kaplan and Josh Snyder and Joshua Achiam and Joy Jiao and Joyce Lee and Juntang Zhuang and Justyn Harriman and Kai Fricke and Kai Hayashi and Karan Singhal and Katy Shi and Kavin Karthik and Kayla Wood and Kendra Rimbach and Kenny Hsu and Kenny Nguyen and Keren Gu-Lemberg and Kevin Button and Kevin Liu and Kiel Howe and Krithika Muthukumar and Kyle Luther and Lama Ahmad and Larry Kai and Lauren Itow and Lauren Workman and Leher Pathak and Leo Chen and Li Jing and Lia Guy and Liam Fedus and Liang Zhou and Lien Mamitsuka and Lilian Weng and Lindsay McCallum and Lindsey Held and Long Ouyang and Louis Feuvrier and Lu Zhang and Lukas Kondraciuk and Lukasz Kaiser and Luke Hewitt and Luke Metz and Lyric Doshi and Mada Aflak and Maddie Simens and Madelaine Boyd and Madeleine Thompson and Marat Dukhan and Mark Chen and Mark Gray and Mark Hudnall and Marvin Zhang and Marwan Aljubeh and Mateusz Litwin and Matthew Zeng and Max Johnson and Maya Shetty and Mayank Gupta and Meghan Shah and Mehmet Yatbaz and Meng Jia Yang and Mengchao Zhong and Mia Glaese and Mianna Chen and Michael Janner and Michael Lampe and Michael Petrov and Michael Wu and Michele Wang and Michelle Fradin and Michelle Pokrass and Miguel Castro and Miguel Oom Temudo de Castro and Mikhail Pavlov and Miles Brundage and Miles Wang and Minal Khan and Mira Murati and Mo Bavarian and Molly Lin and Murat Yesildal and Nacho Soto and Natalia Gimelshein and Natalie Cone and Natalie Staudacher and Natalie Summers and Natan LaFontaine and Neil Chowdhury and Nick Ryder and Nick Stathas and Nick Turley and Nik Tezak and Niko Felix and Nithanth Kudige and Nitish Keskar and Noah Deutsch and Noel Bundick and Nora Puckett and Ofir Nachum and Ola Okelola and Oleg Boiko and Oleg Murk and Oliver Jaffe and Olivia Watkins and Olivier Godement and Owen Campbell-Moore and Patrick Chao and Paul McMillan and Pavel Belov and Peng Su and Peter Bak and Peter Bakkum and Peter Deng and Peter Dolan and Peter Hoeschele and Peter Welinder and Phil Tillet and Philip Pronin and Philippe Tillet and Prafulla Dhariwal and Qiming Yuan and Rachel Dias and Rachel Lim and Rahul Arora and Rajan Troll and Randall Lin and Rapha Gontijo Lopes and Raul Puri and Reah Miyara and Reimar Leike and Renaud Gaubert and Reza Zamani and Ricky Wang and Rob Donnelly and Rob Honsby and Rocky Smith and Rohan Sahai and Rohit Ramchandani and Romain Huet and Rory Carmichael and Rowan Zellers and Roy Chen and Ruby Chen and Ruslan Nigmatullin and Ryan Cheu and Saachi Jain and Sam Altman and Sam Schoenholz and Sam Toizer and Samuel Miserendino and Sandhini Agarwal and Sara Culver and Scott Ethersmith and Scott Gray and Sean Grove and Sean Metzger and Shamez Hermani and Shantanu Jain and Shengjia Zhao and Sherwin Wu and Shino Jomoto and Shirong Wu and Shuaiqi and Xia and Sonia Phene and Spencer Papay and Srinivas Narayanan and Steve Coffey and Steve Lee and Stewart Hall and Suchir Balaji and Tal Broda and Tal Stramer and Tao Xu and Tarun Gogineni and Taya Christianson and Ted Sanders and Tejal Patwardhan and Thomas Cunninghman and Thomas Degry and Thomas Dimson and Thomas Raoux and Thomas Shadwell and Tianhao Zheng and Todd Underwood and Todor Markov and Toki Sherbakov and Tom Rubin and Tom Stasi and Tomer Kaftan and Tristan Heywood and Troy Peterson and Tyce Walters and Tyna Eloundou and Valerie Qi and Veit Moeller and Vinnie Monaco and Vishal Kuo and Vlad Fomenko and Wayne Chang and Weiyi Zheng and Wenda Zhou and Wesam Manassra and Will Sheu and Wojciech Zaremba and Yash Patil and Yilei Qian and Yongjik Kim and Youlong Cheng and Yu Zhang and Yuchen He and Yuchen Zhang and Yujia Jin and Yunxing Dai and Yury Malkov},
      year={2024},
      eprint={2410.21276},
      archivePrefix={arXiv},
      primaryClass={cs.CL},
      url={https://arxiv.org/abs/2410.21276}, 
}

@article{sheng2024hybridflow,
  title   = {HybridFlow: A Flexible and Efficient RLHF Framework},
  author  = {Guangming Sheng and Chi Zhang and Zilingfeng Ye and Xibin Wu and Wang Zhang and Ru Zhang and Yanghua Peng and Haibin Lin and Chuan Wu},
  year    = {2024},
  journal = {arXiv preprint arXiv: 2409.19256}
}

@misc{grattafiori2024llama3herdmodels,
      title={The Llama 3 Herd of Models}, 
      author={Aaron Grattafiori and Abhimanyu Dubey and Abhinav Jauhri and Abhinav Pandey and Abhishek Kadian and Ahmad Al-Dahle and Aiesha Letman and Akhil Mathur and Alan Schelten and Alex Vaughan and Amy Yang and Angela Fan and Anirudh Goyal and Anthony Hartshorn and Aobo Yang and Archi Mitra and Archie Sravankumar and Artem Korenev and Arthur Hinsvark and Arun Rao and Aston Zhang and Aurelien Rodriguez and Austen Gregerson and Ava Spataru and Baptiste Roziere and Bethany Biron and Binh Tang and Bobbie Chern and Charlotte Caucheteux and Chaya Nayak and Chloe Bi and Chris Marra and Chris McConnell and Christian Keller and Christophe Touret and Chunyang Wu and Corinne Wong and Cristian Canton Ferrer and Cyrus Nikolaidis and Damien Allonsius and Daniel Song and Danielle Pintz and Danny Livshits and Danny Wyatt and David Esiobu and Dhruv Choudhary and Dhruv Mahajan and Diego Garcia-Olano and Diego Perino and Dieuwke Hupkes and Egor Lakomkin and Ehab AlBadawy and Elina Lobanova and Emily Dinan and Eric Michael Smith and Filip Radenovic and Francisco Guzmán and Frank Zhang and Gabriel Synnaeve and Gabrielle Lee and Georgia Lewis Anderson and Govind Thattai and Graeme Nail and Gregoire Mialon and Guan Pang and Guillem Cucurell and Hailey Nguyen and Hannah Korevaar and Hu Xu and Hugo Touvron and Iliyan Zarov and Imanol Arrieta Ibarra and Isabel Kloumann and Ishan Misra and Ivan Evtimov and Jack Zhang and Jade Copet and Jaewon Lee and Jan Geffert and Jana Vranes and Jason Park and Jay Mahadeokar and Jeet Shah and Jelmer van der Linde and Jennifer Billock and Jenny Hong and Jenya Lee and Jeremy Fu and Jianfeng Chi and Jianyu Huang and Jiawen Liu and Jie Wang and Jiecao Yu and Joanna Bitton and Joe Spisak and Jongsoo Park and Joseph Rocca and Joshua Johnstun and Joshua Saxe and Junteng Jia and Kalyan Vasuden Alwala and Karthik Prasad and Kartikeya Upasani and Kate Plawiak and Ke Li and Kenneth Heafield and Kevin Stone and Khalid El-Arini and Krithika Iyer and Kshitiz Malik and Kuenley Chiu and Kunal Bhalla and Kushal Lakhotia and Lauren Rantala-Yeary and Laurens van der Maaten and Lawrence Chen and Liang Tan and Liz Jenkins and Louis Martin and Lovish Madaan and Lubo Malo and Lukas Blecher and Lukas Landzaat and Luke de Oliveira and Madeline Muzzi and Mahesh Pasupuleti and Mannat Singh and Manohar Paluri and Marcin Kardas and Maria Tsimpoukelli and Mathew Oldham and Mathieu Rita and Maya Pavlova and Melanie Kambadur and Mike Lewis and Min Si and Mitesh Kumar Singh and Mona Hassan and Naman Goyal and Narjes Torabi and Nikolay Bashlykov and Nikolay Bogoychev and Niladri Chatterji and Ning Zhang and Olivier Duchenne and Onur Çelebi and Patrick Alrassy and Pengchuan Zhang and Pengwei Li and Petar Vasic and Peter Weng and Prajjwal Bhargava and Pratik Dubal and Praveen Krishnan and Punit Singh Koura and Puxin Xu and Qing He and Qingxiao Dong and Ragavan Srinivasan and Raj Ganapathy and Ramon Calderer and Ricardo Silveira Cabral and Robert Stojnic and Roberta Raileanu and Rohan Maheswari and Rohit Girdhar and Rohit Patel and Romain Sauvestre and Ronnie Polidoro and Roshan Sumbaly and Ross Taylor and Ruan Silva and Rui Hou and Rui Wang and Saghar Hosseini and Sahana Chennabasappa and Sanjay Singh and Sean Bell and Seohyun Sonia Kim and Sergey Edunov and Shaoliang Nie and Sharan Narang and Sharath Raparthy and Sheng Shen and Shengye Wan and Shruti Bhosale and Shun Zhang and Simon Vandenhende and Soumya Batra and Spencer Whitman and Sten Sootla and Stephane Collot and Suchin Gururangan and Sydney Borodinsky and Tamar Herman and Tara Fowler and Tarek Sheasha and Thomas Georgiou and Thomas Scialom and Tobias Speckbacher and Todor Mihaylov and Tong Xiao and Ujjwal Karn and Vedanuj Goswami and Vibhor Gupta and Vignesh Ramanathan and Viktor Kerkez and Vincent Gonguet and Virginie Do and Vish Vogeti and Vítor Albiero and Vladan Petrovic and Weiwei Chu and Wenhan Xiong and Wenyin Fu and Whitney Meers and Xavier Martinet and Xiaodong Wang and Xiaofang Wang and Xiaoqing Ellen Tan and Xide Xia and Xinfeng Xie and Xuchao Jia and Xuewei Wang and Yaelle Goldschlag and Yashesh Gaur and Yasmine Babaei and Yi Wen and Yiwen Song and Yuchen Zhang and Yue Li and Yuning Mao and Zacharie Delpierre Coudert and Zheng Yan and Zhengxing Chen and Zoe Papakipos and Aaditya Singh and Aayushi Srivastava and Abha Jain and Adam Kelsey and Adam Shajnfeld and Adithya Gangidi and Adolfo Victoria and Ahuva Goldstand and Ajay Menon and Ajay Sharma and Alex Boesenberg and Alexei Baevski and Allie Feinstein and Amanda Kallet and Amit Sangani and Amos Teo and Anam Yunus and Andrei Lupu and Andres Alvarado and Andrew Caples and Andrew Gu and Andrew Ho and Andrew Poulton and Andrew Ryan and Ankit Ramchandani and Annie Dong and Annie Franco and Anuj Goyal and Aparajita Saraf and Arkabandhu Chowdhury and Ashley Gabriel and Ashwin Bharambe and Assaf Eisenman and Azadeh Yazdan and Beau James and Ben Maurer and Benjamin Leonhardi and Bernie Huang and Beth Loyd and Beto De Paola and Bhargavi Paranjape and Bing Liu and Bo Wu and Boyu Ni and Braden Hancock and Bram Wasti and Brandon Spence and Brani Stojkovic and Brian Gamido and Britt Montalvo and Carl Parker and Carly Burton and Catalina Mejia and Ce Liu and Changhan Wang and Changkyu Kim and Chao Zhou and Chester Hu and Ching-Hsiang Chu and Chris Cai and Chris Tindal and Christoph Feichtenhofer and Cynthia Gao and Damon Civin and Dana Beaty and Daniel Kreymer and Daniel Li and David Adkins and David Xu and Davide Testuggine and Delia David and Devi Parikh and Diana Liskovich and Didem Foss and Dingkang Wang and Duc Le and Dustin Holland and Edward Dowling and Eissa Jamil and Elaine Montgomery and Eleonora Presani and Emily Hahn and Emily Wood and Eric-Tuan Le and Erik Brinkman and Esteban Arcaute and Evan Dunbar and Evan Smothers and Fei Sun and Felix Kreuk and Feng Tian and Filippos Kokkinos and Firat Ozgenel and Francesco Caggioni and Frank Kanayet and Frank Seide and Gabriela Medina Florez and Gabriella Schwarz and Gada Badeer and Georgia Swee and Gil Halpern and Grant Herman and Grigory Sizov and Guangyi and Zhang and Guna Lakshminarayanan and Hakan Inan and Hamid Shojanazeri and Han Zou and Hannah Wang and Hanwen Zha and Haroun Habeeb and Harrison Rudolph and Helen Suk and Henry Aspegren and Hunter Goldman and Hongyuan Zhan and Ibrahim Damlaj and Igor Molybog and Igor Tufanov and Ilias Leontiadis and Irina-Elena Veliche and Itai Gat and Jake Weissman and James Geboski and James Kohli and Janice Lam and Japhet Asher and Jean-Baptiste Gaya and Jeff Marcus and Jeff Tang and Jennifer Chan and Jenny Zhen and Jeremy Reizenstein and Jeremy Teboul and Jessica Zhong and Jian Jin and Jingyi Yang and Joe Cummings and Jon Carvill and Jon Shepard and Jonathan McPhie and Jonathan Torres and Josh Ginsburg and Junjie Wang and Kai Wu and Kam Hou U and Karan Saxena and Kartikay Khandelwal and Katayoun Zand and Kathy Matosich and Kaushik Veeraraghavan and Kelly Michelena and Keqian Li and Kiran Jagadeesh and Kun Huang and Kunal Chawla and Kyle Huang and Lailin Chen and Lakshya Garg and Lavender A and Leandro Silva and Lee Bell and Lei Zhang and Liangpeng Guo and Licheng Yu and Liron Moshkovich and Luca Wehrstedt and Madian Khabsa and Manav Avalani and Manish Bhatt and Martynas Mankus and Matan Hasson and Matthew Lennie and Matthias Reso and Maxim Groshev and Maxim Naumov and Maya Lathi and Meghan Keneally and Miao Liu and Michael L. Seltzer and Michal Valko and Michelle Restrepo and Mihir Patel and Mik Vyatskov and Mikayel Samvelyan and Mike Clark and Mike Macey and Mike Wang and Miquel Jubert Hermoso and Mo Metanat and Mohammad Rastegari and Munish Bansal and Nandhini Santhanam and Natascha Parks and Natasha White and Navyata Bawa and Nayan Singhal and Nick Egebo and Nicolas Usunier and Nikhil Mehta and Nikolay Pavlovich Laptev and Ning Dong and Norman Cheng and Oleg Chernoguz and Olivia Hart and Omkar Salpekar and Ozlem Kalinli and Parkin Kent and Parth Parekh and Paul Saab and Pavan Balaji and Pedro Rittner and Philip Bontrager and Pierre Roux and Piotr Dollar and Polina Zvyagina and Prashant Ratanchandani and Pritish Yuvraj and Qian Liang and Rachad Alao and Rachel Rodriguez and Rafi Ayub and Raghotham Murthy and Raghu Nayani and Rahul Mitra and Rangaprabhu Parthasarathy and Raymond Li and Rebekkah Hogan and Robin Battey and Rocky Wang and Russ Howes and Ruty Rinott and Sachin Mehta and Sachin Siby and Sai Jayesh Bondu and Samyak Datta and Sara Chugh and Sara Hunt and Sargun Dhillon and Sasha Sidorov and Satadru Pan and Saurabh Mahajan and Saurabh Verma and Seiji Yamamoto and Sharadh Ramaswamy and Shaun Lindsay and Shaun Lindsay and Sheng Feng and Shenghao Lin and Shengxin Cindy Zha and Shishir Patil and Shiva Shankar and Shuqiang Zhang and Shuqiang Zhang and Sinong Wang and Sneha Agarwal and Soji Sajuyigbe and Soumith Chintala and Stephanie Max and Stephen Chen and Steve Kehoe and Steve Satterfield and Sudarshan Govindaprasad and Sumit Gupta and Summer Deng and Sungmin Cho and Sunny Virk and Suraj Subramanian and Sy Choudhury and Sydney Goldman and Tal Remez and Tamar Glaser and Tamara Best and Thilo Koehler and Thomas Robinson and Tianhe Li and Tianjun Zhang and Tim Matthews and Timothy Chou and Tzook Shaked and Varun Vontimitta and Victoria Ajayi and Victoria Montanez and Vijai Mohan and Vinay Satish Kumar and Vishal Mangla and Vlad Ionescu and Vlad Poenaru and Vlad Tiberiu Mihailescu and Vladimir Ivanov and Wei Li and Wenchen Wang and Wenwen Jiang and Wes Bouaziz and Will Constable and Xiaocheng Tang and Xiaojian Wu and Xiaolan Wang and Xilun Wu and Xinbo Gao and Yaniv Kleinman and Yanjun Chen and Ye Hu and Ye Jia and Ye Qi and Yenda Li and Yilin Zhang and Ying Zhang and Yossi Adi and Youngjin Nam and Yu and Wang and Yu Zhao and Yuchen Hao and Yundi Qian and Yunlu Li and Yuzi He and Zach Rait and Zachary DeVito and Zef Rosnbrick and Zhaoduo Wen and Zhenyu Yang and Zhiwei Zhao and Zhiyu Ma},
      year={2024},
      eprint={2407.21783},
      archivePrefix={arXiv},
      primaryClass={cs.AI},
      url={https://arxiv.org/abs/2407.21783}, 
}

@misc{qwen2025qwen25technicalreport,
      title={Qwen2.5 Technical Report}, 
      author={Qwen and : and An Yang and Baosong Yang and Beichen Zhang and Binyuan Hui and Bo Zheng and Bowen Yu and Chengyuan Li and Dayiheng Liu and Fei Huang and Haoran Wei and Huan Lin and Jian Yang and Jianhong Tu and Jianwei Zhang and Jianxin Yang and Jiaxi Yang and Jingren Zhou and Junyang Lin and Kai Dang and Keming Lu and Keqin Bao and Kexin Yang and Le Yu and Mei Li and Mingfeng Xue and Pei Zhang and Qin Zhu and Rui Men and Runji Lin and Tianhao Li and Tianyi Tang and Tingyu Xia and Xingzhang Ren and Xuancheng Ren and Yang Fan and Yang Su and Yichang Zhang and Yu Wan and Yuqiong Liu and Zeyu Cui and Zhenru Zhang and Zihan Qiu},
      year={2025},
      eprint={2412.15115},
      archivePrefix={arXiv},
      primaryClass={cs.CL},
      url={https://arxiv.org/abs/2412.15115}, 
}

@inproceedings{ahmadian2024back,
  title={Back to basics: Revisiting REINFORCE-style optimization for learning from human feedback in LLMs},
  author={Ahmadian, Arash and Cremer, Chris and Gall{\'e}, Matthias and Fadaee, Marzieh and Kreutzer, Julia and Pietquin, Olivier and {\"U}st{\"u}n, Ahmet and Hooker, Sara},
  booktitle={Proceedings of the 62nd Annual Meeting of the Association for Computational Linguistics (Volume 1: Long Papers)},
  pages={12248--12267},
  year={2024}
}

@article{beck2025tutorial,
  title={A tutorial on meta-reinforcement learning},
  author={Beck, Jacob and Vuorio, Risto and Zheran Liu, Evan and Xiong, Zheng and Zintgraf, Luisa and Finn, Chelsea and Whiteson, Shimon},
  journal={Foundations and Trends in Machine Learning},
  volume={18},
  number={2-3},
  pages={224--384},
  year={2025},
  publisher={Emerald Publishing Limited}
}

@article{wang2025ragen,
  title={Ragen: Understanding self-evolution in llm agents via multi-turn reinforcement learning},
  author={Wang, Zihan and Wang, Kangrui and Wang, Qineng and Zhang, Pingyue and Li, Linjie and Yang, Zhengyuan and Jin, Xing and Yu, Kefan and Nguyen, Minh Nhat and Liu, Licheng and others},
  journal={arXiv preprint arXiv:2504.20073},
  year={2025}
}

@misc{wei2025reinforcingmultiturnreasoningllm,
      title={Reinforcing Multi-Turn Reasoning in LLM Agents via Turn-Level Reward Design}, 
      author={Quan Wei and Siliang Zeng and Chenliang Li and William Brown and Oana Frunza and Wei Deng and Anderson Schneider and Yuriy Nevmyvaka and Yang Katie Zhao and Alfredo Garcia and Mingyi Hong},
      year={2025},
      eprint={2505.11821},
      archivePrefix={arXiv},
      primaryClass={cs.LG},
      url={https://arxiv.org/abs/2505.11821}, 
}

@misc{gao2025navigateunknownenhancingllm,
      title={Navigate the Unknown: Enhancing LLM Reasoning with Intrinsic Motivation Guided Exploration}, 
      author={Jingtong Gao and Ling Pan and Yejing Wang and Rui Zhong and Chi Lu and Qingpeng Cai and Peng Jiang and Xiangyu Zhao},
      year={2025},
      eprint={2505.17621},
      archivePrefix={arXiv},
      primaryClass={cs.LG},
      url={https://arxiv.org/abs/2505.17621}, 
}

@inproceedings{
liu2025position,
title={Position: Truly Self-Improving Agents Require Intrinsic Metacognitive Learning},
author={Tennison Liu and Mihaela van der Schaar},
booktitle={Forty-second International Conference on Machine Learning Position Paper Track},
year={2025},
url={https://openreview.net/forum?id=4KhDd0Ozqe}
}

@misc{wang2025harnessinguncertaintyentropymodulatedpolicy,
      title={Harnessing Uncertainty: Entropy-Modulated Policy Gradients for Long-Horizon LLM Agents}, 
      author={Jiawei Wang and Jiacai Liu and Yuqian Fu and Yingru Li and Xintao Wang and Yuan Lin and Yu Yue and Lin Zhang and Yang Wang and Ke Wang},
      year={2025},
      eprint={2509.09265},
      archivePrefix={arXiv},
      primaryClass={cs.LG},
      url={https://arxiv.org/abs/2509.09265}, 
}

@misc{wang2025sparlreinforcingllmagents,
      title={SPA-RL: Reinforcing LLM Agents via Stepwise Progress Attribution}, 
      author={Hanlin Wang and Chak Tou Leong and Jiashuo Wang and Jian Wang and Wenjie Li},
      year={2025},
      eprint={2505.20732},
      archivePrefix={arXiv},
      primaryClass={cs.CL},
      url={https://arxiv.org/abs/2505.20732}, 
}
\bibliographystyle{tmlr}

\clearpage
\appendix

\tableofcontents
\addtocontents{toc}{\protect\setcounter{tocdepth}{2}}
\clearpage


\section{Theoretical Analysis -- Joint benefit of dual intrinsic feedback}
\label{app:theoretical_analysis}

We analyze how reflection memory and intrinsic numerical feedback jointly affect
policy optimization. Memory retrieval changes the rollout distribution by
increasing the probability of sampling trajectories with high reflected
potential, while the capability-evolution reward assigns larger return to such
trajectories. Thus, trajectories exhibiting partial progress can be reinforced
even without extrinsic success.

Let \(q_0=x\) be the original prompt and
\[
    q_m=f_{\mathrm{memory}}(x,\mathcal{M})
\]
be the memory-augmented prompt. For a fixed behavior policy
\(\pi_{\theta_{\mathrm{old}}}\), define
\[
    Q_0(\tau)=p_{\theta_{\mathrm{old}}}(\tau\mid q_0),
    \qquad
    Q_m(\tau)=p_{\theta_{\mathrm{old}}}(\tau\mid q_m).
\]
If memory-augmented rollouts are sampled with probability \(\rho\in[0,1]\), the
rollout distribution is
\begin{equation}
    Q_\rho(\tau)=(1-\rho)Q_0(\tau)+\rho Q_m(\tau).
    \label{eq:joint_mixture_distribution}
\end{equation}
In \retroagent{}, \(\rho=1/2\).

Recall that \(\tau\) is \(\delta\)-valuable at iteration \(k\) if
\[
    \phi_{(x,\tau),k}\ge \Phi_x^{(k)}+\delta .
\]
Let
\[
    p_{k,\delta}^0
    =
    \Pr_{\tau\sim Q_0}
    \left[\tau\in\mathcal{V}_{k,\delta}(x)\right],
    \qquad
    p_{k,\delta}^m
    =
    \Pr_{\tau\sim Q_m}
    \left[\tau\in\mathcal{V}_{k,\delta}(x)\right].
\]
Then, under \(Q_\rho\),
\begin{equation}
    p_{k,\delta}^{\rho}
    =
    \Pr_{\tau\sim Q_\rho}
    \left[\tau\in\mathcal{V}_{k,\delta}(x)\right]
    =
    (1-\rho)p_{k,\delta}^0+\rho p_{k,\delta}^m .
    \label{eq:joint_mixture_probability}
\end{equation}

\paragraph{Intrinsic feedback as selective gradient weighting.}

At iteration \(k\), the augmented reward is
\begin{equation}
    R_k^{\mathrm{aug}}(\tau)
    =
    R^{\mathrm{ext}}(\tau)
    +
    R_k^{\mathrm{int}}(\tau),
    \qquad
    R_k^{\mathrm{int}}(\tau)
    =
    \left[\phi_{(x,\tau),k}-\Phi_x^{(k)}\right]_+ .
    \label{eq:joint_augmented_reward}
\end{equation}
For \(q\in\{q_0,q_m\}\), define
\begin{equation}
    J_k^{\mathrm{aug}}(\theta;q)
    =
    \mathbb{E}_{\tau\sim p_\theta(\cdot\mid q)}
    \left[
        R_k^{\mathrm{aug}}(\tau)
    \right].
    \label{eq:joint_prompt_objective}
\end{equation}
Treating rewards as trajectory-level feedback, the likelihood-ratio identity
gives
\begin{equation}
\begin{aligned}
    \nabla_\theta J_k^{\mathrm{aug}}(\theta;q)
    &=
    \mathbb{E}_{\tau\sim p_\theta(\cdot\mid q)}
    \left[
        R_k^{\mathrm{aug}}(\tau)
        \nabla_\theta\log p_\theta(\tau\mid q)
    \right] \\
    &=
    \mathbb{E}
    \left[
        R^{\mathrm{ext}}(\tau)
        \nabla_\theta\log p_\theta(\tau\mid q)
    \right]
    +
    \mathbb{E}
    \left[
        R_k^{\mathrm{int}}(\tau)
        \nabla_\theta\log p_\theta(\tau\mid q)
    \right].
\end{aligned}
\label{eq:joint_pg_decomposition}
\end{equation}
Thus, intrinsic feedback adds the gradient component
\begin{equation}
    g_k^{\mathrm{int}}(\theta;q)
    =
    \mathbb{E}_{\tau\sim p_\theta(\cdot\mid q)}
    \left[
        R_k^{\mathrm{int}}(\tau)
        \nabla_\theta\log p_\theta(\tau\mid q)
    \right].
    \label{eq:joint_intrinsic_gradient}
\end{equation}
For every \(\tau\in\mathcal{V}_{k,\delta}(x)\),
\(R_k^{\mathrm{int}}(\tau)\ge\delta\), whereas trajectories below the current
baseline receive no intrinsic reward. Hence, the capability-evolution reward
selectively upweights trajectories whose reflected potential exceeds the
current capability threshold.

In practice, returns are replaced by advantages. Let
\begin{equation}
    A_k^{\mathrm{aug}}(\tau,q)
    =
    R_k^{\mathrm{aug}}(\tau)-b_k(q),
    \label{eq:joint_augmented_advantage}
\end{equation}
where \(b_k(q)\) is an action-independent baseline, such as the GRPO group mean.
Since
\[
    \mathbb{E}_{\tau\sim p_\theta(\cdot\mid q)}
    \left[
        b_k(q)\nabla_\theta\log p_\theta(\tau\mid q)
    \right]=0,
\]
the policy gradient is equivalently
\begin{equation}
    \nabla_\theta J_k^{\mathrm{aug}}(\theta;q)
    =
    \mathbb{E}_{\tau\sim p_\theta(\cdot\mid q)}
    \left[
        A_k^{\mathrm{aug}}(\tau,q)
        \nabla_\theta\log p_\theta(\tau\mid q)
    \right].
    \label{eq:joint_advantage_pg}
\end{equation}
Therefore, even if \(R^{\mathrm{ext}}(\tau)=0\), a trajectory can receive
positive gradient weight whenever its augmented advantage is positive. Intrinsic
feedback consequently separates failures with meaningful reflected progress
from uninformative failures.

\begin{proposition}[Intrinsic feedback lower-bounds valuable-trajectory weight]
\label{prop:joint_intrinsic_weight}
Fix \(q\in\{q_0,q_m\}\), and define
\[
    p_{k,\delta}(\theta;q)
    =
    \Pr_{\tau\sim p_\theta(\cdot\mid q)}
    \left[
        \tau\in\mathcal{V}_{k,\delta}(x)
    \right].
\]
For
\[
    J_k^{\mathrm{int}}(\theta;q)
    =
    \mathbb{E}_{\tau\sim p_\theta(\cdot\mid q)}
    \left[
        R_k^{\mathrm{int}}(\tau)
    \right],
\]
we have
\begin{equation}
    J_k^{\mathrm{int}}(\theta;q)
    \ge
    \delta\,p_{k,\delta}(\theta;q).
    \label{eq:joint_intrinsic_lower_bound}
\end{equation}
Moreover,
\begin{equation}
\begin{aligned}
    \nabla_\theta J_k^{\mathrm{int}}(\theta;q)
    &=
    \mathbb{E}
    \left[
        R_k^{\mathrm{int}}(\tau)
        \nabla_\theta\log p_\theta(\tau\mid q)
        \mathbf{1}\{\tau\in\mathcal{V}_{k,\delta}(x)\}
    \right] \\
    &\quad+
    \mathbb{E}
    \left[
        R_k^{\mathrm{int}}(\tau)
        \nabla_\theta\log p_\theta(\tau\mid q)
        \mathbf{1}\{\tau\notin\mathcal{V}_{k,\delta}(x)\}
    \right],
\end{aligned}
\label{eq:joint_intrinsic_gradient_decomposition}
\end{equation}
where the first term assigns at least \(\delta\) intrinsic return weight to each
\(\delta\)-valuable trajectory.
\end{proposition}

\begin{proof}
For \(\tau\in\mathcal{V}_{k,\delta}(x)\),
\(R_k^{\mathrm{int}}(\tau)\ge\delta\), and otherwise
\(R_k^{\mathrm{int}}(\tau)\ge0\). Hence,
\[
\begin{aligned}
    J_k^{\mathrm{int}}(\theta;q)
    &=
    \mathbb{E}
    \left[
        R_k^{\mathrm{int}}(\tau)
    \right] \\
    &\ge
    \mathbb{E}
    \left[
        \delta\,\mathbf{1}\{\tau\in\mathcal{V}_{k,\delta}(x)\}
    \right] \\
    &=
    \delta\,p_{k,\delta}(\theta;q).
\end{aligned}
\]
The gradient decomposition follows by inserting
\[
    1
    =
    \mathbf{1}\{\tau\in\mathcal{V}_{k,\delta}(x)\}
    +
    \mathbf{1}\{\tau\notin\mathcal{V}_{k,\delta}(x)\}
\]
into the likelihood-ratio gradient.
\end{proof}

\paragraph{Joint effect.}

The mixture objective is
\begin{equation}
    J_k^\rho(\theta)
    =
    (1-\rho)J_k^{\mathrm{aug}}(\theta;q_0)
    +
    \rho J_k^{\mathrm{aug}}(\theta;q_m).
    \label{eq:joint_mixture_objective}
\end{equation}
Its gradient is
\begin{equation}
\begin{aligned}
    \nabla_\theta J_k^\rho(\theta)
    &=
    (1-\rho)
    \mathbb{E}_{\tau\sim Q_0}
    \left[
        A_k^{\mathrm{aug}}(\tau,q_0)
        \nabla_\theta\log p_\theta(\tau\mid q_0)
    \right] \\
    &\quad+
    \rho
    \mathbb{E}_{\tau\sim Q_m}
    \left[
        A_k^{\mathrm{aug}}(\tau,q_m)
        \nabla_\theta\log p_\theta(\tau\mid q_m)
    \right].
\end{aligned}
\label{eq:joint_mixture_pg}
\end{equation}
Thus, memory retrieval affects the sampling distribution, while intrinsic
feedback affects return and advantage weights.

\begin{theorem}[Joint improvement from memory retrieval and intrinsic reward]
\label{thm:joint_memory_intrinsic}
Assume memory augmentation increases the probability of sampling a
\(\delta\)-valuable trajectory by at least \(\Delta_{\mathcal{M}}>0\):
\begin{equation}
    p_{k,\delta}^m
    \ge
    p_{k,\delta}^0+\Delta_{\mathcal{M}} .
    \label{eq:joint_memory_gain_assumption}
\end{equation}
Then
\begin{equation}
    p_{k,\delta}^\rho
    \ge
    p_{k,\delta}^0+\rho\Delta_{\mathcal{M}},
    \label{eq:joint_probability_gain}
\end{equation}
and
\begin{equation}
    \mathbb{E}_{\tau\sim Q_\rho}
    \left[
        R_k^{\mathrm{int}}(\tau)
    \right]
    \ge
    \delta
    \left(
        p_{k,\delta}^0+\rho\Delta_{\mathcal{M}}
    \right).
    \label{eq:joint_intrinsic_gain_bound}
\end{equation}
Therefore, relative to base-only sampling, the intrinsic-reward lower bound
improves by at least \(\rho\delta\Delta_{\mathcal{M}}\). For
\(\rho=1/2\), the gain is at least \(\delta\Delta_{\mathcal{M}}/2\).
\end{theorem}

\begin{proof}
By Equation~\eqref{eq:joint_mixture_probability} and
Equation~\eqref{eq:joint_memory_gain_assumption},
\[
\begin{aligned}
    p_{k,\delta}^{\rho}
    &=
    (1-\rho)p_{k,\delta}^0+\rho p_{k,\delta}^m \\
    &\ge
    (1-\rho)p_{k,\delta}^0
    +
    \rho
    \left(
        p_{k,\delta}^0+\Delta_{\mathcal{M}}
    \right) \\
    &=
    p_{k,\delta}^0+\rho\Delta_{\mathcal{M}}.
\end{aligned}
\]
Since every \(\delta\)-valuable trajectory receives intrinsic reward at least
\(\delta\), and all intrinsic rewards are nonnegative,
\[
    \mathbb{E}_{\tau\sim Q_\rho}
    \left[
        R_k^{\mathrm{int}}(\tau)
    \right]
    \ge
    \delta p_{k,\delta}^{\rho}
    \ge
    \delta
    \left(
        p_{k,\delta}^0+\rho\Delta_{\mathcal{M}}
    \right).
\]
\end{proof}

\paragraph{Internalization through shared parameters.}

Because base and memory-augmented prompts share policy parameters, a
high-advantage trajectory sampled under one prompt can increase the likelihood
of related behavior under the other prompt when their score gradients are
aligned.

Let \(\tau^+\in\mathcal{V}_{k,\delta}(x)\) be sampled under
\(q_a\in\{q_0,q_m\}\), and suppose
\(A_k^{\mathrm{aug}}(\tau^+,q_a)>0\). A single-sample update is
\begin{equation}
    \theta^+
    =
    \theta
    +
    \eta
    A_k^{\mathrm{aug}}(\tau^+,q_a)
    \nabla_\theta\log p_\theta(\tau^+\mid q_a),
    \label{eq:joint_single_sample_update}
\end{equation}
with learning rate \(\eta>0\). If \(\tau^+\) is extrinsically failed but
valuable, then \(R^{\mathrm{ext}}(\tau^+)=0\) and
\(R_k^{\mathrm{int}}(\tau^+)\ge\delta\), so it may still be reinforced whenever
its augmented advantage is positive.

\begin{proposition}[Internalization of memory-guided exploration]
\label{prop:joint_internalization}
Let \(\tau^+\in\mathcal{V}_{k,\delta}(x)\) be sampled under
\(q_a\in\{q_0,q_m\}\), and let \(q_b\in\{q_0,q_m\}\). Under
Equation~\eqref{eq:joint_single_sample_update},
\begin{equation}
\begin{aligned}
    \log p_{\theta^+}(\tau^+\mid q_b)
    -
    \log p_\theta(\tau^+\mid q_b)
    &=
    \eta
    A_k^{\mathrm{aug}}(\tau^+,q_a)
    \left\langle
        \nabla_\theta\log p_\theta(\tau^+\mid q_b),
        \nabla_\theta\log p_\theta(\tau^+\mid q_a)
    \right\rangle \\
    &\quad+
    O(\eta^2).
\end{aligned}
\label{eq:joint_internalization_change}
\end{equation}
Thus, if the inner product is positive, then for sufficiently small \(\eta\) the
update increases the likelihood of \(\tau^+\) under \(q_b\) to first order.
\end{proposition}

\begin{proof}
A first-order Taylor expansion gives
\[
\begin{aligned}
    \log p_{\theta^+}(\tau^+\mid q_b)
    &=
    \log p_\theta(\tau^+\mid q_b)
    +
    \left\langle
        \nabla_\theta\log p_\theta(\tau^+\mid q_b),
        \theta^+-\theta
    \right\rangle
    +
    O(\|\theta^+-\theta\|_2^2).
\end{aligned}
\]
Substituting Equation~\eqref{eq:joint_single_sample_update} proves
Equation~\eqref{eq:joint_internalization_change}. Since
\(A_k^{\mathrm{aug}}(\tau^+,q_a)>0\), positive alignment gives a positive
first-order change.
\end{proof}

\begin{corollary}[Expected transfer from memory-augmented to base prompting]
\label{cor:joint_expected_internalization}
Define
\[
    h_0(\tau)=\nabla_\theta\log p_\theta(\tau\mid q_0),
    \qquad
    h_m(\tau)=\nabla_\theta\log p_\theta(\tau\mid q_m).
\]
Suppose
\begin{equation}
    \mathbb{E}_{\tau\sim Q_m}
    \left[
        A_k^{\mathrm{aug}}(\tau,q_m)
        \left\langle h_0(\tau),h_m(\tau)\right\rangle
        \mathbf{1}\{\tau\in\mathcal{V}_{k,\delta}(x)\}
    \right]
    >
    0.
    \label{eq:joint_expected_alignment}
\end{equation}
Then, for sufficiently small \(\eta\), a policy-gradient update on
memory-augmented valuable trajectories increases their expected log-likelihood
under the base prompt to first order:
\begin{equation}
\begin{aligned}
    \mathbb{E}_{\tau\sim Q_m}
    \left[
        \left(
            \log p_{\theta^+}(\tau\mid q_0)
            -
            \log p_\theta(\tau\mid q_0)
        \right)
        \mathbf{1}\{\tau\in\mathcal{V}_{k,\delta}(x)\}
    \right]
    >
    0,
\end{aligned}
\label{eq:joint_expected_base_transfer}
\end{equation}
up to \(O(\eta^2)\) terms.
\end{corollary}

\begin{proof}
Apply Proposition~\ref{prop:joint_internalization} with \(q_a=q_m\) and
\(q_b=q_0\), multiply by
\(\mathbf{1}\{\tau\in\mathcal{V}_{k,\delta}(x)\}\), and take expectation over
\(\tau\sim Q_m\). The first-order term is positive by
Equation~\eqref{eq:joint_expected_alignment}; for sufficiently small \(\eta\),
it dominates the \(O(\eta^2)\) remainder.
\end{proof}

Overall, memory retrieval shifts sampling toward trajectories with higher
reflected potential, intrinsic numerical feedback upweights trajectories that
exceed the current capability baseline, and shared parameters allow
high-advantage memory-guided behaviors to transfer to the base prompt when score
gradients are aligned. Together, these mechanisms explain how \retroagent{}
improves exploration, credit assignment, and online adaptation under sparse
extrinsic rewards.

\clearpage
\section{Implementation Details}
\label{app:implementation_details}

\begin{table*}[th]
\caption{Default hyperparameters and training configurations for \textsc{RetroAgent} across all environments.}
\label{tab:hyperparameters}
\centering
\small
\begin{threeparttable}
\resizebox{\textwidth}{!}{
\begin{tabular}{lccp{4cm}}
\toprule
\textbf{Hyperparameter} & \textbf{\qwentwoseven{}} & \textbf{Llama-3.1-8B-Instruct} & \textbf{Description} \\
\midrule
\multicolumn{4}{l}{\textit{Training Configuration}} \\
\midrule
Training batch size & 16 & 16 & Accumulated batch size per update \\
Validation batch size & 128 & 128 & Batch size for validation \\
Learning rate & $10^{-6}$ & $10^{-6}$ & Optimizer learning rate \\
Max prompt length & 16\,384 & 16\,384 & Maximum input context length (tokens) \\
Max response length & 2\,048 & 2\,048 & Maximum generated response length (tokens) \\
Group size ($N$) & 8 & 8 & Number of rollouts per prompt \\
Total steps & 150\,/\,300 & 150\,/\,300 & Training epochs (150 for ALFWorld and WebShop; 300 for Sokoban and Minesweeper) \\
Evaluation frequency & 5 & 5 & Epochs between consecutive evaluations \\
\midrule
\multicolumn{4}{l}{\textit{Reward and Regularization}} \\
\midrule
Extrinsic reward ($R^{\text{ext}}$) & $\{0,\,10\}$ & $\{0,\,10\}$ & Scalar reward from the environment \\
Intrinsic reward ($R^{\text{int}}$) & $[0,\,1]$ & $[0,\,1]$ & Capability-evolution intrinsic reward \\
KL coefficient ($\beta$) & 0.01 & 0.01 & KL-divergence regularization weight \\
Discount factor ($\gamma$) & 0.95 & 0.95 & Discount factor for multi-step returns \\
\midrule
\multicolumn{4}{l}{\textit{Memory and Sampling}} \\
\midrule
Training temperature & 0.4 & 0.4 & Sampling temperature during rollouts \\
Validation temperature & 0.4 & 0.4 & Sampling temperature during validation \\
Initial utility score & 0.5 & 0.5 & Initial utility assigned to each memory entry \\
Utility smoothing ($\beta_{\text{util}}$) & 0.05 & 0.05 & Exponential moving average coefficient for utility updates \\
UCB exploration constant ($c$) & 1.0 & 1.0 & Exploration coefficient in UCB-based retrieval \\
Relevance--utility weight ($\alpha$) & 0.7 & 0.7 & Trade-off coefficient in retrieval scoring \\
Memory-augmented ratio & $1{:}1$ & $1{:}1$ & Ratio of memory-augmented to base rollouts \\
\midrule
\multicolumn{4}{l}{\textit{Self-Reflection (RL-Trained Variant)}} \\
\midrule
Reflection reward ($R^{\text{reflect}}$) & $\{0,\,10\}$ & $\{0,\,10\}$ & Scalar reward for reflection accuracy \\
Reflection weight ($\lambda_{\text{reflect}}$) & 1.0 & 1.0 & Weight of the self-reflection objective relative to the decision-making objective \\
\midrule
\multicolumn{4}{l}{\textit{Evaluation Configuration}} \\
\midrule
Evaluation temperature & 0.4 & 0.4 & Sampling temperature during evaluation \\
Max inference tokens & 2\,048 & 2\,048 & Maximum token budget per inference step \\
\bottomrule
\end{tabular}
}
\end{threeparttable}
\end{table*}


Detailed hyperparameter settings and training configurations are provided in Table~\ref{tab:hyperparameters}. We use \texttt{sentence-transformers/all-MiniLM-L6-v2}.\footnote{\url{https://huggingface.co/sentence-transformers/all-MiniLM-L6-v2}} All experiments were conducted on 4 NVIDIA H200 GPUs. 

\section{Superiority of Pairwise Induction over Single Induction}
\label{app:eval_pairwise_induction_single_induction}

We analyze reflection sequences generated during GRPO training augmented with either single-trajectory or pairwise-trajectory induction. Specifically, we sample four sequences per task prompt every 10 training steps, randomly selecting 800 sequences from this pool. Using GPT-4o as an external evaluator, we assess two metrics: the alignment between the agent's induced subtask completion scores and oracle outcomes (induced by GPT-4o), and the quality of the agent's summarized lessons (see Appendix~\ref{app:prompt_templates_analyzing_self_reflection} for evaluation prompts).

\begin{figure*}[th]
\centering
\begin{subfigure}[th]{0.48\linewidth}
    \centering
    \includegraphics[width=\linewidth]{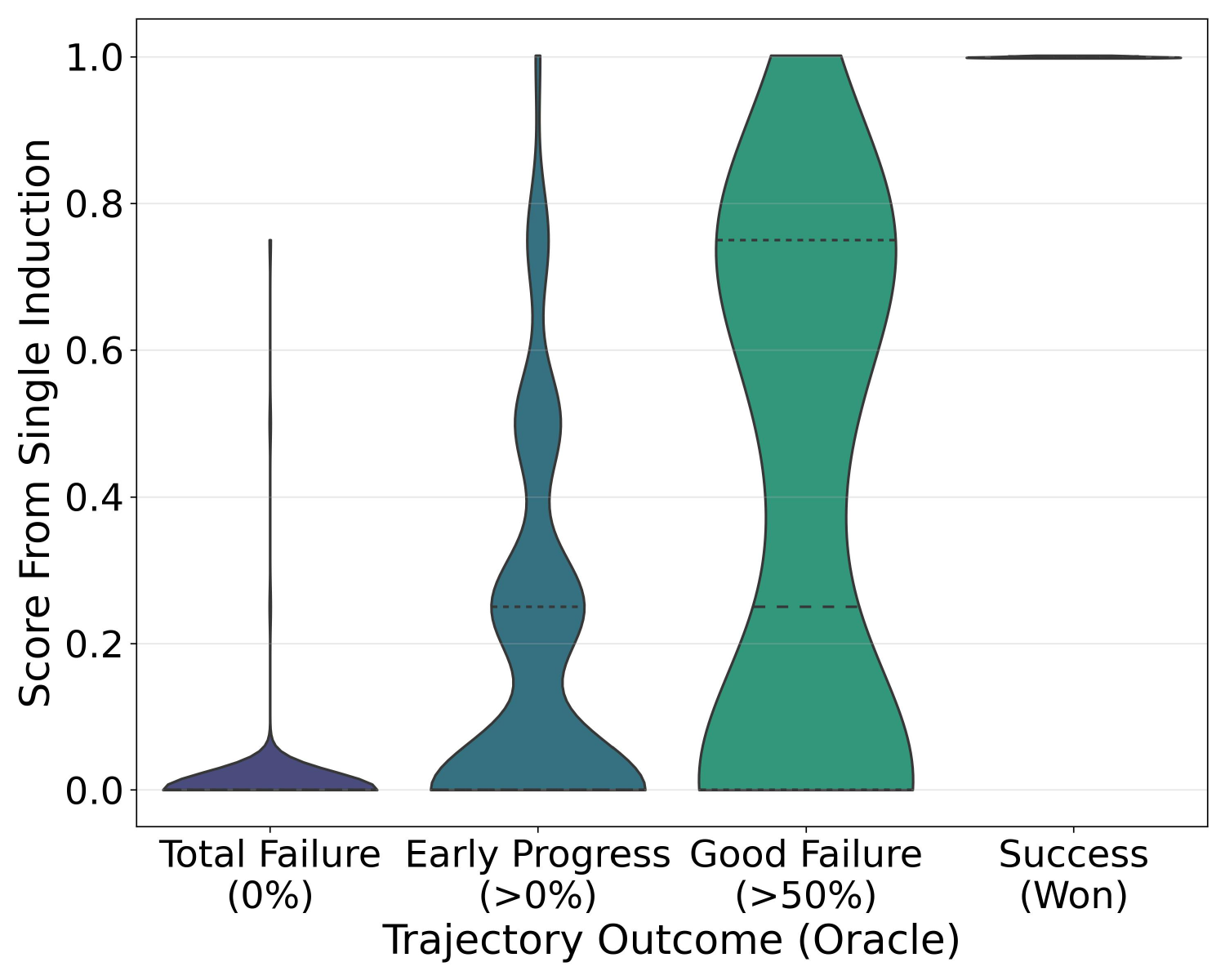}
    \caption{Single induction.}
    \label{fig:discri_single_induction}
\end{subfigure}
\hfill
\begin{subfigure}[th]{0.48\linewidth}
    \centering
    \includegraphics[width=\linewidth]{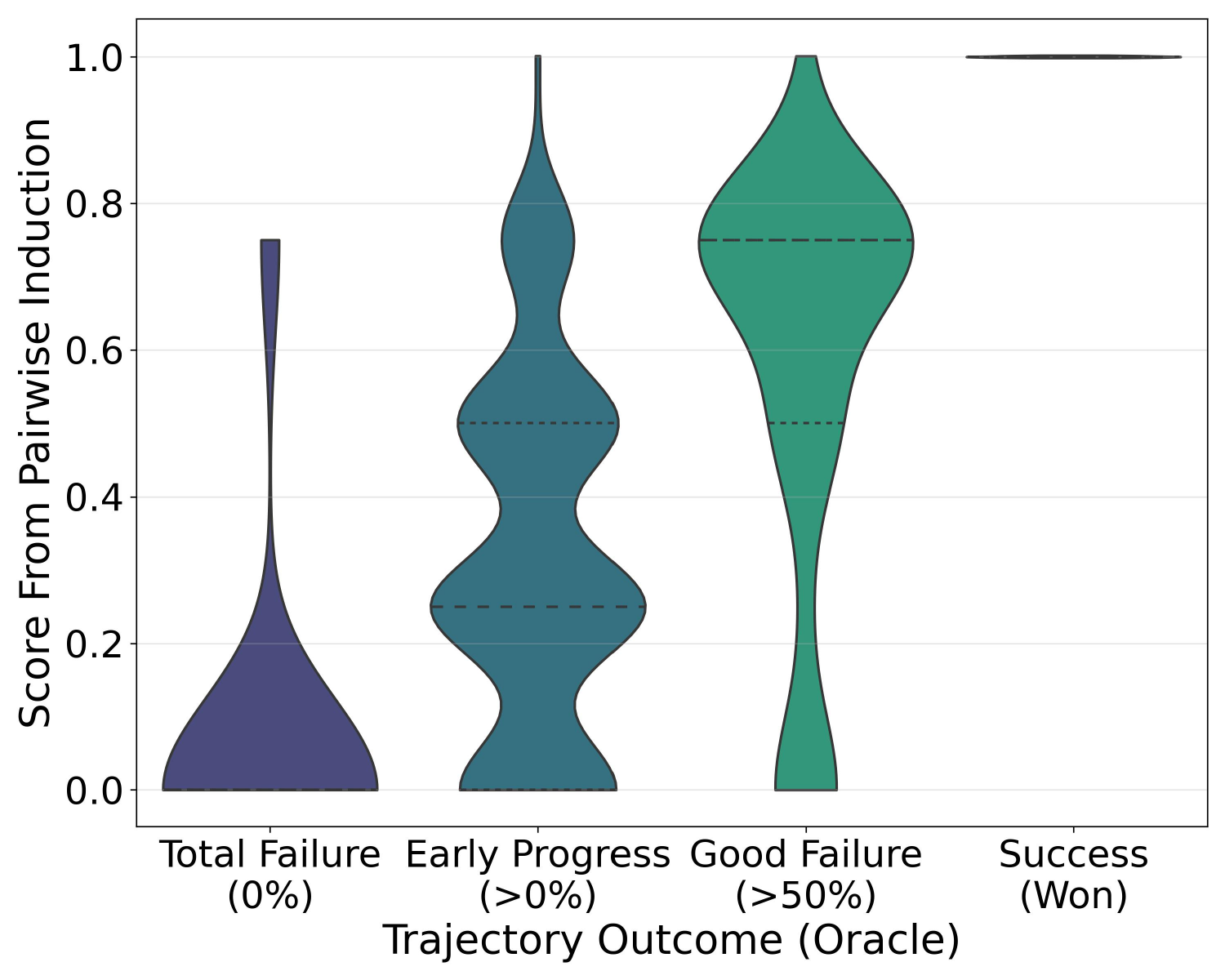}
    \caption{Pairwise induction.}
    \label{fig:discri_pairwise_induction}
\end{subfigure}
\caption{Comparison of intrinsic discrimination capability for Qwen-2.5-7B-Instruct using single \textit{vs.} pairwise induction, measured by subtask completion scores on WebShop.}
\label{fig:comparison_intrinsic_discrimination}
\end{figure*}

Figure~\ref{fig:comparison_intrinsic_discrimination} compares the ability of single- and pairwise-trajectory induction to discriminate trajectory quality within the training-free self-reflection mechanism. For this analysis, we define a successful trajectory as having a completion score of 1. While both methods correctly assign minimal scores to ``Total Failure'' cases, pairwise induction demonstrates superior discrimination for intermediate outcomes. For example, as shown in Figure~\ref{fig:discri_single_induction}, single induction struggles to distinguish ``Early Progress'' from ``Total Failure,'' resulting in a distribution heavily skewed toward zero for partial successes. In contrast, pairwise induction (Figure~\ref{fig:discri_pairwise_induction}) effectively recognizes partial progress, assigning distinctly higher scores to ``Early Progress'' trajectories and establishing a clear separation from total failures.

\clearpage

\section{Generated Trajectories by \retroagent{}}
\label{app:generated_trajectories}
\begin{table*}[ht]
\caption{Successful trajectory generated by \retroagent{} (in-context
    self-reflection) at training step~150 on the WebShop validation set.}
\label{tab:traj_retroagent}
\centering
\small 
\renewcommand{\arraystretch}{1.3} 
\begin{tabular}{p{0.95\textwidth}}
\toprule
\textbf{Instruction:} You are an expert autonomous agent operating in the WebShop e‑commerce environment. Your task is to: \textcolor{blue}{Find me officially licensed, machine wash men's t-shirts with polyester heathers, heathers cotton, cotton heather, needle sleeve, classic fit with color: pink, and fit type: youth, and size: small, and price lower than 50.00 dollars.} \\
\textbf{Past reflections on similar tasks:} \textcolor{purple}{Action Insight: The search terms used did not match the available products, leading to no suitable options being found. | Navigation Insight: Efficient use of filters could have potentially narrowed down the search, but the specific dimensions and color were not met in any of the results.} Warning: These lessons may be outdated. Use them only if they align with your current observation. \\
\midrule

\textbf{Observation 1:} Your current observation is: `Search'. Your admissible actions of the current situation are: \texttt{['search[<your query>]', 'click[search]']} Now it's your turn to take one action for the current step. You should first reason step-by-step about the current situation, then think carefully which admissible action best advances the shopping goal. This reasoning process MUST be enclosed within \textless{}think\textgreater{} \textless{}/think\textgreater{} tags. Once you've finished your reasoning, you should choose an admissible action for current step and present it within \textless{}action\textgreater{} \textless{}/action\textgreater{} tags. \\
\textbf{\retroagent{}:} \textcolor{orange}{\textless{}think\textgreater{} To achieve the goal of finding officially licensed, machine wash men's t-shirts with specific attributes, I need to construct a search query that includes all the required attributes... I will construct the search query to include all these attributes and then execute the search. \textless{}/think\textgreater{}} \textbf{\textit{Action:}} \textcolor{islamicgreen}{\texttt{search[officially licensed men's t-shirts polyester heathers heathers cotton cotton heather needle sleeve classic fit color pink fit type youth fit type size small price < 50.00]}} \\
\midrule

\textbf{Observation 2:} You are now at step 2 and your current observation is: 'Back to Search' [SEP] 'Page 1 (Total results: 50)' [SEP] 'Next $>$' [SEP] 'B09GL561XH' [SEP] 'Disney Mickey And Friends Happy Birthday Mickey Confetti T-Shirt' [SEP] '\$22.99' [SEP] 'B07XPR3R7N' [SEP] 'Marvel Avengers: Endgame Captain America America's Language T-Shirt' [SEP] '\$22.99' ... [SEP] 'B09Q8RD8YN' [SEP] 'Bungo Stray Anime Dogs Anime Character, Long Sleeve, Sweatshirt, Hoodie, T shirt' [SEP] '\$19.99'. \\
\textbf{\retroagent{}:} \textcolor{orange}{\textless{}think\textgreater{} The current page does not list any t-shirts that are officially licensed... The best action at this point is to click on the first product listed to see if it is the correct product... \textless{}/think\textgreater{}} \textit{\textbf{Action:}} \textcolor{islamicgreen}{\texttt{click[b09gl561xh]}} \\
\midrule

\textbf{Observation 3:} You are now at step 3... 'fit type' [SEP] 'men' [SEP] 'women' [SEP] 'youth' [SEP] 'color' [SEP] 'white' ... [SEP] 'pink' [SEP] 'size' [SEP] 'small' ... 'Disney Mickey And Friends Happy Birthday Mickey Confetti T-Shirt' [SEP] 'Price: \$22.99'. \\
\textbf{\retroagent{}:} \textcolor{orange}{\textless{}think\textgreater{} The current page is not displaying the results for the search query that was previously conducted... The fit type 'youth' is one of the specific fit types mentioned in the search query, so I will click on 'youth' to ensure that the filter is set correctly. \textless{}/think\textgreater{}} \textit{\textbf{Action:}} \textcolor{islamicgreen}{\texttt{click[youth]}} \\
\midrule

\textbf{Observation 4:} You are now at step 4... 'fit type' [SEP] 'men' [SEP] 'women' [SEP] 'youth' [SEP] 'color' ... [SEP] 'pink' [SEP] 'size' [SEP] 'small' ... 'Disney Mickey And Friends Happy Birthday Mickey Confetti T-Shirt' [SEP] 'Price: \$22.99'. \\
\textbf{\retroagent{}:} \textcolor{orange}{\textless{}think\textgreater{} The current situation shows that the search has narrowed down to a specific product... The next step is to select the color pink to ensure that the color is officially licensed and meets the criteria. \textless{}/think\textgreater{}} \textit{\textbf{Action:}} \textcolor{islamicgreen}{\texttt{click[pink]}} \\
\midrule

\textbf{Observation 5:} You are now at step 5... 'Disney Mickey And Friends Happy Birthday Mickey Confetti T-Shirt' [SEP] 'Price: \$22.99' ... \\
\textbf{\retroagent{}:} \textcolor{orange}{\textless{}think\textgreater{} The current situation shows that the search has narrowed down to a specific product... with the color being pink. The size is not yet selected... The next step is to select the size small to confirm that it meets the criteria. \textless{}/think\textgreater{}} \textit{\textbf{Action:}} \textcolor{islamicgreen}{\texttt{click[small]}} \\
\midrule

\textbf{Observation 6:} You are now at step 6... 'Disney Mickey And Friends Happy Birthday Mickey Confetti T-Shirt' [SEP] 'Price: \$22.99' ... \\
\textbf{\retroagent{}:} \textcolor{orange}{\textless{}think\textgreater{} The current situation shows that the search has narrowed down to a specific product... The price is \$22.99, which is within the budget... I can proceed with clicking on 'Buy Now' to confirm the selection. \textless{}/think\textgreater{}} \textit{\textbf{Action:}} \textcolor{islamicgreen}{\texttt{click[buy now]}} \\

\bottomrule
\end{tabular}
\end{table*}

\begin{table*}[ht]
\caption{Failed trajectory generated by \retroagent{} (in-context
    self-reflection) at training step~65 on the WebShop validation set.}
\label{tab:traj_grpo}
\centering
\small 
\renewcommand{\arraystretch}{1.3} 
\begin{tabular}{p{0.95\textwidth}}
\toprule
\textbf{Instruction:} You are an expert autonomous agent operating in the WebShop e‑commerce environment. Your task is to: \textcolor{blue}{Find me officially licensed, machine wash men's t-shirts with polyester heathers, heathers cotton, cotton heather, needle sleeve, classic fit with color: pink, and fit type: youth, and size: small, and price lower than 50.00 dollars.} \\
\midrule

\textbf{Observation 1:} Your current observation is: `Search'. Your admissible actions of the current situation are: \texttt{['search[<your query>]', 'click[search]']}. Now it's your turn to take one action for the current step. You should first reason step-by-step about the current situation, then think carefully which admissible action best advances the shopping goal. This reasoning process MUST be enclosed within \textless{}think\textgreater{} \textless{}/think\textgreater{} tags. Once you've finished your reasoning, you should choose an admissible action for current step and present it within \textless{}action\textgreater{} \textless{}/action\textgreater{} tags. \\
\textbf{GRPO:} \textcolor{orange}{\textless{}think\textgreater{} To achieve the goal of finding officially licensed, machine wash men's t-shirts with specific attributes, I need to construct a search query that includes all the required attributes... I will construct the search query to include all these attributes and then execute the search. \textless{}/think\textgreater{}} \textit{\textbf{Action:}} \textcolor{islamicgreen}{\texttt{search[officially licensed men's t-shirts polyester heathers heathers cotton cotton heather needle sleeve classic fit color pink fit type youth fit type size small price < 50.00]}} \\
\midrule

\textbf{Observation 2:} You are now at step 2 and your current observation is: 'Back to Search' [SEP] 'Page 1 (Total results: 50)' [SEP] 'Next >' [SEP] 'B08HF132PN' [SEP] 'Flamingo Pink waterbird costume Gift Premium T-Shirt' [SEP] '\$19.99' [SEP] 'B07JVVDJ6L' [SEP] 'Canadian Tuxedo T Shirt - Novelty Denim Tux T-Shirt' [SEP] '\$19.99' [SEP] 'B09GL561XH' [SEP] 'Disney Mickey And Friends Happy Birthday Mickey Confetti T-Shirt' [SEP] '\$22.99' ... [SEP] 'B07HRFSNL4' [SEP] 'Nautica Men's Solid Crew Neck Short-Sleeve Pocket T-Shirt' [SEP] '\$16.05 to \$40.98'. \\
\textbf{GRPO:} \textcolor{orange}{\textless{}think\textgreater{} The current observation indicates that the search results are being displayed, and the first result is "B08HF132PN" with the title "Flamingo Pink waterbird costume Gift Premium T-Shirt" and a price of \$19.99. However, this product does not match the specific criteria we are looking for, as it is a costume t-shirt rather than a men's t-shirt, and it is not officially licensed. Given the current search results, the next logical step is to click on the first result to see if it matches the criteria... The most appropriate action here is to click on the first result, "B08HF132PN", to check if it matches the criteria. \textless{}/think\textgreater{}} \textit{\textbf{Action:}} \textcolor{red}{\texttt{click[b08hf132pn]}} \\
\midrule

\textbf{Observation 3:} You are now at step 3 and your current observation is: 'Back to Search' [SEP] '< Prev' [SEP] 'fit type' [SEP] 'men' [SEP] 'women' [SEP] 'youth' [SEP] 'size' [SEP] 'small' ... 'Flamingo Pink waterbird costume Gift Premium T-Shirt' [SEP] 'Price: \$19.99' [SEP] 'Rating: N.A.' [SEP] 'Description' [SEP] 'Features' [SEP] 'Reviews' [SEP] 'Buy Now'. \\
\textbf{GRPO:} \textcolor{orange}{\textless{}think\textgreater{} 1. The current observation shows that the search results are being displayed... 2. The product does not match the search criteria... 3. The next step is to refine the search results to ensure that the search criteria are met. The most appropriate action here is to click on "youth" to filter the results to show only products that are intended for youth. \textless{}/think\textgreater{}} \textit{\textbf{Action:}} \textcolor{islamicgreen}{\texttt{click[youth]}} \\
\midrule

\textbf{Observation 4:} You are now at step 4... 'fit type' [SEP] 'men' [SEP] 'women' [SEP] 'youth' [SEP] 'size' [SEP] 'small' ... 'Flamingo Pink waterbird costume Gift Premium T-Shirt' [SEP] 'Price: \$19.99' ... \\
\textbf{GRPO:} \textcolor{orange}{\textless{}think\textgreater{} The current observation shows that the search results are being displayed, and the last action was to click on "youth" fit type... However, the product does not match the criteria we are looking for... Given the current situation, the best action to take is to click on "small" size, as we are looking for a men's t-shirt with specific attributes, and the size is part of the search criteria. \textless{}/think\textgreater{}} \textit{\textbf{Action:}} \textcolor{red}{\texttt{click[small]}} \\
\midrule

\textbf{Observation 5:} You are now at step 5... 'Flamingo Pink waterbird costume Gift Premium T-Shirt' [SEP] 'Price: \$19.99' ... \\
\textbf{GRPO:} \textcolor{orange}{\textless{}think\textgreater{} The current observation indicates that the search results are filtered by fit type, size, and price. The last action was to click on "small" size... However, this product does not match the specific criteria we are looking for... Given the current search results, the best action to take is to click on "Buy Now" for the current product to see if it meets the criteria or if it is a close match. \textless{}/think\textgreater{}} \textit{\textbf{Action:}} \textcolor{red}{\texttt{click[Buy Now]}} \\

\bottomrule
\end{tabular}
\end{table*}
\section{Task Prompts}
\label{app:task_prompts}
\subsection{Prompt Templates for In-Context Self-Reflection}
\label{app:prompt_templates_self_reflection}

\begin{tcolorbox}[
  colframe=blue!75!black,
  colback=blue!5,
  title=Prompt for Single Induction on Webshop,
  coltitle=white,
  fonttitle=\bfseries
]
You are an expert evaluating a WebShop shopping attempt.\\
Your task is to: \{task\_description\}

You have just completed an attempt at this shopping task. The task was \{success\} completed.

Trajectory of the attempt:\\
\{current\_trajectory\}

\textless think\textgreater\\
Given the task outcome, analyze the trajectory to understand:
\begin{enumerate}
    \item What subtasks were attempted? (search, filter, select, purchase)
    \item Which subtasks succeeded vs failed based on the observations?
    \item What specific actions or decisions led to this outcome?
    \item What are the 1-2 most valuable lessons from this attempt?
\end{enumerate}
\textless/think\textgreater

Output your evaluation as JSON:
\begin{verbatim}
{
"subtasks": [
{"name": "search_product", "description": "[describe actual search]", 
 "status": "[completed or incomplete]"},
{"name": "apply_filters", "description": "[describe filters used]", 
 "status": "[completed or incomplete]"},
{"name": "select_item", "description": "[describe selection]", 
 "status": "[completed or incomplete]"},
{"name": "complete_purchase", "description": "[describe purchase]", 
 "status": "[completed or incomplete]"}
],
"task_success": [true if successfully completed, false if unsuccessfully 
completed],
"action_lesson": "[key action insight, e.g., 'Precise search with brand+model 
found exact match' OR 'Generic search missed required features']",
"navigation_lesson": "[navigation insight, e.g., 'Efficient use of filters 
saved time' OR 'Failed to check additional pages with better options']"
}
\end{verbatim}

EVALUATION GUIDELINES:
\begin{itemize}
    \item The task outcome has been provided - use it to set task\_success accordingly
    \item Focus on WHY the attempt had this outcome:
    \begin{itemize}
        \item If successful: What strategies worked well?
        \item If unsuccessful: What went wrong and where?
    \end{itemize}
    \item Each subtask status must reflect actual trajectory events
    \item Lessons should explain factors that led to the outcome
    \item Reference specific elements from trajectory (item IDs, pages, search terms)
    \item Use null for lessons only if truly not applicable
\end{itemize}

Output ONLY the JSON evaluation.
\end{tcolorbox}

\begin{tcolorbox}[
  colframe=blue!75!black,
  colback=blue!5,
  title=Prompt for Pairwise Induction on Webshop,
  coltitle=white,
  fonttitle=\bfseries
]
You are an expert evaluating a WebShop shopping attempt.\\
Your task is to: \{task\_description\}

You have just completed an attempt at this shopping task. The task was \{success\} completed.

\{reference\_trajectory\}

Trajectory of the attempt:\\
\{current\_trajectory\}

\textless think\textgreater\\
If a reference trajectory exists, compare it with the current trajectory.\\
Given the task outcome, analyze the trajectory to understand:
\begin{enumerate}
    \item What subtasks were attempted? (search, filter, select, purchase)
    \item Which subtasks succeeded vs failed based on the observations?
    \item What specific actions or decisions led to this outcome?
    \item What are the 1-2 most valuable lessons from this attempt?
\end{enumerate}
\textless/think\textgreater

Output your evaluation as JSON:
\begin{verbatim}
{
"subtasks": [
{"name": "search_product", "description": "[describe actual search]", 
 "status": "[completed or incomplete]"},
{"name": "apply_filters", "description": "[describe filters used]", 
 "status": "[completed or incomplete]"},
{"name": "select_item", "description": "[describe selection]", 
 "status": "[completed or incomplete]"},
{"name": "complete_purchase", "description": "[describe purchase]", 
 "status": "[completed or incomplete]"}
],
"task_success": [true if successfully completed, false if unsuccessfully 
completed],
"action_lesson": "[key action insight, e.g., 'Precise search with brand+model 
found exact match' OR 'Generic search missed required features']",
"navigation_lesson": "[navigation insight, e.g., 'Efficient use of filters 
saved time' OR 'Failed to check additional pages with better options']"
}
\end{verbatim}

EVALUATION GUIDELINES:
\begin{itemize}
    \item The task outcome has been provided - use it to set task\_success accordingly
    \item Focus on WHY the attempt had this outcome:
    \begin{itemize}
        \item If successful: What strategies worked well?
        \item If unsuccessful: What went wrong and where?
    \end{itemize}
    \item Each subtask status must reflect actual trajectory events
    \item Lessons should explain factors that led to the outcome
    \item Reference specific elements from trajectory (item IDs, pages, search terms)
    \item Use null for lessons only if truly not applicable
\end{itemize}

Output ONLY the JSON evaluation.
\end{tcolorbox}

\begin{tcolorbox}[
  colframe=blue!75!black,
  colback=blue!5,
  title=Prompt for Pairwise Induction on Alfworld,
  coltitle=white,
  fonttitle=\bfseries
]
You are an expert evaluating an ALFRED Embodied Environment task attempt.\\
Your task is to: \{task\_description\}

You have just completed an attempt at this task. The task was \{success\} completed.

\{reference\_trajectory\}

Trajectory of the attempt:\\
\{current\_trajectory\}

\textless think\textgreater\\
If a reference trajectory exists, compare it with the current trajectory.\\
Given the task outcome, analyze the trajectory to understand:
\begin{enumerate}
    \item What subtasks were attempted? (pick up, navigate, use appliance, place object)
    \item Which subtasks succeeded vs failed based on the observations?
    \item What specific actions or decisions led to this outcome?
    \item What is the most valuable lesson from this attempt?
\end{enumerate}
\textless/think\textgreater

Output your evaluation as JSON:
\begin{verbatim}
{
"subtasks": [
{"name": "pick_up_object", "description": "[describe pickup action, e.g., 
 'Pick up mug from countertop']", "status": "[completed or incomplete]"},
{"name": "navigate_to_location", "description": "[describe navigation, e.g., 
 'Go to microwave 1']", "status": "[completed or incomplete]"},
{"name": "use_appliance", "description": "[describe appliance use, e.g., 
 'Heat mug in microwave']", "status": "[completed or incomplete]"},
{"name": "place_object", "description": "[describe placement, e.g., 
 'Place heated mug in cabinet']", "status": "[completed or incomplete]"}
],
"task_success": [true if successfully completed task goal, false if failed],
"action_lesson": "[key action insight, e.g., 'Attempted to place mug 1 
 directly in cabinet 2 without heating - must use microwave 1 first' OR 
 'Successfully found knife in drawer 3 after checking wrong locations']",
"navigation_lesson": "[spatial insight, e.g., 'Microwave 1 located in 
 kitchen area, not near cabinets' OR 'Multiple sinkbasins exist - must 
 check all for target object']"
}
\end{verbatim}

EVALUATION GUIDELINES:
\begin{itemize}
    \item The task outcome has been provided - use it to set task\_success accordingly
    \item Focus on WHY the attempt had this outcome:
    \begin{itemize}
        \item If successful: What sequence or strategy worked well?
        \item If unsuccessful: What step failed or was missed?
    \end{itemize}
    \item Each subtask status must reflect actual trajectory events
    \item Lessons should explain factors that led to the outcome
    \item Reference specific elements from trajectory (object IDs, locations, appliances)
    \item Use null for lessons only if truly not applicable
\end{itemize}

Output ONLY the JSON evaluation.
\end{tcolorbox}

\begin{tcolorbox}[
  colframe=blue!75!black,
  colback=blue!5,
  title=Prompt for Pairwise Induction on Minesweeper (1/2),
  coltitle=white,
  fonttitle=\bfseries
]
You are an expert evaluating a Minesweeper game attempt.\\
Task Requirements: Reveal all non-mine cells on a \{board\_size\}x\{board\_size\} board with \{n\_mines\} mines without detonating any mine.

You have just completed an attempt at this Minesweeper game. The game was \{success\} completed.

\{reference\_trajectory\}

Current Trajectory of the attempt:\\
\{current\_trajectory\}

\textless think\textgreater\\
If a reference trajectory exists, compare it with the current trajectory.\\
Analyze the current trajectory to determine:
\begin{enumerate}
    \item Which subtasks were attempted and their completion status
    \item Specific actions/decisions that caused the outcome
    \item What went wrong (if failed) or right (if succeeded)
    \item Devise a concise, new plan of action that accounts for any mistakes with reference to specific actions that should be taken in the next trial
\end{enumerate}

Game notation for reference:
\begin{itemize}
    \item Cell states: ? (unopened), . (blank/no neighbors), 1-8 (mine count), * (mine)
    \item Coordinates: rows/columns indexed 1 to \{board\_size\}
    \item Valid actions: (row, col) where $1 \le \text{row,col} \le$ \{board\_size\}
    \item Blank cells auto-cascade to reveal connected blanks + borders
\end{itemize}

Subtask Completion Criteria (binary evaluation for failed trajectories too):
\begin{itemize}
    \item valid\_moves: COMPLETED if made at least 2 valid format moves; INCOMPLETE if mostly invalid actions
    \item exploration\_progress: COMPLETED if revealed >10\% of board; INCOMPLETE if revealed <10\%
    \item logical\_attempt: COMPLETED if attempted any deduction (even if wrong); INCOMPLETE if only random/invalid moves
    \item error\_recovery: COMPLETED if corrected any error within 3 attempts; INCOMPLETE if repeated same errors
    \item cascade\_usage: COMPLETED if triggered or attempted any cascade; INCOMPLETE if only single cell reveals
    \item systematic\_approach: COMPLETED if showed any pattern in move selection; INCOMPLETE if purely random
\end{itemize}
\textless/think\textgreater
\end{tcolorbox}

\begin{tcolorbox}[
  colframe=blue!75!black,
  colback=blue!5,
  title=Prompt for Pairwise Induction on Minesweeper (2/2),
  coltitle=white,
  fonttitle=\bfseries
]
Required JSON Output:
\begin{verbatim}
{
  "subtasks": [
    {"name": "valid_moves", "description": "[e.g., 'Made 5 valid moves like 
     (1,1), (2,3)' or 'Only invalid formats like (-1,-1)']", 
     "status": "[completed/incomplete]"},
    {"name": "exploration_progress", "description": "[e.g., 'Revealed 15 
     cells (25% of board)' or 'Only revealed 2 cells']", 
     "status": "[completed/incomplete]"},
    {"name": "logical_attempt", "description": "[e.g., 'Tried to use 
     (3,3)=1 constraint' or 'No deduction attempts']", 
     "status": "[completed/incomplete]"},
    {"name": "error_recovery", "description": "[e.g., 'Fixed format after 
     2 attempts' or 'Repeated invalid action 10 times']", 
     "status": "[completed/incomplete]"},
    {"name": "cascade_usage", "description": "[e.g., '(1,1) triggered 
     8-cell cascade' or 'No cascade attempts']", 
     "status": "[completed/incomplete]"},
    {"name": "systematic_approach", "description": "[e.g., 'Checked corners 
     first' or 'Random clicking']", "status": "[completed/incomplete]"}
  ],
  "trajectory_value": [count of completed subtasks out of 6],
  "task_success": [true if successfully completed, false if unsuccessfully 
   completed],
  "next_priority": "[Most important fix, e.g., 'Use valid (row,col) format' 
   or 'When cell shows 1, count unopened neighbors']"
}
\end{verbatim}

Evaluation Rules:
\begin{itemize}
    \item Award COMPLETED for ANY positive demonstration, even in failed games
    \item valid\_moves: Just need 2+ correctly formatted moves anywhere in trajectory
    \item exploration\_progress: 10\% is roughly 6 cells on 8x8 board - achievable even if hit mine
    \item logical\_attempt: Credit for trying logic, even if conclusion was wrong
    \item error\_recovery: Credit for any correction, even if made new errors later
    \item cascade\_usage: Credit for choosing corners/edges that could cascade
    \item systematic\_approach: Credit for any non-random pattern in moves
    \item trajectory\_value helps distinguish quality among failed attempts (0-6 scale)
\end{itemize}

Output JSON only.
\end{tcolorbox}

\begin{tcolorbox}[
  colframe=blue!75!black,
  colback=blue!5,
  title=Prompt for Pairwise Induction on Sokoban (1/2),
  coltitle=white,
  fonttitle=\bfseries
]
You are an expert evaluating a Sokoban game attempt.\\
Task Requirements: Push all boxes ('X') onto target spots ('O') in the grid without getting them stuck against walls ('\#') or in corners.

You have just completed an attempt at this Sokoban level. The game was \{success\} completed.

\{reference\_trajectory\}

Current Trajectory of the attempt:\\
\{current\_trajectory\}

\textless think\textgreater

If a reference trajectory exists, compare it with the current trajectory.\\
Given the task outcome, analyze the trajectory to understand:
\begin{enumerate}
    \item Which subtasks were attempted and their completion status
    \item Specific actions/decisions that caused the outcome
    \item What went wrong (if failed) or right (if succeeded)
    \item Devise a concise, new plan of action that accounts for any mistakes with reference to specific actions that should be taken in the next trial
\end{enumerate}

Game notation for reference:
\begin{itemize}
    \item Symbols: \# (wall), \_ (floor), O (target), X (box), P (player), $\surd$ (box on target)
    \item Coordinates: (row, col)
    \item Valid actions: ["up", "down", "left", "right"]
    \item Rules: Push only (no pull), one box at a time, walls block movement.
\end{itemize}

Subtask Completion Criteria (binary evaluation for failed trajectories too):
\begin{itemize}
    \item valid\_moves: COMPLETED if made at least 2 valid directional moves; INCOMPLETE if mostly invalid formats/hallucinations
    \item navigation\_logic: COMPLETED if player successfully navigated to a box; INCOMPLETE if stuck hitting walls/looping
    \item box\_interaction: COMPLETED if at least one box was pushed to a new coordinate; INCOMPLETE if no boxes moved
    \item deadlock\_avoidance: COMPLETED if avoided pushing boxes into unrecoverable corners/walls; INCOMPLETE if immediate deadlock created
    \item goal\_progress: COMPLETED if at least one box was placed on a target; INCOMPLETE if 0 boxes on targets
    \item systematic\_approach: COMPLETED if moves showed clear intent (e.g., moving behind a box to push); INCOMPLETE if random walking
\end{itemize}
\textless/think\textgreater
\end{tcolorbox}

\begin{tcolorbox}[
  colframe=blue!75!black,
  colback=blue!5,
  title=Prompt for Pairwise Induction on Sokoban (2/2),
  coltitle=white,
  fonttitle=\bfseries
]
Required JSON Output:
\begin{verbatim}
{
"subtasks": [
{"name": "valid_moves", "description": "[e.g., 'Outputted valid directions 
 like up, down' or 'Used invalid commands']", 
 "status": "[completed/incomplete]"},
{"name": "navigation_logic", "description": "[e.g., 'Reached box at (3,2)' 
 or 'Walked into wall at (1,1) repeatedly']", 
 "status": "[completed/incomplete]"},
{"name": "box_interaction", "description": "[e.g., 'Pushed box from (2,2) 
 to (2,3)' or 'No boxes moved']", "status": "[completed/incomplete]"},
{"name": "deadlock_avoidance", "description": "[e.g., 'Kept boxes away from 
 corners' or 'Pushed box into corner (1,1)']", 
 "status": "[completed/incomplete]"},
{"name": "goal_progress", "description": "[e.g., '1/3 boxes placed on target' 
 or 'No boxes on targets']", "status": "[completed/incomplete]"},
{"name": "systematic_approach", "description": "[e.g., 'Cleared path for 
 second box' or 'Random movement']", "status": "[completed/incomplete]"}
],
"trajectory_value": [count of completed subtasks out of 6],
"task_success": [true if successfully completed, false if unsuccessfully 
 completed],
"next_priority": "[Most important fix, e.g., 'Don't push box into corner at 
 (1,1)' or 'Move to (2,3) to push down']"
}
\end{verbatim}

Evaluation Rules:
\begin{itemize}
    \item Award COMPLETED for ANY positive demonstration, even in failed games
    \item valid\_moves: Just need 2+ correctly formatted actions
    \item navigation\_logic: Credit for traversing the map without getting stuck on walls immediately
    \item box\_interaction: Credit for changing the state of the board (moving a box)
    \item deadlock\_avoidance: Credit if the first box move didn't result in an immediate game-over state
    \item goal\_progress: Credit for securing at least one objective, even if others failed
    \item systematic\_approach: Credit for positioning the player specifically to push a box
    \item trajectory\_value helps distinguish quality among failed attempts (0-6 scale)
\end{itemize}

Output JSON only.
\end{tcolorbox}


\subsection{Prompt Templates for RL-Trained Self-Reflection}

\begin{tcolorbox}[
  colframe=blue!75!black,
  colback=blue!5,
  title=Prompt for Pairwise Induction on Webshop,
  coltitle=white,
  fonttitle=\bfseries
]
You are an expert evaluating a WebShop shopping attempt.\\
Target Task: \{task\_description\}

You have just completed an attempt at this shopping task.\\
Trajectory of the attempt:\\
\{current\_trajectory\}

\textless think\textgreater\\
If a reference trajectory exists, compare it with the current trajectory.\\
Analyze the trajectory to determine if the task was successful:
\begin{enumerate}[nosep]
    \item Identify the specific requirements in the `Target Task' (attributes, type, options).
    \item Examine the final action in the trajectory. Did it end in a `click[buy]'?
    \item If a purchase was made, compare the purchased item's details against the `Target Task' requirements.
    \item Did the purchased item match ALL requirements? (If no purchase was made, it is a failure).
    \item What specific actions or decisions led to this outcome?
    \item What are the 1-2 most valuable lessons from this attempt?
\end{enumerate}
\textless/think\textgreater

Output your evaluation as JSON:

\begin{verbatim}
{
"subtasks": [
{"name": "search_product", "description": "[describe actual search]", 
 "status": "[completed or incomplete]"},
{"name": "apply_filters", "description": "[describe filters used]", 
 "status": "[completed or incomplete]"},
{"name": "select_item", "description": "[describe selection]", 
 "status": "[completed or incomplete]"},
{"name": "complete_purchase", "description": "[describe purchase]", 
 "status": "[completed or incomplete]"}
],
"task_success": [true if the correct item was purchased, false otherwise],
"action_lesson": "[key action insight, e.g., 'Precise search with brand+model 
found exact match' OR 'Generic search missed required features']",
"navigation_lesson": "[navigation insight, e.g., 'Efficient use of filters 
saved time' OR 'Failed to check additional pages with better options']"
}
\end{verbatim}

EVALUATION GUIDELINES:
\begin{itemize}
    \item \textbf{Determine Success Yourself:} You must judge `task\_success' by comparing the purchased item in the trajectory to the Target Task.
    \item \textbf{Criteria for Success:} The task is ONLY true if the agent successfully clicked `buy' on an item that matches all required attributes (color, size, flavor, etc.).
    \item \textbf{Criteria for Failure:} If the trajectory ends without a purchase, or if the wrong item was bought, `task\_success' is false.
    \item Each subtask status must reflect actual trajectory events.
    \item Lessons should explain factors that led to the outcome.
    \item Reference specific elements from trajectory (item IDs, pages, search terms).
    \item Use null for lessons only if truly not applicable.
\end{itemize}

Output ONLY the JSON evaluation.
\end{tcolorbox}

\begin{tcolorbox}[
  colframe=blue!75!black,
  colback=blue!5,
  title=Prompt for Pairwise Induction on ALFWorld (1/2),
  coltitle=white,
  fonttitle=\bfseries,
  arc=2mm,
  boxrule=1.5pt
]

You are an expert evaluating an ALFWorld embodied agent attempt.

\textbf{Target Task:} \{task\_description\}

\vspace{0.5em}
You have just completed an attempt at this household task.

\textbf{Trajectory of the attempt:}
\newline
\{current\_trajectory\}

\vspace{1em}
\textbf{\textless think\textgreater}
\begin{enumerate}[leftmargin=*, nosep]
    \item If a reference trajectory exists, compare it with the current trajectory.
    \item Analyze the trajectory to determine if the task was successful:
    \begin{enumerate}
        \item Identify the specific requirements in the `Target Task' (target object, required state change, final destination).
        \item Examine the sequence of actions. Did the agent successfully locate the correct object?
        \item If a state change was required (clean, heat, cool, slice), was the correct appliance or tool used?
        \item Did the agent place the object in the correct final receptacle?
        \item Did the trajectory end with the `stop' action after achieving the goal state? (If the agent stopped prematurely or failed to stop, it is a failure).
        \item What specific actions or decisions led to this outcome?
        \item What are the 1-2 most valuable lessons from this attempt?
    \end{enumerate}
\end{enumerate}
\textbf{\textless /think\textgreater}

\vspace{1em}
Output your evaluation as JSON:

\begin{verbatim}
{{
"subtasks": [
{{"name": "locate_object", "description": "[describe search for target
object]", "status": "[completed or incomplete]"}},
{{"name": "acquire_object", "description": "[describe picking up target]",
"status": "[completed or incomplete]"}},
{{"name": "modify_state", "description": "[describe heating/cleaning
/cooling/slicing if applicable, else 'N/A']", 
"status": "[completed, incomplete, or N/A]"}},
{{"name": "place_object", "description": "[describe final placement]",
"status": "[completed or incomplete]"}}
],
"task_success": [true if the goal state was achieved and 'stop'
was called, false otherwise],
"action_lesson": "[key action insight, e.g., 'Used microwave to heat
apple instead of fridge' OR 'Failed to slice bread before plating']",
"navigation_lesson": "[spatial/search insight, e.g., 'systematically
checked all cabinet receptacles' OR 'wasted steps revisiting empty drawers']"
}}
\end{verbatim}
\end{tcolorbox}

\begin{tcolorbox}[
  colframe=blue!75!black,
  colback=blue!5,
  title=Prompt for Pairwise Induction on ALFWorld (2/2),
  coltitle=white,
  fonttitle=\bfseries,
  arc=2mm,
  boxrule=1.5pt
]

\textbf{EVALUATION GUIDELINES:}
\begin{itemize}[leftmargin=*]
    \item \textbf{Determine Success Yourself:} You must judge `task\_success' by comparing the final state in the trajectory to the Target Task.
    \item \textbf{Criteria for Success:} The task is ONLY true if the agent manipulated the correct object, achieved the correct state (e.g., hot, clean), placed it in the correct target, and issued the `stop' command.
    \item \textbf{Criteria for Failure:} If the trajectory ends without the `stop' command, or if the agent stopped without completing the goal (e.g., holding the object instead of placing it), `task\_success' is false.
    \item Each subtask status must reflect actual trajectory events.
    \item Lessons should explain factors that led to the outcome.
    \item Reference specific elements from trajectory (object IDs like `apple 1', receptacle IDs like `countertop 2').
    \item Use null for lessons only if truly not applicable.
\end{itemize}

Output ONLY the JSON evaluation.
\end{tcolorbox}

\begin{tcolorbox}[
  colframe=blue!75!black,
  colback=blue!5,
  title=Prompt for Pairwise Induction on Sokoban (1/2),
  coltitle=white,
  fonttitle=\bfseries,
  arc=2mm,
  boxrule=1.5pt
]

You are an expert evaluating a Sokoban game attempt.

\textbf{Task Requirements:} Push all boxes (`X') onto target spots (`O') in the grid without getting them stuck against walls (`\#') or in corners.

\vspace{0.5em}
You have just completed an attempt at this Sokoban level.

\textbf{Current Trajectory of the attempt:}
\newline
\{current\_trajectory\}

\vspace{1em}
\textbf{\textless think\textgreater}
\begin{enumerate}[leftmargin=*]
    \item If a reference trajectory exists, compare it with the current trajectory.
    \item Analyze the trajectory to determine if the task was successful:
    \begin{enumerate}
        \item Identify the grid layout and target locations in the `Target Task'.
        \item Examine the final board state in the trajectory. Are ALL boxes (`X') placed on targets (`O') resulting in `$\surd$'?
        \item If the game ended without success, check for deadlocks (boxes stuck in corners or against walls).
        \item Did the player successfully navigate the player (`P') to push positions without hitting walls repeatedly?
        \item What specific logic or movement behavior led to this outcome?
        \item What are the 1-2 most valuable lessons from this attempt?
        \item Devise a concise, new plan of action that accounts for any mistakes with reference to specific actions that should be taken in the next trial
    \end{enumerate}
\end{enumerate}

\vspace{0.5em}
\textbf{Game notation for reference:}
\begin{itemize}[nosep]
    \item Symbols: \# (wall), \_ (floor), O (target), X (box), P (player), $\surd$ (box on target)
    \item Coordinates: (row, col)
    \item Valid actions: ["up", "down", "left", "right"]
    \item Rules: Push only (no pull), one box at a time, walls block movement.
\end{itemize}

\vspace{0.5em}
\textbf{Subtask Completion Criteria (binary evaluation for failed trajectories too):}
\begin{itemize}[leftmargin=*]
    \item \textbf{valid\_moves:} COMPLETED if made at least 2 valid directional moves; INCOMPLETE if mostly invalid formats/hallucinations
    \item \textbf{navigation\_logic:} COMPLETED if player successfully navigated to a box; INCOMPLETE if stuck hitting walls/looping
    \item \textbf{box\_interaction:} COMPLETED if at least one box was pushed to a new coordinate; INCOMPLETE if no boxes moved
    \item \textbf{deadlock\_avoidance:} COMPLETED if avoided pushing boxes into unrecoverable corners/walls; INCOMPLETE if immediate deadlock created
    \item \textbf{goal\_progress:} COMPLETED if at least one box was placed on a target; INCOMPLETE if 0 boxes on targets
    \item \textbf{systematic\_approach:} COMPLETED if moves showed clear intent (e.g., moving behind a box to push); INCOMPLETE if random walking
\end{itemize}
\textbf{\textless /think\textgreater}
\end{tcolorbox}

\begin{tcolorbox}[
  colframe=blue!75!black,
  colback=blue!5,
  title=Prompt for Pairwise Induction on Sokoban (2/2),
  coltitle=white,
  fonttitle=\bfseries,
  arc=2mm,
  boxrule=1.5pt
]
\textbf{Required JSON Output:}
\begin{verbatim}
{{
  "subtasks": [
    {{"name": "valid_moves", "description": "[e.g., 'Outputted valid
    directions like up, down' or 'Used invalid commands']",
    "status": "[completed/incomplete]"}},
    {{"name": "navigation_logic", "description": "[e.g., 'Reached box at (3,2)'
    or 'Walked into wall at (1,1) repeatedly']",
    "status": "[completed/incomplete]"}},
    {{"name": "box_interaction", "description": "[e.g., 'Pushed box from (2,2)
    to (2,3)' or 'No boxes moved']",
    "status": "[completed/incomplete]"}},
    {{"name": "deadlock_avoidance", "description": "[e.g., 'Kept boxes away
    from corners' or 'Pushed box into corner (1,1)']",
    "status": "[completed/incomplete]"}},
    {{"name": "goal_progress", "description": "[e.g., '1/3 boxes placed
    on target' or 'No boxes on targets']",
    "status": "[completed/incomplete]"}},
    {{"name": "systematic_approach", "description": "[e.g., 'Cleared path
    for second box' or 'Random movement']",
    "status": "[completed/incomplete]"}}
  ],
  "trajectory_value": [count of completed subtasks out of 6],
  "task_success": [true if successfully placed all boxes on targets,
  false if deadlock or incomplete],
  "next_priority": "[Most important fix, e.g., 'Don't push box into corner
  at (1,1)' or 'Move to (2,3) to push down']"
}}
\end{verbatim}

\textbf{Evaluation Rules:}
\begin{itemize}[leftmargin=*]
    \item \textbf{Determine Success Yourself:} You must judge `task\_success' by comparing the final board state in the trajectory to the Target Task.
    \item \textbf{Criteria for Success:} The task is ONLY true if ALL boxes are on target spots (`$\surd$').
    \item \textbf{Criteria for Failure:} If the trajectory ends with a deadlock, or if the agent stopped before placing all boxes, `task\_success' is false.
    \item Each subtask status must reflect actual trajectory events.
    \item Lessons should explain factors that led to the outcome (planning vs. random).
    \item Reference specific elements from trajectory (coordinates, symbols).
    \item Use null for lessons only if truly not applicable.
\end{itemize}

Output ONLY the JSON evaluation.

\end{tcolorbox}

\begin{tcolorbox}[
  colframe=blue!75!black,
  colback=blue!5,
  title=Prompt for Pairwise Induction on Minesweeper (1/2),
  coltitle=white,
  fonttitle=\bfseries,
  arc=2mm,
  boxrule=1.5pt
]

You are an expert evaluating a Minesweeper game attempt.

\textbf{Task Requirements:} Reveal all non-mine cells on a \{board\_size\}x\{board\_size\} board with \{n\_mines\} mines without detonating any mine.

\vspace{0.5em}
You have just completed an attempt at this Minesweeper game.

\textbf{Current Trajectory of the attempt:}
\newline
\{current\_trajectory\}

\vspace{1em}
\textbf{\textless think\textgreater}
\begin{enumerate}[leftmargin=*]
    \item If a reference trajectory exists, compare it with the current trajectory.
    \item Analyze the trajectory to determine if the task was successful:
    \begin{enumerate}
        \item Identify the board constraints (size, mine count) in the `Target Task'.
        \item Examine the final action in the trajectory. Did it result in a mine detonation (loss) or a cleared board (win)?
        \item If the game ended without a mine detonation, check if ALL safe cells were revealed.
        \item Did the player successfully flag mines (optional but helpful) and reveal all safe spots? (If a mine was hit or safe cells remain hidden, it is a failure).
        \item What specific logic or guessing behavior led to this outcome?
        \item What are the 1-2 most valuable lessons from this attempt?
        \item Devise a concise, new plan of action that accounts for any mistakes with reference to specific actions that should be taken in the next trial
    \end{enumerate}
\end{enumerate}

\vspace{0.5em}
\textbf{Game notation for reference:}
\begin{itemize}[nosep]
    \item Cell states: ? (unopened), . (blank/no neighbors), 1-8 (mine count), * (mine)
    \item Coordinates: rows/columns indexed 1 to \{board\_size\}
    \item Valid actions: (row, col) where 1 $\le$ row,col $\le$ \{board\_size\}
    \item Blank cells auto-cascade to reveal connected blanks + borders
\end{itemize}

\vspace{0.5em}
\textbf{Subtask Completion Criteria (binary evaluation for failed trajectories too):}
\begin{itemize}[leftmargin=*]
    \item \textbf{valid\_moves:} COMPLETED if made at least 2 valid format moves; INCOMPLETE if mostly invalid actions
    \item \textbf{exploration\_progress:} COMPLETED if revealed >10\% of board; INCOMPLETE if revealed <10\%
    \item \textbf{logical\_attempt:} COMPLETED if attempted any deduction (even if wrong); INCOMPLETE if only random/invalid moves
    \item \textbf{error\_recovery:} COMPLETED if corrected any error within 3 attempts; INCOMPLETE if repeated same errors
    \item \textbf{cascade\_usage:} COMPLETED if triggered or attempted any cascade; INCOMPLETE if only single cell reveals
    \item \textbf{systematic\_approach:} COMPLETED if showed any pattern in move selection; INCOMPLETE if purely random
\end{itemize}
\textbf{\textless /think\textgreater}

\end{tcolorbox}

\begin{tcolorbox}[
  colframe=blue!75!black,
  colback=blue!5,
  title=Prompt for Pairwise Induction on Minesweeper (2/2),
  coltitle=white,
  fonttitle=\bfseries,
  arc=2mm,
  boxrule=1.5pt
]
\textbf{Required JSON Output:}
\begin{verbatim}
{{
  "subtasks": [
    {{"name": "valid_moves", "description": "[e.g., 'Made 5 valid moves
    like (1,1), (2,3)' or 'Only invalid formats like (-1,-1)']",
    "status": "[completed/incomplete]"}},
    {{"name": "exploration_progress", "description": "[e.g., 'Revealed 15 cells
    (25% of board)' or 'Only revealed 2 cells']",
    "status": "[completed/incomplete]"}},
    {{"name": "logical_attempt", "description": "[e.g., 'Tried to use (3,3)=1
    constraint' or 'No deduction attempts']",
    "status": "[completed/incomplete]"}},
    {{"name": "error_recovery", "description": "[e.g., 'Fixed format after
    2 attempts' or 'Repeated invalid action 10 times']",
    "status": "[completed/incomplete]"}},
    {{"name": "cascade_usage", "description": "[e.g., '(1,1) triggered 8-cell
    cascade' or 'No cascade attempts']",
    "status": "[completed/incomplete]"}},
    {{"name": "systematic_approach", "description": "[e.g., 'Checked corners
    first' or 'Random clicking']",
    "status": "[completed/incomplete]"}}
  ],
  "trajectory_value": [count of completed subtasks out of 6],
  "task_success": [true if successfully cleared all safe cells,
  false if detonated mine or incomplete],
  "next_priority": "[Most important fix, e.g., 'Use valid (row,col) format'
  or 'When cell shows 1, count unopened neighbors']"
}}
\end{verbatim}

\textbf{Evaluation Rules:}
\begin{itemize}[leftmargin=*]
    \item \textbf{Determine Success Yourself:} You must judge `task\_success' by comparing the final board state in the trajectory to the Target Task.
    \item \textbf{Criteria for Success:} The task is ONLY true if the agent successfully revealed ALL safe cells without detonating a mine.
    \item \textbf{Criteria for Failure:} If the trajectory ends with a mine detonation, or if the agent stopped before revealing all safe cells, `task\_success' is false.
    \item Each subtask status must reflect actual trajectory events.
    \item Lessons should explain factors that led to the outcome (logic vs. guessing).
    \item Reference specific elements from trajectory (coordinates, cell values).
    \item Use null for lessons only if truly not applicable.
\end{itemize}

Output ONLY the JSON evaluation.

\end{tcolorbox}

\subsection{Prompts for Analyzing the Quality of Intrinsic Feedback}
\label{app:prompt_templates_analyzing_self_reflection}
To assess the fidelity of the intrinsic feedback generated via self-reflection, we employ GPT-4o~\citep{openai2024gpt4ocard} as an external judge. Our evaluation focuses on two key components: the accuracy of the induced subtask completion scores (intrinsic rewards) and the quality of the summarized lessons (intrinsic feedback).

To verify the accuracy of the subtask completion scores, we utilize the prompt detailed in Section~\ref{app:prompt_templates_self_reflection}. To evaluate the quality of the summarized lessons derived from the agent's trajectories, we use the prompt presented below. 


\begin{tcolorbox}[
  colframe=blue!75!black,
  colback=blue!5,
  title=Prompt for Evaluating Summarized Lessons,
  coltitle=white,
  fonttitle=\bfseries
]
\small
\textbf{System Prompt:} \\
You are an expert evaluator of AI Memory Systems. Your goal is to determine the `Information Gain' and `Cruciality' of lessons generated by an agent. You must distinguish between generic fluff (low quality) and specific, actionable insights (high quality).

\vspace{0.2cm}
\hrule
\vspace{0.2cm}

\textbf{User Prompt:} \\
\textbf{\# Context} \\
The agent performed a task in a web environment. \\
Actual Outcome: \texttt{\{actual\_outcome\}}

\textbf{\# Trajectory (History of Actions)} \\
\texttt{\{trajectory\}}

\textbf{\# Agent's Generated Reflection (containing Lessons)} \\
\texttt{\{reflection\}}

\textbf{\# Evaluation Task} \\
Analyze the \texttt{action\_lesson} and \texttt{navigation\_lesson} in the reflection above.

\begin{enumerate}
    \item \textbf{Specificity}: Is the lesson specific to the UI elements/errors encountered? (e.g., ``Clicking `Submit' failed because the form was empty'' vs. ``I failed to click'').
    \item \textbf{Causal Accuracy}: Does the lesson correctly identify the root cause of the \texttt{\{actual\_outcome\}}?
    \item \textbf{Utility}: If the agent retrieves this lesson in a future attempt, will it significantly improve the success rate?
\end{enumerate}

\textbf{\# Output Format (JSON Only)}
\begin{verbatim}
{
    "lesson_quality_score": <int 1-10>,
    "specificity_rating": <"High"|"Medium"|"Low">,
    "utility_rating": <"High"|"Medium"|"Low">,
    "reasoning": "<Short explanation of why this lesson is useful/useless>",
    "is_hallucination": <bool, true if lesson mentions events not in trajectory>
}
\end{verbatim}
\end{tcolorbox}


\end{document}